\documentclass[11pt]{article}
\usepackage[final]{acl}
\usepackage{times}
\usepackage{latexsym}
\usepackage[T1]{fontenc}
\usepackage[utf8]{inputenc}
\usepackage{inconsolata}
\usepackage{microtype}
\usepackage{amsmath,amssymb}
\usepackage{booktabs}
\usepackage{graphicx}
\usepackage{xcolor}
\usepackage{adjustbox}
\usepackage{url}
\usepackage{tabularx}
\usepackage{array}
\usepackage{algorithm}
\usepackage{algpseudocode}
\usepackage{float}
\usepackage{placeins}
\usepackage[most]{tcolorbox}
\usepackage{fontawesome5} 

\definecolor{questionaccent}{HTML}{285B87} 
\definecolor{questionbg}{HTML}{F6F8FA}     
\newtcolorbox{questionbox}{
    enhanced,
    breakable,
    colback=questionbg,
    coltext=black,
    frame hidden,
    borderline west={1.5pt}{0pt}{questionaccent},
    sharp corners,
    boxrule=0pt,
    boxsep=0pt,
    left=7pt,
    right=7pt,
    top=5pt,
    bottom=5pt,
    before skip=7pt,
    after skip=7pt
}

\definecolor{figurebase}{HTML}{4F9BC7}
\definecolor{figurece}{HTML}{4DBA9A}
\definecolor{figurelp}{HTML}{F0B93F}
\definecolor{figuresrp}{HTML}{E47B68}

\newcommand{\secbest}[1]{%
  \begin{tabular}[t]{@{}c@{}}
    #1\\[-0.55ex]
    \hline
  \end{tabular}%
}

\usepackage{amsthm}
\newtheorem{srptheorem}{Theorem}
\newtheorem{srpproposition}[srptheorem]{Proposition}
\newtheorem{srplemma}[srptheorem]{Lemma}
\newtheorem{srpcorollary}[srptheorem]{Corollary}
\theoremstyle{remark}

\theoremstyle{plain}
\makeatletter
\newcommand{\appsection}[1]{%
  \section{#1}%
  \addcontentsline{apc}{section}{\protect\numberline{\thesection}#1}%
}
\newcommand{\appsubsection}[1]{%
  \subsection{#1}%
  \addcontentsline{apc}{subsection}{\protect\numberline{\thesubsection}#1}%
}

\newcommand{\applink}[1]{\hyperref[#1]{Appendix~\ref*{#1}}}
\newcommand{\appendixcontents}{%
  \pdfbookmark[1]{Appendix}{appendix-contents}%
  \section*{Appendix}%
  \phantomsection\label{app:contents}%
  \begingroup
    \setcounter{tocdepth}{2}
    \@starttoc{apc}%
  \endgroup
  \vspace{0.5em}%
}
\makeatother

\newcolumntype{Y}{>{\centering\arraybackslash}X}
\graphicspath{{figures/}}
\title{Beyond Distribution Matching: Selective Margin Preservation for Supervised Fine-Tuning}
\author{
  \textbf{Yuxuan Wang}$^{1,3*}$,
  \textbf{Guangxiang Zhao}$^{2*}$,
  \textbf{Xinyu Tan}$^{1}$,
  \textbf{Yahui Han}$^{1}$,
  \textbf{Xinxiang Yin}$^{4}$ \\
  \textbf{Zhonghai Wu}$^{1}$,
  \textbf{Xiangzheng Zhang}$^{2}$,
  \textbf{Lin Sun}$^{2}$,
  \textbf{Tong Yang}$^{1\dagger}$ \\[0.35em]
  \textnormal{$^{1}$Peking University \quad $^{2}$Qiyuan Tech} \\
  \textnormal{$^{3}$Xidian University 
  \quad $^{4}$The University of Hong Kong} \\[0.15em]
  {\small\normalfont $^{*}$Equal contribution.\qquad
  $^{\dagger}$Corresponding author: Tong Yang.}
}

\begin{document}

\maketitle
\begin{abstract}
Supervised fine-tuning (SFT) changes logit margins among non-target tokens
at teacher-forced prefixes. Matching a local reference distribution keeps
those probabilities and also penalizes margin expansions. We ask whether a
soft auxiliary can preserve selected local non-target margin relations
without freezing that distribution. We introduce \textbf{Selective
Rank-Preserving SFT (SRP-SFT)}, a soft margin target: a hinge on
reference-scaled lower bounds for pairs from the reference's top-$K$
non-target tokens, applied at fixed teacher-forced prefixes. The hinge
stops once the requested bound is met. We derive the KL
cost of violating a margin, characterize the feasible region and its
conditional utility, and give a finite-step condition for reducing
aggregate violations. In one-epoch experiments on Qwen3-4B-Base and
UltraFeedback, this target raises MATH-500 and code-generation scores over
CE and LP-SFT at all three training seeds. At seed 1234, MATH-500 pass@1
rises from 49.56\% to 54.56\% and the four-metric code mean rises from
65.91\% to 69.50\% relative to LP-SFT. Ranking controls, held-out cache
diagnostics, sensitivity sweeps, and six additional backbones show where
these score changes appear.
\end{abstract}

\section{Introduction}
\suppressfloats[t]

\begin{figure}[!t]
    \centering
    \includegraphics[width=\columnwidth]{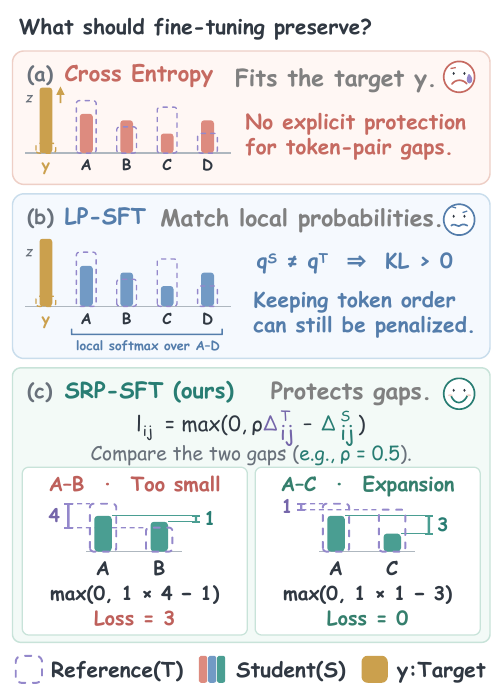}
    \caption{The three objectives preserve different local quantities.
    CE fits the observed token by maximizing its likelihood.
    LP-SFT matches a reference distribution.
    SRP-SFT imposes selected margin lower bounds and stops penalizing a pair
    once its bound is met.}
    \label{fig:motivation}
\end{figure}

Supervised fine-tuning (SFT) adapts pretrained language models to downstream tasks such as instruction following, reasoning, and code generation. At each training position, cross-entropy (CE) increases the probability of the observed next token but does not explicitly preserve the relative preferences among other tokens, potentially weakening or reversing preferences learned during pretraining.

One approach is to keep the fine-tuned model close to a frozen reference model. LP-SFT, our closest comparator, does so by matching the reference's probability distribution over a set of plausible non-target tokens \citep{lpsft}. However, probability matching is more restrictive than preference preservation: if the reference prefers token A over token B, a student that strengthens this preference is still penalized because its probabilities differ from the reference.

This raises a natural question:

\begin{questionbox}
\textcolor{questionaccent}{\faLightbulb}\enspace
\textit{Can we preserve the reference model's
preferences without requiring the student to
reproduce its exact probability distribution?}
\end{questionbox}

Preserving only the ordering of non-target tokens
is insufficient, since the ranking can remain
unchanged even as the gap between tokens approaches
zero. We therefore seek a middle ground that
prevents excessive preference weakening while
allowing stronger preferences during adaptation.

Motivated by this idea, we propose \textbf{Selective Rank-Preserving SFT (SRP-SFT)}, which preserves selected preference relationships through soft margin constraints. At each fixed teacher-forced prefix, a frozen reference identifies plausible non-target alternatives and provides their pairwise logit margins. Rather than matching their full probability distribution, SRP-SFT applies a one-sided hinge penalty only when a student margin falls below a reference-scaled threshold. Once the threshold is met, no further penalty is imposed, allowing additional separation. The threshold can be adjusted to permit limited contraction or encourage stronger preferences, while the standard CE objective continues to fit the observed target token.

We study the theoretical properties and empirical behavior of this selective preservation objective, characterizing how margin constraints differ from KL-based distribution matching, what changes they permit, and under what conditions the auxiliary penalty can reduce margin violations. In one-epoch experiments on Qwen3-4B-Base and UltraFeedback, SRP-SFT improves reasoning and code-generation scores over CE and LP-SFT across three matched training seeds. Held-out diagnostics further show that, although LP-SFT achieves lower local KL divergence, SRP-SFT exhibits less margin contraction and fewer ranking inversions. Comparisons with alternative ranking objectives and parameter sensitivity studies reveal that stronger preference preservation does not always translate into better task performance. Experiments on six additional backbones illustrate both the benefits and limitations of the proposed objective.

Our contributions are threefold. \textbf{First}, we introduce SRP-SFT, a selective margin-based objective that protects local non-target preference relationships without requiring exact distribution matching. \textbf{Second}, we establish theoretical properties of this objective, including the KL cost of margin violations, the geometry of its feasible region, a conditional utility bound, and conditions for reducing aggregate violations during optimization. \textbf{Third}, we conduct extensive evaluations with matched multi-seed comparisons, alternative ranking objectives, held-out structural diagnostics, sensitivity analyses, and six additional backbones.

\section{Related Work}
\label{sec:related}
\paragraph{Auxiliary targets in supervised fine-tuning.}
Several SFT objectives add a signal beyond the observed target
\citep{luo2025empirical}. Learning without Forgetting matches old outputs \citep{li2018lwf}. EWC and $L^2$-SP anchor parameters
\citep{kirkpatrick2017ewc,li2018l2sp}. Stored-example constraints and NTK
analyses address update interference \citep{theory-gem,theory-ntk}.
Recent SFT objectives reweight target supervision in DFT
\citep{wu2025generalization} and EAFT \citep{diao2026entropy}, preserve
diversity in GEM \citep{li2025preserving}, regularize adverse updates in LFU
\citep{lfu}, or add a KL anchor in ASFT \citep{zhu2025anchored}. RPSFT
protects projected rotations in pretrained singular subspaces
\citep{jin2026rpsft}. Our auxiliary is a soft margin target:
LP-SFT matches a frozen reference's local non-target distribution
\citep{lpsft}, whereas SRP-SFT places lower bounds on selected non-target
pairs at teacher-forced prefixes.

\paragraph{Distillation and relational supervision.}
Knowledge distillation transfers teacher knowledge through output matching or alternative training signals \citep{hinton2015distilling,gu2024minillm,agarwal2024policy,zhao2022dkd}. Ranking objectives \citep{burges2005ranknet,cao2007listnet,xia2008listmle} motivate relational transfer through candidate rankings \citep{reddi2021rankdistil,rlkd,kendallkd}, prediction correlations \citep{huang2022dist}, and logit differences \citep{marginmse,theory-bild}. Unlike distribution matching, relational supervision need not reproduce the teacher's exact probabilities. Ordering alone allows margins to shrink, while exact gap matching penalizes expansion. SRP-SFT instead imposes one-sided lower bounds on selected non-target margins at each fixed teacher-forced prefix, allowing further separation without penalty.

\section{Preliminaries}
\label{sec:preliminaries}
\paragraph{Supervised model update.}
Let $\mathcal D=\{(x^{(n)},y^{(n)})\}_{n=1}^{N}$ be an SFT dataset,
with response $y=(y_1,\ldots,y_T)$ and vocabulary $\mathcal V$.
At the teacher-forced prefix $(x,y_{<t})$, the student produces logits
$z_t^S\in\mathbb R^{|\mathcal V|}$ and probabilities
$p_t^S=\operatorname{softmax}(z_t^S)$. CE minimizes
$\mathcal L_{\mathrm{CE},t}=-\log p_t^S(y_t)$ over valid response positions.
The frozen reference $T$ is a copy of the pretrained checkpoint from
which the trainable student $S$ is initialized. We compare the student and reference logits
at identical teacher-forced prefixes, so a change in local preference
is not confounded with a change in generated history.

\paragraph{Local preference geometry.}
For $M\in\{T,S\}$ and token pair $(i,j)$, the margin is
\begin{equation}
 \Delta_{ij,t}^{M}=z_{t,i}^{M}-z_{t,j}^{M}
 =\log\frac{p_t^{M}(i)}{p_t^{M}(j)}.
 \label{eq:margin_def}
\end{equation}
Its sign gives the ordering and its magnitude the log-odds strength.
For a reference-preferred pair ($\Delta_{ij,t}^{T}>0$), a negative
student margin is an inversion, a smaller nonnegative margin is a
contraction, and a larger margin is an expansion. Renormalizing over a
shared candidate set leaves this ratio unchanged. We can therefore
specify relations within that set without prescribing its total
probability mass. The selected-margin target is
\begin{equation}
 \Delta_{ij,t}^{S}\geq\rho\Delta_{ij,t}^{T},\qquad \rho\geq0.
 \label{eq:directional_constraint}
\end{equation}
Here $\rho=0$ requires weak order, $0<\rho<1$ allows bounded contraction,
$\rho=1$ requests at least the reference margin, and $\rho>1$ requests
selective sharpening. These are bounds on individual selected margins.
The trained soft-penalty objective need not satisfy all of them.

\paragraph{CE does not impose this target.}
In free-logit coordinates, CE gives
$\dot\Delta_{ij}=-(p_i-p_j)$: a non-target margin still contracts toward zero without
finite-time inversion. Shared parameters can instead produce contraction
or expansion (Appendices~\ref{app:theory_contraction}
and~\ref{app:theory_parameter}). The margin target specifies
which selected margins enter the auxiliary loss under these
parameter-dependent dynamics.

\section{Method: Selective Margin Constraints}
\label{sec:method}
SRP-SFT uses a frozen reference to select non-target alternatives and set
margin constraints between them. The construction specifies the candidate
set, pair orientations, and penalty for violating a requested margin.
CE continues to supervise the observed token, while the auxiliary acts
on the selected non-target relations (Figure~\ref{fig:srp_main}).

\begin{figure*}[t]
    \centering
    \includegraphics[width=\textwidth]{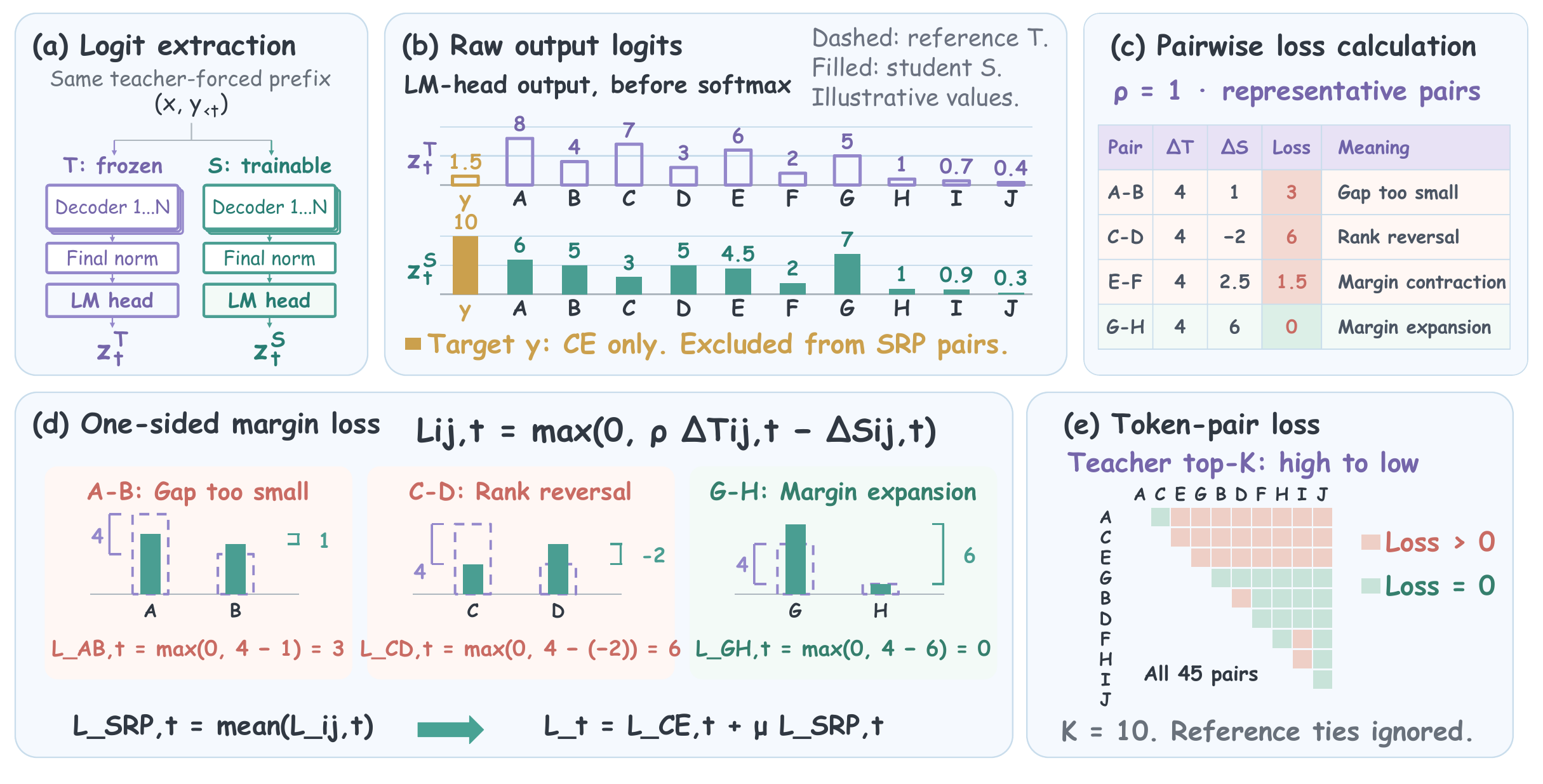}
    \caption{The selective margin target. The frozen reference
    selects plausible non-target alternatives and orients their pairs.
    CE fits the excluded target. The current student receives a pairwise
    penalty only below the requested margin. A satisfied pair supplies
    no direct auxiliary gradient. Both models use the same prefix.}
    \label{fig:srp_main}
\end{figure*}

\subsection{Selecting Non-Target Alternatives}
At each valid supervised position $t$, select the reference's top-$K$
non-target alternatives:
\begin{equation}
 \mathcal A_t=\operatorname{TopK}_{k\in\mathcal V\setminus\{y_t\}}
 z_{t,k}^{T}.
 \label{eq:alternative_set}
\end{equation}
Reference selection focuses the objective on tokens that are plausible
at the current prefix. Increasing $K$ broadens candidate coverage and
increases the maximum number of constraints quadratically.

Excluding the target $y_t$ separates the auxiliary's role from target fitting. CE can
promote an observed answer even when the reference scores it poorly.
Protecting reference-oriented pairs involving that answer could directly
oppose its promotion. Non-target constraints remove this direct
competition over the target logit, although shared parameters can still
produce gradient conflict.

\subsection{Orienting Selected Pairs}
Within $\mathcal A_t$, use the frozen reference to orient pairs:
\begin{equation}
 \mathcal P_t=\{(i,j):i,j\in\mathcal A_t,\ i\ne j,\
 \Delta_{ij,t}^{T}>0\}.
 \label{eq:pair_set}
\end{equation}
Each ordered unordered pair contributes one orientation, giving
at most $K(K-1)/2$ constraints. Teacher ties are ignored. Candidate
identities and orientations remain fixed per prefix. The student's ranking does not redefine what is protected.

Neither relations outside $\mathcal A_t$ nor its total probability mass
receive direct SRP supervision. A common shift of its logits is allowed
by both SRP and local KL. The additional freedom in SRP is to change
\emph{relative} logits while satisfying the selected margin bounds
(Theorem~\ref{thm:srp_geometry_utility}).

\subsection{Penalizing Only Margin Shortfalls}
At position $t$, average the selected hinge violations:
\begin{equation}
 \mathcal L_{\mathrm{SRP},t}
 =\frac{1}{|\mathcal P_t|}\sum_{(i,j)\in\mathcal P_t}
 [\rho\Delta_{ij,t}^{T}-\Delta_{ij,t}^{S}]_+,
 \label{eq:srp_loss}
\end{equation}
Here $[u]_+=\max(u,0)$, and an empty $\mathcal P_t$ contributes zero.
Average the total loss over valid supervised response positions:
\begin{equation}
 \mathcal L=\frac{1}{T}\sum_{t=1}^{T}
 (\mathcal L_{\mathrm{CE},t}+\mu\mathcal L_{\mathrm{SRP},t}),
 \label{eq:total_loss}
\end{equation}
The coefficient $\mu\geq0$ prices violations against target fitting,
whereas $\rho$ sets the requested margin. Changing $\rho$ changes the
boundary and active pairs. Changing $\mu$ rescales their contribution.
The two controls are not interchangeable.

A pair is penalized only while
$\Delta_{ij,t}^{S}<\rho\Delta_{ij,t}^{T}$. At or above the boundary,
its loss and chosen subgradient are zero. In particular, at $\rho=1$,
the objective penalizes contraction and inversion without rewarding or
penalizing further expansion. For $\rho>1$, expansion relative to the
teacher can still fall short of the scaled bound. Finite $\mu$ permits
residual violations: the zero-loss region characterizes the target,
not an invariant set for training updates.

\subsection{Gradients and Cached Implementation}
For an active pair,
$\ell_{ij,t}=\rho\Delta_{ij,t}^{T}-(z_{t,i}^{S}-z_{t,j}^{S})$
has logit gradients
\begin{equation}
 \frac{\partial\ell_{ij,t}}{\partial z_{t,i}^{S}}=-1,
 \qquad
 \frac{\partial\ell_{ij,t}}{\partial z_{t,j}^{S}}=+1.
 \label{eq:active_gradient}
\end{equation}
An isolated logit-gradient step increases the preferred token relative
to its partner. CE and other active pairs can move a satisfied pair
below its bound again, reactivating its hinge. With shared parameters,
aggregate corrections need not improve every margin in a single update.

Reference token IDs and logits are cached for fixed teacher-forced
examples. Each update gathers student logits and recomputes hinge
activity without another reference pass. The headline setting
is $K=10,\rho=1,\mu=1$, with at most 45 pairs per position.
Algorithm~\ref{alg:srp-sft} and Appendix~\ref{app:configuration} give the
implementation. Appendix~\ref{app:rho1_efficiency} reports training pilots
and the separate cost of constructing the cache.

\section{Properties of the Selected-Margin Target}
\label{sec:srp_theory}
We characterize the target: how its cost differs
from KL matching, the geometry of its feasible set, and when its
soft penalty reduces aggregate violations. All results use finite logits, fixed
teacher-forced prefixes, and the pairs in Eq.~\ref{eq:pair_set}.
At one prefix, write
$q=\operatorname{softmax}(z^T|_{\mathcal A})$ and
$p=\operatorname{softmax}(z^S|_{\mathcal A})$ for the locally normalized
distributions. For a selected state--pair triple $a$, let
\begin{equation}
 v_a=[\rho d_a^T-d_a^S]_+,\qquad d_a^T>0.
 \label{eq:theory_violation}
\end{equation}
The analysis specializes classical KL minimization, order-restricted
approximation, and stochastic-order arguments
\citep{srp-csiszar1975,srp-stout2012,srp-shaked2007} to the selected-margin
target. Proofs are in Appendix~\ref{app:srp_proofs}.

\subsection{The KL Cost of Losing a Margin}
\label{sec:theory_kl}
\begin{srptheorem}[KL threshold for a prescribed pair]
\label{thm:srp_kl_threshold}
Fix $q_i>q_j>0$, $m=q_i+q_j$, and
$d=\log(q_i/q_j)$. For any $b<d$,
\begin{equation}
 \inf_{p:\,\log(p_i/p_j)\leq b}\operatorname{KL}(q\|p)
 =m\,\operatorname{kl}\!\left(\sigma(d)\middle\|\sigma(b)\right),
 \label{eq:srp_kl_threshold}
\end{equation}
where the infimum is over strictly positive probability vectors on the
same set, $\sigma$ is the logistic function, and $\operatorname{kl}$ is
Bernoulli KL. The minimum is attained by leaving all other probabilities
unchanged and setting $p_i=m\sigma(b)$, $p_j=m\sigma(-b)$.
\end{srptheorem}
KL \emph{can} therefore certify a prescribed margin: a value below the
threshold excludes $d^S\leq b$. The threshold also identifies how
matching prices a loss of preference. For $b=d-\delta$, the minimum
cost is
$m\sigma(d)(1-\sigma(d))\delta^2/2+O(\delta^3)$, whereas at $\rho=1$
that pair contributes $\delta/|\mathcal P|$ to SRP. The KL cost depends
on the pair's local probability mass and gap. Without conditions on
both, there is no positive pair-independent threshold.

Conversely, the scaled logits $z^S=c z^T$ on the candidates satisfy every SRP constraint
when $c\geq\max\{1,\rho\}$, while local KL can diverge as $c\to\infty$
for nonuniform $q$. The two objectives therefore assign different costs
to reference deviations and define different preservation targets.

\subsection{Admissible Geometry and Conditional Utility}
\label{sec:theory_geometry}
Define the feasible logit set and worst violation:
\begin{equation}
 \begin{aligned}
 \mathcal C_\rho&=\{u:u_i-u_j\geq\rho d_{ij}^T,
                              \ (i,j)\in\mathcal P\},\\
 V&=\max_{(i,j)\in\mathcal P}v_{ij},\qquad
 q^{(\rho)}=\operatorname{softmax}(\rho z^T|_{\mathcal A}).
 \end{aligned}
 \label{eq:srp_cone_definition}
\end{equation}
The maximum is zero for an empty pair set. Subtracting $\rho z^T$
turns $\mathcal C_\rho$ into an order cone. This geometry describes the
logit changes permitted by the margin constraints per prefix.

\begin{srptheorem}[Repair distance and local utility]
\label{thm:srp_geometry_utility}
For the full teacher-preferred pair set,
\begin{equation}
 \inf_{u\in\mathcal C_\rho}\|z^S|_{\mathcal A}-u\|_\infty=V/2.
 \label{eq:srp_repair_distance}
\end{equation}
Let $U_i\in[0,1]$ be any candidate utility that is nondecreasing with
$z_i^T$ and constant on teacher ties. Then
\begin{equation}
 \mathbb E_p U\geq\mathbb E_{q^{(\rho)}}U-\tanh(V/4).
 \label{eq:srp_utility_bound}
\end{equation}
For $\rho\geq1$, $\mathbb E_{q^{(\rho)}}U\geq\mathbb E_qU$.
The error term is tight already for two candidates and $\rho>0$.
\end{srptheorem}
The first statement quantifies the correction needed to reach that
region. The second gives it a local utility interpretation: at zero
violation and $\rho\geq1$, no teacher-monotone conditional utility is
lowered relative to the reference. Nonzero violations yield the stated
degradation bound. Since $V\leq|\mathcal P|\mathcal L_{\mathrm{SRP},t}$,
a small mean hinge need not protect every pair equally. The appendix
provides a probability-weighted refinement and sparse-pair analysis.
Teacher--utility alignment is a substantive assumption. The guarantee
concerns choices within $\mathcal A$, not its probability mass, the
excluded target, or sequence-level correctness.

\subsection{When a Soft Penalty Repairs Violations}
\label{sec:theory_optimization}
Let $f$ be the CE objective on a fixed collection of prefixes and
$R(\theta)=\sum_a w_a v_a(\theta)$, with the actual fixed reduction
weights $w_a>0$. Set
$r=\sum_{a:v_a>0}w_a\nabla d_a^S$ and
$g=\nabla f-\mu r$. Suppose each margin gradient is $L_a$-Lipschitz
along the update segment and write $L_R=\sum_a w_aL_a$.

\begin{srpproposition}[Finite-step aggregate repair]
\label{prop:srp_aggregate_repair}
For $\theta^+=\theta-\eta g$, $\eta>0$,
\begin{equation}
 \begin{aligned}
 R(\theta^+)-R(\theta)
 &\leq\eta\langle r,\nabla f\rangle-\eta\mu\|r\|_2^2\\
 &\quad+C_\eta+\tfrac12L_R\eta^2\|g\|_2^2.
 \end{aligned}
 \label{eq:srp_finite_repair}
\end{equation}
Here $C_\eta\geq0$ is the explicitly computable linearized hinge-crossing
correction in Eq.~\ref{eq:srp_crossing}. It is zero when no linearized
hinge crosses its boundary. A negative right-hand side certifies a strict
decrease of the aggregate violation once updated.
\end{srpproposition}

The bound separates auxiliary repair from CE interference, hinge crossings, and curvature. It requires neither convexity nor a frozen neural tangent kernel, but certifies an aggregate, not every pair. Appendix~\ref{app:srp_repair} gives the version for any optimizer step, including AdamW. When the reference equals the initialization and $0<\rho\leq1$, a checkpoint satisfying $(f+\mu R)(\theta)\leq f(\theta_0)+\xi$ obeys $R(\theta)\leq(f(\theta_0)-f(\theta)+\xi)/\mu$ for $\mu>0$. With the teacher margin distribution, this bounds sorting failures at nonzero residual (Appendix~\ref{app:srp_train_to_risk}).

\begin{table*}[!t]
\centering
\scriptsize
\setlength{\tabcolsep}{1.2pt}
\renewcommand{\arraystretch}{1.10}
\begin{tabularx}{\textwidth}{@{}>{\raggedright\arraybackslash}p{1.20cm}*{13}{Y}@{}}
\toprule
Method
& \multicolumn{2}{c}{MATH-500}
& \multicolumn{2}{c}{AIME}
& \multicolumn{2}{c}{GSM-Plus mini}
& LogiQA2
& BBH-5
& ARC-C
& \shortstack{Hella-\\Swag}
& MMLU
& Phy-Oly
& Mean \\
\cmidrule(lr){2-3}
\cmidrule(lr){4-5}
\cmidrule(lr){6-7}
& p@1
& p@4
& p@1
& p@4
& Strict
& Flexible
& Acc.
& EM
& Norm.
& Norm.
& 5-shot
& Acc.
& \\
\midrule
Base$^{\dagger}$ & 48.11 & 71.23 & 1.81 & 5.17 & 51.92 & \textbf{60.21} & \secbest{44.15} & \secbest{16.40} & 52.30 & \textbf{73.70} & \textbf{73.00} & 8.90 & 42.24 \\
CE & 27.98 & 54.91 & 0.54 & 1.99 & 45.62 & 43.83 & 43.32 & 13.60 & 53.07 & 72.16 & 68.90 & 5.93 & 35.99 \\
GEM & 22.21 & 47.74 & 0.54 & 1.94 & 46.58 & 47.38 & 43.70 & 12.80 & 49.49 & \secbest{73.40} & 68.83 & 5.51 & 35.01 \\
DFT & 47.73 & 59.57 & 0.66 & 2.19 & 45.58 & 44.50 & 39.63 & 0.40 & 47.70 & 65.96 & 65.23 & 3.81 & 35.25 \\
ASFT & 33.41 & 57.22 & 2.03 & 5.42 & 47.50 & 48.79 & 37.09 & \textbf{18.40} & \textbf{55.03} & 70.72 & 67.56 & 8.47 & 37.64 \\
EAFT & 28.52 & 54.57 & 0.54 & 1.91 & 45.83 & 33.71 & 41.09 & 14.80 & 51.88 & 72.32 & 69.57 & 4.24 & 34.91 \\
LP-SFT & \secbest{49.56} & \secbest{77.52} & \secbest{3.28} & \secbest{9.25} & \secbest{57.88} & 59.08 & 44.02 & 14.80 & 53.75 & 73.35 & 71.48 & 8.05 & \secbest{43.50} \\
\textbf{SRP-SFT} & \textbf{54.56} & \textbf{80.50} & \textbf{4.76} & \textbf{11.68} & \textbf{58.83} & \secbest{59.92} & \textbf{46.25} & 14.80 & \secbest{54.18} & 73.36 & \secbest{71.89} & \textbf{10.59} & \textbf{45.11} \\
\bottomrule
\end{tabularx}
\caption{One-epoch Qwen3-4B-Base reasoning results on UltraFeedback
(seed 1234, percentages). All fine-tuned methods use the same training
implementation, software and hardware environment, and common training
hyperparameters, including a 478-step cosine schedule and 15-step warmup.
Mean averages all twelve metrics from the same checkpoint.
Bold/underline mark best/second-best. $^{\dagger}$ is the untouched Base.}
\label{tab:main_benchmark}
\end{table*}

\section{Experimental Setup}
\label{sec:experimental_setup}
We fine-tune Qwen3-4B-Base \citep{yang2025qwen3} on UltraFeedback's \texttt{train\_sft} split \citep{cui2023ultrafeedback}. ChatML sequences are right-truncated at 2{,}048 tokens; loss covers assistant responses. Fine-tuned methods share initialization, per-seed data order, and optimizer: one epoch, 478 AdamW steps, learning rate $2\times10^{-5}$, a cosine schedule with 15 warmup steps, and effective batch size 128 on eight A800-80GB GPUs. Seeds: 1234, 42, and 3407. Tables~\ref{tab:main_benchmark} and~\ref{tab:code_benchmark} report seed 1234. Base denotes the untouched checkpoint. Appendix~\ref{app:configuration} details revisions, row counts, and the software stack.

The matched comparison includes CE, GEM, DFT, ASFT, EAFT, LP-SFT, and
SRP-SFT. LP-SFT applies forward KL over the reference's top-10 non-target
tokens at $\mu=1$. SRP-SFT adds the soft hinge on strictly ordered pairs
among those tokens, with $K=10$, $\rho=1$, and $\mu=1$. Reference logits are
cached once. These coefficients are fixed before evaluation and reused at
all three seeds. Other loss definitions, sensitivity sweeps, and the
six-backbone studies are in Appendix~\ref{app:configuration}.

The reasoning mean averages twelve metrics from one checkpoint, including
MATH-500 pass@1/pass@4 at temperature 0.6 with 16 samples. The code mean
averages HumanEval+ and MBPP+ pass@1/pass@4 under EvalPlus 0.3.1, at
temperature 0.8 with 20 samples. Held-out margin diagnostics use the
998-example \texttt{test\_sft} cache.

\section{Results}
\label{sec:results}

\begin{table}[!tbp]
\centering
\footnotesize
\setlength{\tabcolsep}{2.0pt}
\renewcommand{\arraystretch}{1.10}
\begin{tabularx}{\columnwidth}{@{}l*{5}{Y}@{}}
\toprule
Method & \multicolumn{2}{c}{HumanEval+} & \multicolumn{2}{c}{MBPP+} & Mean \\
\cmidrule(lr){2-3}\cmidrule(lr){4-5}
& pass@1 & pass@4 & pass@1 & pass@4 & \\
\midrule
Base$^{\dagger}$ & 45.49 & 71.56 & 46.12 & \secbest{72.50} & 58.92 \\
CE & 44.18 & 69.56 & 44.66 & 65.79 & 56.04 \\
GEM & 22.93 & 46.34 & 26.40 & 52.13 & 36.95 \\
DFT & \secbest{64.12} & 69.98 & \textbf{57.24} & 63.56 & 63.72 \\
ASFT & \textbf{64.79} & 77.18 & \secbest{57.09} & 67.90 & \secbest{66.74} \\
EAFT & 46.43 & 72.00 & 44.59 & 65.45 & 57.12 \\
LP-SFT & 57.90 & \secbest{81.67} & 51.87 & 72.23 & 65.91 \\
\textbf{SRP-SFT} & 62.84 & \textbf{83.87} & 56.24 & \textbf{75.06} & \textbf{69.50} \\
\bottomrule
\end{tabularx}
\caption{One-epoch code results for Qwen3-4B-Base on UltraFeedback
(seed 1234, percentages).
Mean averages HumanEval+ and MBPP+ pass@1/pass@4.}
\label{tab:code_benchmark}
\end{table}

\subsection{Task Utility of Margin Constraints}
Table~\ref{tab:main_benchmark} reports the reasoning screen, and
Table~\ref{tab:code_benchmark} reports the code evaluation, both at seed
1234 under the matched one-epoch protocol. SRP-SFT raises MATH pass@1 from
49.56 for LP-SFT to 54.56 and raises the code mean from 65.91 to 69.50. It
improves all four code metrics over LP-SFT and has the highest code mean
among the seven fine-tuned methods. ASFT and DFT still lead the individual
code pass@1 columns on each benchmark.

Across the three training seeds, SRP-SFT remains above CE and LP-SFT.
Its code mean is $69.48\pm0.02$ and its twelve-metric reasoning mean is
$44.96\pm0.55$, compared with $66.17\pm0.27$ and $43.10\pm0.46$
for LP-SFT. CE reaches $55.66\pm1.18$ and $35.96\pm0.10$ on the same
aggregates.

\begin{table}[!t]
\centering
\footnotesize
\setlength{\tabcolsep}{1.8pt}
\renewcommand{\arraystretch}{1.10}
\begin{tabularx}{\columnwidth}{@{}l*{6}{Y}@{}}
\toprule
Method & \multicolumn{3}{c}{Reasoning} & \multicolumn{3}{c}{Code} \\
\cmidrule(lr){2-4}\cmidrule(lr){5-7}
& \shortstack{seed\\1234} & \shortstack{seed\\42} & \shortstack{seed\\3407} & \shortstack{seed\\1234} & \shortstack{seed\\42} & \shortstack{seed\\3407} \\
\midrule
CE & 35.99 & 36.03 & 35.85 & 56.04 & 54.33 & 56.60 \\
GEM & 35.01 & 34.89 & 35.52 & 36.95 & 33.93 & 36.88 \\
DFT & 35.25 & 33.37 & 35.77 & 63.72 & 61.35 & 58.38 \\
ASFT & 37.64 & 40.51 & 37.46 & \secbest{66.74} & \secbest{66.37} & \secbest{67.87} \\
EAFT & 34.91 & 35.81 & 35.89 & 57.12 & 57.32 & 57.65 \\
LP-SFT & \secbest{43.50} & \secbest{43.20} & \secbest{42.60} & 65.91 & 66.14 & 66.45 \\
\textbf{SRP-SFT} & \textbf{45.11} & \textbf{45.41} & \textbf{44.35} & \textbf{69.50} & \textbf{69.47} & \textbf{69.48} \\
\bottomrule
\end{tabularx}
\caption{Three-seed reasoning and code means (percentages).
Reasoning averages twelve metrics from the same checkpoint.
Code averages four metrics.}
\label{tab:two_seed_summary}
\end{table}

\subsection{Comparing Preservation Targets}
\label{sec:ablation}
With data, initialization, and coefficients fixed, we replace the SRP hinge with alternatives: distribution matching (LP-SFT, ListNet), pairwise/listwise ranking losses, exact-gap regression, and distillation objectives (Table~\ref{tab:ranking_loss_performance}; Appendix~\ref{app:loss_ablation_full}).

\begin{table}[!htb]
\centering
\footnotesize
\setlength{\tabcolsep}{1.5pt}
\renewcommand{\arraystretch}{1.06}
\begin{tabularx}{\columnwidth}{@{}l*{4}{Y}@{}}
\toprule
Loss & MATH & HE+ & MBPP+ & Mean \\
\midrule
CE & 27.98 & 44.18 & 44.66 & 38.94 \\
LP-SFT & 49.56 & 57.90 & 51.87 & 53.11 \\
RankOnly ($\rho=0$) & 41.89 & 55.37 & 49.96 & 49.07 \\
RankNet & 46.41 & 59.09 & 52.63 & 52.71 \\
ListMLE & 52.15 & \secbest{61.71} & \secbest{56.18} & \secbest{56.68} \\
Spearman & 46.74 & 56.68 & 52.08 & 51.83 \\
ListNet$^{\dagger}$ & 49.33 & 58.45 & 52.34 & 53.37 \\
Margin-MSE & 48.68 & 61.07 & 54.72 & 54.82 \\
BiLD teacher-led & 50.11 & 58.54 & 53.72 & 54.12 \\
BiLD T+S & 45.17 & 55.58 & 51.28 & 50.68 \\
RLKD-style union & \textbf{55.31} & 58.57 & 52.94 & 55.61 \\
Kendall-$\tau$ & 43.96 & 59.48 & 51.97 & 51.81 \\
Smooth-SRP & 50.65 & 59.85 & 53.65 & 54.72 \\
\textbf{SRP-SFT} & \secbest{54.56} & \textbf{62.84} & \textbf{56.24} & \textbf{57.88} \\
\bottomrule
\end{tabularx}
\caption{Fixed-configuration loss screen (pass@1, percentages).
Mean averages MATH, HumanEval+, and MBPP+.
Code columns use the Table~\ref{tab:code_benchmark} evaluation.
$^{\dagger}$Same target as LP-SFT.
Definitions are in Appendix~\ref{app:loss_ablation_full}.}
\label{tab:ranking_loss_performance}
\end{table}

Task utility tracks preservation of the base candidate order: across
14 losses, the mean falls as rank flips and margin contraction rise
(Spearman $-0.87$ and $-0.84$; Table~\ref{tab:ranking_loss_diagnostics}).
Order alone is insufficient: RankOnly ($\rho=0$), without a margin target,
contracts margins by 0.41 logits and scores 49.07, below LP-SFT; $\rho=1$
cuts contraction to 0.11, giving the best mean, 57.88. Penalizing both
directions falls short: Margin-MSE has the fewest flips
(0.0380 versus 0.0565) but scores 54.82. SRP-SFT leads both code columns
and trails only the RLKD-style union on MATH (54.56 versus 55.31).

\subsection{Distributional Fidelity versus Margin Preservation}
\label{sec:mechanism}
On the 998-example held-out cache, SRP-SFT reduces mean margin
contraction from 0.268 for LP-SFT to 0.110 and reduces inversions from
7.41\% to 5.65\%. CE has contraction 0.600 and inversion 15.40\%. LP-SFT
is closer in local KL, at 0.049 versus 0.075 for SRP-SFT, consistent with
the mass-dependent threshold in Theorem~\ref{thm:srp_kl_threshold}.
On this cache, the lower local KL does not identify the checkpoint with
less contraction and fewer inversions.

Table~\ref{tab:sharpness} and Figure~\ref{fig:mechanism} plot the same cache.
SRP-SFT has smaller margin shortfalls and larger local distribution
changes. The cache is training-disjoint and contains no benchmark labels
(Appendix~\ref{app:cache_provenance}).

\begin{figure}[!tbp]
    \centering
    \includegraphics[width=\columnwidth]{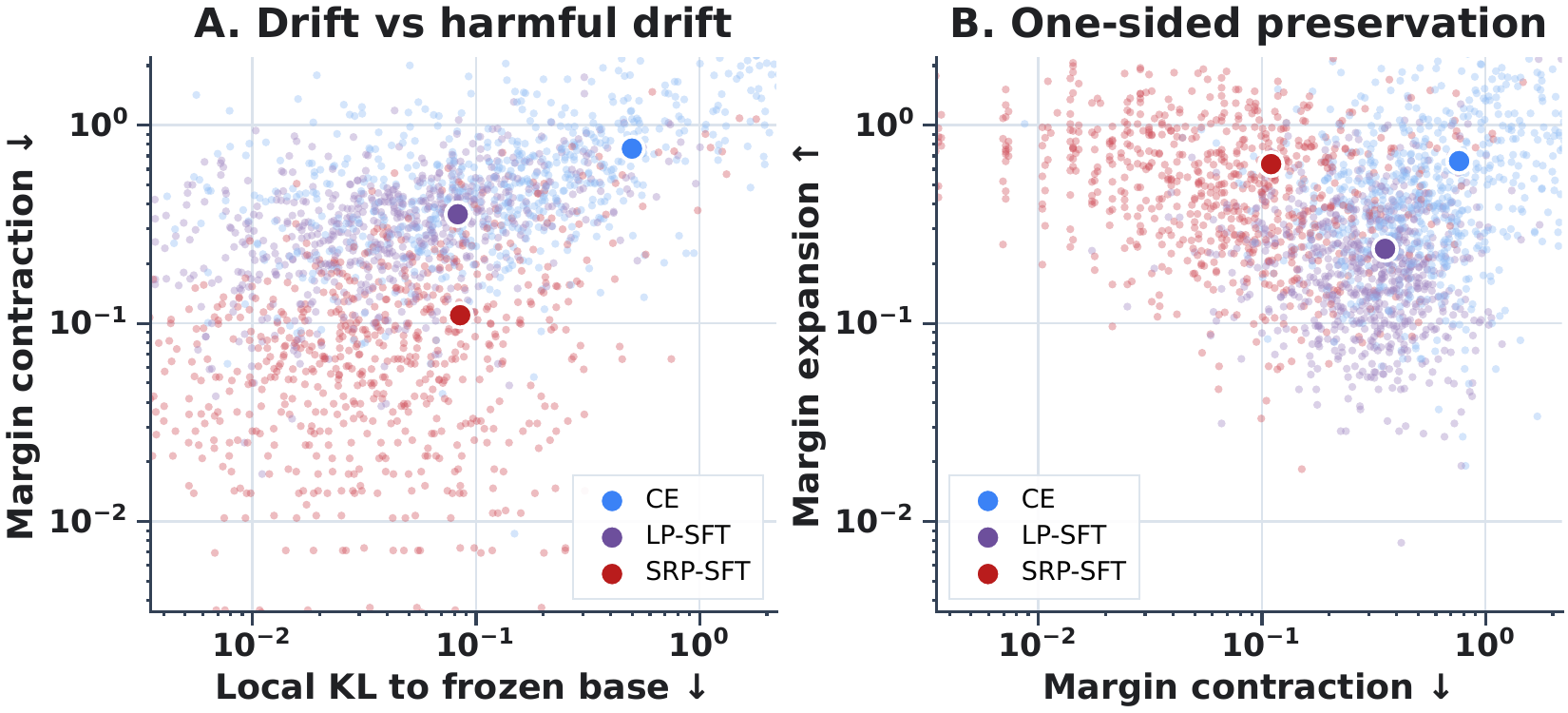}
    \caption{Mechanism diagnostics on the held-out cache.
    Distributional fidelity and margin preservation disagree on this cache:
    SRP-SFT has less contraction despite larger local KL.
    Small points are tokens. Large markers are pair-weighted means.}
    \label{fig:mechanism}
\end{figure}


\subsection{Preservation Strength versus Task Utility}
\label{sec:parameter_sensitivity_main}
The lower-bound formulation has three controls: the requested margin
$\rho$, the penalty weight $\mu$, and candidate coverage $K$. In the
$\rho$ sweep, the aggregate score rises from 47.05 at $\rho=0$ to 61.56 at
$\rho=2.5$ and then declines. From $\rho=1$ to $\rho=2.5$, the score
increases from 57.05 to 61.56 while Flip rises from 0.0565 to 0.0937 and
contraction rises from 0.1098 to 0.1891 (Table~\ref{tab:sensitivity}).
Over this interval, the task score and the realized margin statistics move
in opposite directions.

Increasing $\mu$ beyond 3 lowers the task score while reducing Flip, and
the $K$ sweep peaks at $K=5$. The $\rho$ sweep fixes $\mu=1,K=10$, whereas
the $\mu$ and $K$ sweeps use $\rho=0.5$. The reported one-factor peaks are
$\rho=2.5$, $K=5$, and $\mu=3$. The pre-fixed main setting remains
$K=10,\rho=1,\mu=1$ (Appendix~\ref{app:sensitivity_ablation}).

\begin{figure}[!tbp]
\centering
\includegraphics[width=\columnwidth,trim=0 17 0 0,clip]{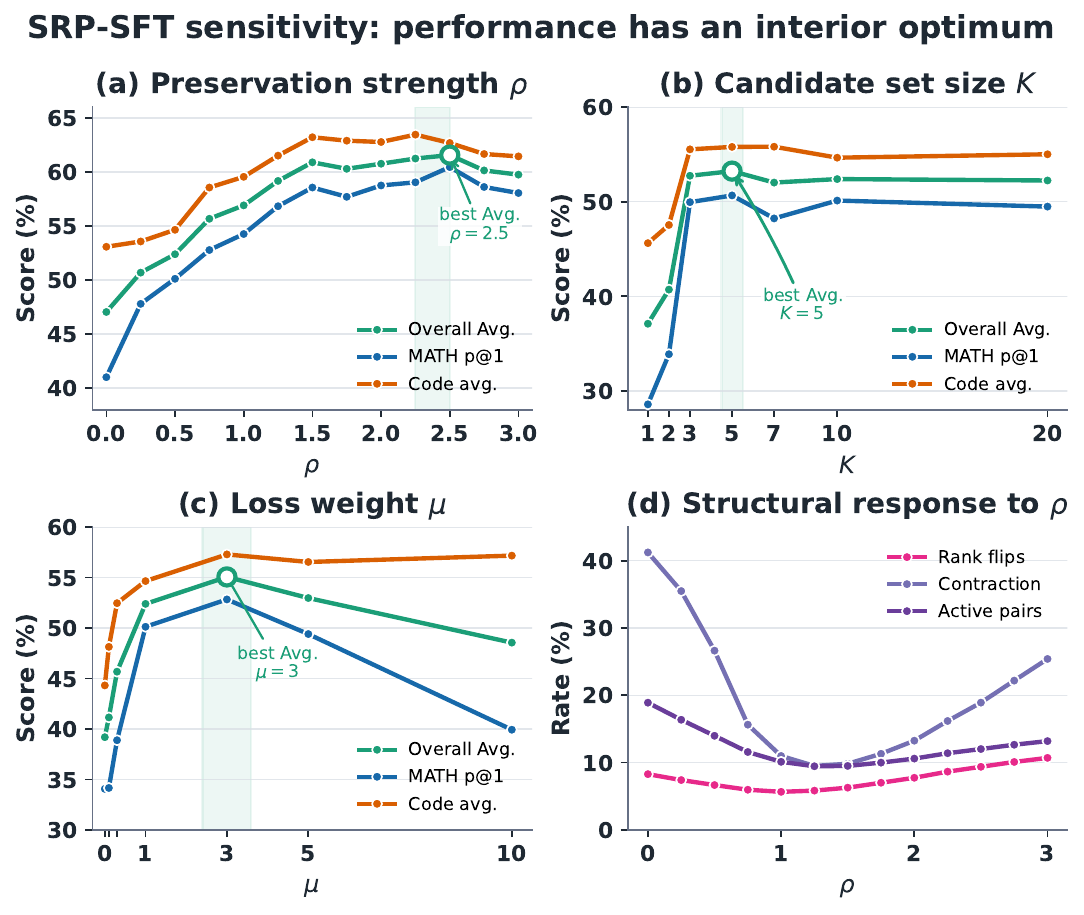}
\caption{Preservation demand and task utility need not move together.
Each panel varies one of $\rho$, $\mu$, and $K$. Bands mark the best score.
The score averages MATH pass@1 and mean code pass@1.}
    \label{fig:parameter_sensitivity_sweeps}
\end{figure}

\subsection{Can Global Scaling Recover the Gains?}
\label{sec:temperature}
For LP-SFT, we fit $\alpha=1.0993$ on 998 held-out examples,
without benchmark labels, to match SRP-SFT's mean full-vocabulary
entropy (Table~\ref{tab:tempmatch}). This single scalar changes
sharpness but leaves each token's candidate ranking unchanged.

\begin{table}[!tbp]
\centering
\scriptsize
\setlength{\tabcolsep}{1.1pt}
\renewcommand{\arraystretch}{1.10}
\begin{tabularx}{\columnwidth}{@{}l*{7}{Y}@{}}
\toprule
Method & $\alpha$ & \multicolumn{2}{c}{MATH} & \multicolumn{2}{c}{HE+} & \multicolumn{2}{c}{MBPP+} \\
\cmidrule(lr){3-4}\cmidrule(lr){5-6}\cmidrule(lr){7-8}
& & p@1 & p@4 & p@1 & p@4 & p@1 & p@4 \\
\midrule
CE & 1 & 27.98 & 54.91 & 44.18 & 69.56 & 44.66 & 65.79 \\
CE-TM & 1.081 & 28.98 & 55.87 & 45.88 & 70.44 & 45.73 & 65.33 \\
LP-SFT & 1 & 49.56 & 77.52 & 57.90 & 81.67 & 51.87 & 72.23 \\
LP-TM & 1.099 & 51.96 & 78.08 & 59.82 & 82.39 & 54.84 & 73.06 \\
SRP-SFT & 1 & 54.56 & 80.50 & 62.84 & 83.87 & 56.24 & 75.06 \\
\bottomrule
\end{tabularx}
\caption{Entropy-matched temperature control (seed 1234, percentages).
LP-TM and CE-TM are calibrated to the SRP-SFT held-out entropy.}
\label{tab:tempmatch}
\end{table}

Global scaling recovers less than half of SRP-SFT's gain, so the gain is
not a byproduct of a sharper output distribution. Entropy matching raises
LP-SFT's MATH pass@1 from 49.56 to 51.96 and its code mean from 65.91 to
67.53, closing 48\% and 45\% of the gap to SRP-SFT (54.56 and 69.50). CE
gains 1.00 MATH and 0.81 code points. At seed 42, the fitted LP scale is
1.0989 and the code mean is 67.90 (Table~\ref{tab:seed42_matched_available}).

\subsection{Transfer Across Backbones}
\label{sec:cross_backbone}
The six additional backbones are trained separately per domain:
mathematics SFT on 96{,}949 filtered OpenMathReasoning examples (758
updates) and code SFT on 95{,}966 filtered rStar-Coder examples (750
updates), with 4{,}096-token sequences, one epoch, and seed 1234.
MATH-500 and GSM8K evaluate the mathematics checkpoints, and HumanEval+
and MBPP+ evaluate the code checkpoints. Within each backbone and domain,
all methods share the data and budget.

The advantage of SRP-SFT over LP-SFT extends from UltraFeedback to
new architectures and domain-specific data and is steadiest on code.
SRP-SFT leads LP-SFT in 18 of 24 backbone--benchmark pairs,
including 11 of 12 code pairs; HumanEval+ rises on every backbone
($+0.55$ to $+4.09$). Mathematics is more backbone-dependent: SRP-SFT
leads MATH-500 on both Qwen bases, Ministral, and SmolLM3, but
trails on Gemma (61.59 versus 63.94) and Granite.
Figure~\ref{fig:cross_backbone_robustness} shows SRP-SFT minus the
CE/LP-SFT mean.
Table~\ref{tab:cross_backbone_delta_lp} gives SRP-minus-LP
differences.

\begin{figure}[!t]
    \centering
    \includegraphics[width=\columnwidth]{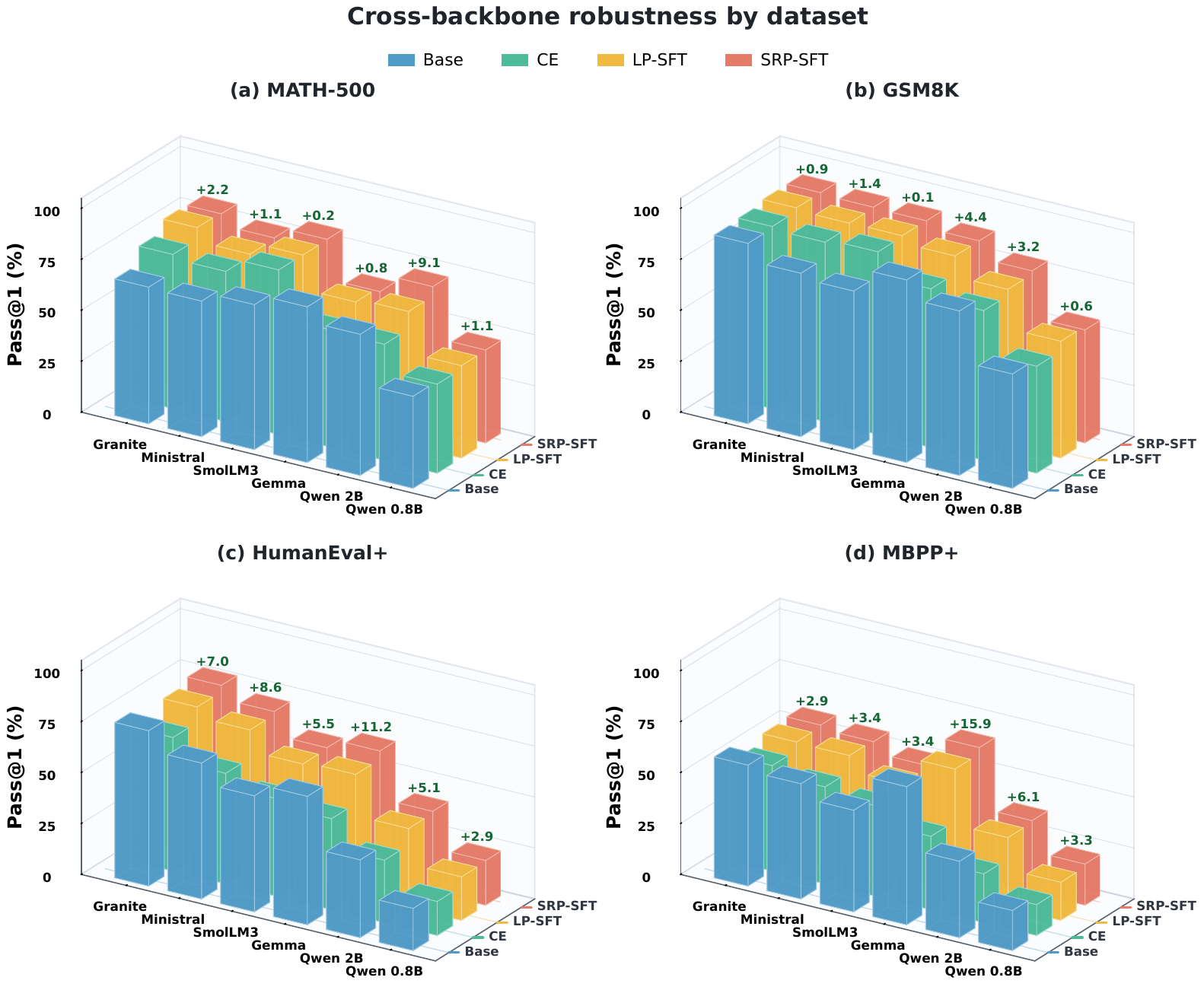}
    \caption{Additional-backbone pass@1 on MATH-500, GSM8K, HumanEval+,
    and MBPP+. Math and code use separate domain-trained checkpoints.
    $\Delta$ is SRP-SFT minus the mean of CE and LP-SFT.}
    \label{fig:cross_backbone_robustness}
\end{figure}

\subsection{Convergence Behavior}
\label{sec:convergence}

\begin{figure}[!tbp]
    \centering
    \includegraphics[width=\columnwidth]{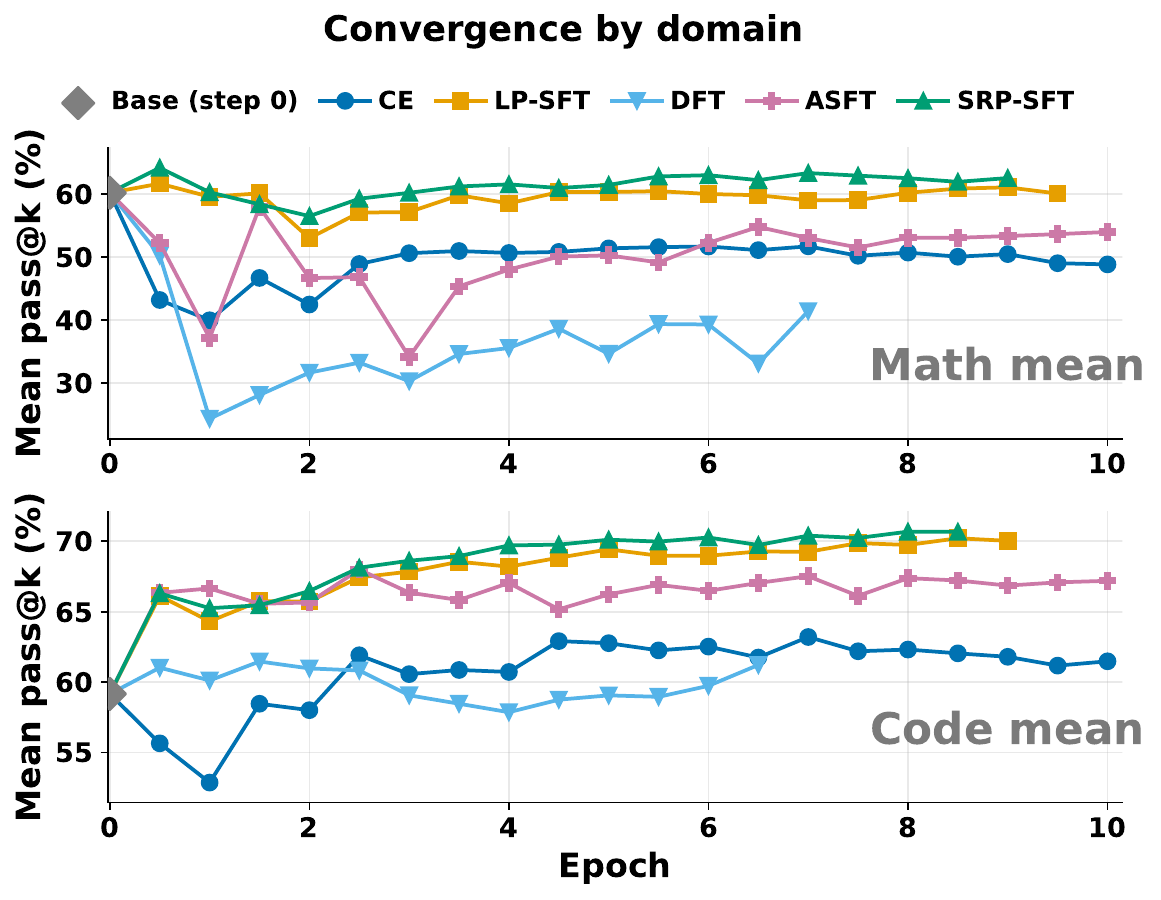}
    \caption{Ten-epoch trajectories with SRP $\rho=0.5$.
    MATH averages pass@1 and pass@4. Code averages pass@1 and pass@4 on both benchmarks.}
    \label{fig:convergence_means}
\end{figure}

One Qwen3-4B run continues for ten epochs, evaluated every half epoch. LP-SFT uses $K=10,\mu=1$, and SRP-SFT uses $K=10,\rho=0.5,\mu=1$. The MATH mean averages MATH-500 pass@1 and pass@4. The code mean averages HumanEval+ and MBPP+ pass@1 and pass@4. The step-0 Base anchor evaluates the untouched checkpoint. Its MATH-500 manifest and sampling protocol match Table~\ref{tab:main_benchmark}: the same data revision, all 500 test problems, the same prompt, temperature 0.6, top-$p$ 1, 16 samples, and a 4{,}096-token cap. The separate draw gives pass@1 48.43, pass@4 72.06, and mean 60.24. The main-table scores 48.11 and 71.23 average to 59.67. The extended-training curve uses this as its step-0 anchor. Its code mean, 59.17, uses 768-token generations in a 2{,}048-token context. Table~\ref{tab:code_benchmark} uses the 4{,}096/8{,}192 protocol, and its four Base scores average to 58.92.

The $\rho=0.5$ SRP run declines early, then keeps MATH near Base while
maintaining code gains later in training. It does not lead every early
checkpoint. Its last MATH and code measurements occur at epochs 10.0 and
9.5 (Tables~\ref{tab:convergence_math_mean} and
\ref{tab:convergence_code_mean}). The one-epoch headline uses $\rho=1$, unlike this run.

\section{Conclusion}

SRP-SFT is a soft hinge on selected non-target margins at fixed
teacher-forced prefixes. Matched seeds, loss screens, sweeps,
entropy matching, six backbones, and ten-epoch training show where this
margin target improves reasoning and code scores over CE and LP-SFT.

\section*{Limitations}
SRP-SFT is trained on models that already follow a chat template and already
have usable downstream ability. Where a base release supplies that template,
the study uses the base checkpoint. Where it does not, the study uses the
instruction-tuned release. Training and evaluation both rely on this chat
formatting, so the setup does not cover a weak model with no chat-template
training.

The trained models are at most 4B parameters, including the main study and
the additional backbones. No larger model is trained. The one-epoch schedule
and the reported math and code scores are measurements at this scale.

\FloatBarrier
\bibliography{references}

@inproceedings{cui2023ultrafeedback,
  title = {{UltraFeedback}: Boosting Language Models with Scaled {AI} Feedback},
  author = {Cui, Ganqu and Yuan, Lifan and Ding, Ning and Yao, Guanming and He, Bingxiang and Zhu, Wei and Ni, Yuan and Xie, Guotong and Xie, Ruobing and Lin, Yankai and Liu, Zhiyuan and Sun, Maosong},
  booktitle = {Proceedings of the 41st International Conference on Machine Learning},
  series = {Proceedings of Machine Learning Research},
  volume = {235},
  pages = {9722--9744},
  year = {2024},
  publisher = {PMLR},
  url = {https://proceedings.mlr.press/v235/cui24f.html}
}

@article{yang2025qwen3,
  title={Qwen3 technical report},
  author={Yang, An and Li, Anfeng and Yang, Baosong and Zhang, Beichen and Hui, Binyuan and Zheng, Bo and Yu, Bowen and Gao, Chang and Huang, Chengen and Lv, Chenxu and others},
  journal={arXiv preprint arXiv:2505.09388},
  year={2025},
  url = {https://arxiv.org/abs/2505.09388}
}

@inproceedings{wu2025generalization,
  title={On the generalization of {SFT}: A reinforcement learning perspective with reward rectification},
  author={Wu, Yongliang and Zhou, Yizhou and Ziheng, Zhou and Peng, Yingzhe and Ye, Xinyu and Hu, Xinting and Zhu, Wenbo and Qi, Lu and Yang, Ming-Hsuan and Yang, Xu},
  booktitle={The Fourteenth International Conference on Learning Representations},
  year={2026}
}

@article{diao2026entropy,
  title={Entropy-Adaptive Fine-Tuning: Resolving Confident Conflicts to Mitigate Forgetting},
  author={Diao, Muxi and Yang, Lele and Gong, Wuxuan and Zhang, Yutong and Yan, Zhonghao and Han, Yufei and Liang, Kongming and Xu, Weiran and Ma, Zhanyu},
  journal={arXiv preprint arXiv:2601.02151},
  year={2026},
  url = {https://arxiv.org/abs/2601.02151}
}

@article{lpsft,
  title = {{LP-SFT: Keeping Plausible Alternatives Alive in Supervised Fine-Tuning}},
  author = {Wang, Yueyang and Bi, Baolong and Lu, Shuo and Zhang, Jingyuan and Shi, Jiajun and Zhang, Yuxuan},
  journal = {arXiv preprint arXiv:2607.04733},
  year = {2026},
  url = {https://arxiv.org/abs/2607.04733v4}
}

@article{liu2023your,
  title={Is your code generated by chatgpt really correct? rigorous evaluation of large language models for code generation},
  author={Liu, Jiawei and Xia, Chunqiu Steven and Wang, Yuyao and Zhang, Lingming},
  journal={Advances in neural information processing systems},
  volume={36},
  pages={21558--21572},
  year={2023}
}

@article{hendrycks2021measuring,
  title={Measuring mathematical problem solving with the math dataset},
  author={Hendrycks, Dan and Burns, Collin and Kadavath, Saurav and Arora, Akul and Basart, Steven and Tang, Eric and Song, Dawn and Steinhardt, Jacob},
  journal={arXiv preprint arXiv:2103.03874},
  year={2021},
  url = {https://arxiv.org/abs/2103.03874}
}

@article{hendrycks2020measuring,
  title={Measuring massive multitask language understanding},
  author={Hendrycks, Dan and Burns, Collin and Basart, Steven and Zou, Andy and Mazeika, Mantas and Song, Dawn and Steinhardt, Jacob},
  journal={arXiv preprint arXiv:2009.03300},
  year={2020},
  url = {https://arxiv.org/abs/2009.03300}
}

@article{chen2021evaluating,
  title={Evaluating large language models trained on code},
  author={Chen, Mark and Tworek, Jerry and Jun, Heewoo and Yuan, Qiming and Pinto, Henrique Ponde De Oliveira and Kaplan, Jared and Edwards, Harri and Burda, Yuri and Joseph, Nicholas and Brockman, Greg and others},
  journal={arXiv preprint arXiv:2107.03374},
  year={2021},
  url = {https://arxiv.org/abs/2107.03374}
}

@inproceedings{li2025preserving,
  title={Preserving diversity in supervised fine-tuning of large language models},
  author={Li, Ziniu and Chen, Congliang and Xu, Tian and Qin, Zeyu and Xiao, Jiancong and Luo, Zhi-Quan and Sun, Ruoyu},
  booktitle={International Conference on Learning Representations},
  volume={2025},
  pages={66127--66154},
  year={2025}
}

@inproceedings{zhu2025anchored,
  title={Anchored Supervised Fine-Tuning},
  author={Zhu, He and Su, Junyou and Lai, Peng and Ma, Ren and Zhang, Wenjia and Yang, Linyi and Chen, Guanhua},
  booktitle={The Fourteenth International Conference on Learning Representations},
  year={2026}
}

@article{luo2025empirical,
  title={An empirical study of catastrophic forgetting in large language models during continual fine-tuning},
  author={Luo, Yun and Yang, Zhen and Meng, Fandong and Li, Yafu and Zhou, Jie and Zhang, Yue},
  journal={IEEE Transactions on Audio, Speech and Language Processing},
  year={2025},
  publisher={IEEE}
}

@article{hinton2015distilling,
  title={Distilling the knowledge in a neural network},
  author={Hinton, Geoffrey and Vinyals, Oriol and Dean, Jeff},
  journal={arXiv preprint arXiv:1503.02531},
  year={2015},
  url = {https://arxiv.org/abs/1503.02531}
}

@inproceedings{gu2024minillm,
  title={Minillm: Knowledge distillation of large language models},
  author={Gu, Yuxian and Dong, Li and Wei, Furu and Huang, Minlie},
  booktitle={International Conference on Learning Representations},
  volume={2024},
  pages={32694--32717},
  year={2024}
}

@inproceedings{agarwal2024policy,
  title={On-policy distillation of language models: Learning from self-generated mistakes},
  author={Agarwal, Rishabh and Vieillard, Nino and Zhou, Yongchao and Stanczyk, Piotr and Ramos Garea, Sabela and Geist, Matthieu and Bachem, Olivier},
  booktitle={International Conference on Learning Representations},
  volume={2024},
  pages={21246--21263},
  year={2024}
}

@book{theory-convex,
  author    = {Stephen Boyd and Lieven Vandenberghe},
  title     = {Convex Optimization},
  publisher = {Cambridge University Press},
  year      = {2004},
  url       = {https://stanford.edu/~boyd/cvxbook/}
}

@inproceedings{theory-ntk,
  author    = {Thang Doan and Mehdi Abbana Bennani and Bogdan Mazoure and Guillaume Rabusseau and Pierre Alquier},
  title     = {A Theoretical Analysis of Catastrophic Forgetting through the {NTK} Overlap Matrix},
  booktitle = {Proceedings of the 24th International Conference on Artificial Intelligence and Statistics},
  series    = {Proceedings of Machine Learning Research},
  volume    = {130},
  pages     = {1072--1080},
  publisher = {PMLR},
  year      = {2021},
  url       = {https://proceedings.mlr.press/v130/doan21a.html}
}

@inproceedings{theory-robustness,
  author    = {Matthias Hein and Maksym Andriushchenko},
  title     = {Formal Guarantees on the Robustness of a Classifier against Adversarial Manipulation},
  booktitle = {Advances in Neural Information Processing Systems},
  volume    = {30},
  year      = {2017},
  url       = {https://arxiv.org/abs/1705.08475}
}

@article{theory-hoeffding,
  author  = {Wassily Hoeffding},
  title   = {Probability Inequalities for Sums of Bounded Random Variables},
  journal = {Journal of the American Statistical Association},
  volume  = {58},
  number  = {301},
  pages   = {13--30},
  year    = {1963},
  doi     = {10.1080/01621459.1963.10500830},
  url     = {https://doi.org/10.1080/01621459.1963.10500830}
}

@inproceedings{theory-bild,
  author    = {Minchong Li and Feng Zhou and Xiaohui Song},
  title     = {{BiLD}: Bi-directional Logits Difference Loss for Large Language Model Distillation},
  booktitle = {Proceedings of the 31st International Conference on Computational Linguistics},
  pages     = {1168--1182},
  publisher = {Association for Computational Linguistics},
  year      = {2025},
  url       = {https://aclanthology.org/2025.coling-main.78/}
}

@inproceedings{theory-gem,
  author    = {David Lopez-Paz and Marc'Aurelio Ranzato},
  title     = {Gradient Episodic Memory for Continual Learning},
  booktitle = {Advances in Neural Information Processing Systems},
  volume    = {30},
  year      = {2017},
  url       = {https://arxiv.org/abs/1706.08840}
}

@article{theory-penalty,
  author  = {Hua Zhou and Kenneth Lange},
  title   = {Path Following in the Exact Penalty Method of Convex Programming},
  journal = {arXiv preprint arXiv:1201.3593},
  year    = {2012},
  url     = {https://arxiv.org/abs/1201.3593}
}

@inproceedings{burges2005ranknet,
  author    = {Christopher J. C. Burges and Tal Shaked and Erin Renshaw and Ari Lazier and Matt Deeds and Nicole Hamilton and Gregory N. Hullender},
  title     = {Learning to Rank Using Gradient Descent},
  booktitle = {Proceedings of the 22nd International Conference on Machine Learning},
  pages     = {89--96},
  year      = {2005},
  doi       = {10.1145/1102351.1102363},
  url       = {https://mlanthology.org/icml/2005/burges2005icml-learning/}
}

@inproceedings{cao2007listnet,
  author    = {Zhe Cao and Tao Qin and Tie-Yan Liu and Ming-Feng Tsai and Hang Li},
  title     = {Learning to Rank: From Pairwise Approach to Listwise Approach},
  booktitle = {Proceedings of the 24th International Conference on Machine Learning},
  pages     = {129--136},
  year      = {2007},
  doi       = {10.1145/1273496.1273513},
  url       = {https://mlanthology.org/icml/2007/cao2007icml-learning/}
}

@inproceedings{xia2008listmle,
  author    = {Fen Xia and Tie-Yan Liu and Jue Wang and Wensheng Zhang and Hang Li},
  title     = {Listwise Approach to Learning to Rank: Theory and Algorithm},
  booktitle = {Proceedings of the 25th International Conference on Machine Learning},
  pages     = {1192--1199},
  year      = {2008},
  doi       = {10.1145/1390156.1390306},
  url       = {https://mlanthology.org/icml/2008/xia2008icml-listwise/}
}

@article{spearman1904,
  author  = {Charles Spearman},
  title   = {The Proof and Measurement of Association between Two Things},
  journal = {The American Journal of Psychology},
  volume  = {15},
  number  = {1},
  pages   = {72--101},
  year    = {1904},
  doi     = {10.2307/1412159},
  url     = {https://www.jstor.org/stable/1412159}
}

@inproceedings{lfu,
  author    = {Yunhun Nam and Jaehyung Kim and Jongheon Jeong},
  title     = {Learning from the Undesirable: Robust Adaptation of Language Models Without Forgetting},
  booktitle = {Proceedings of the AAAI Conference on Artificial Intelligence},
  volume    = {40},
  pages     = {32537--32545},
  year      = {2026},
  doi       = {10.1609/aaai.v40i38.40530},
  url       = {https://ojs.aaai.org/index.php/AAAI/article/view/40530}
}

@article{marginmse,
  author  = {Sebastian Hofst{\"a}tter and Sophia Althammer and Michael Schr{\"o}der and Mete Sertkan and Allan Hanbury},
  title   = {Improving Efficient Neural Ranking Models with Cross-Architecture Knowledge Distillation},
  journal = {arXiv preprint arXiv:2010.02666},
  year    = {2020},
  url     = {https://arxiv.org/abs/2010.02666}
}

@inproceedings{rlkd,
  author    = {Tianyu Peng and Jiajun Zhang},
  title     = {Enhancing Knowledge Distillation of Large Language Models through Efficient Multi-Modal Distribution Alignment},
  booktitle = {Proceedings of the 31st International Conference on Computational Linguistics},
  pages     = {2478--2496},
  address   = {Abu Dhabi, UAE},
  publisher = {Association for Computational Linguistics},
  year      = {2025},
  url       = {https://aclanthology.org/2025.coling-main.169/}
}

@inproceedings{kendallkd,
  author    = {Yuchen Guan and Runxi Cheng and Kang Liu and Chun Yuan},
  title     = {Enhancing Logits Distillation with Plug\&Play Kendall's $\tau$ Ranking Loss},
  booktitle = {Proceedings of the 42nd International Conference on Machine Learning},
  series    = {Proceedings of Machine Learning Research},
  volume    = {267},
  pages     = {20537--20552},
  publisher = {PMLR},
  year      = {2025},
  url       = {https://proceedings.mlr.press/v267/guan25a.html}
}

@article{li2018lwf,
  title = {{Learning without Forgetting}},
  author = {Li, Zhizhong and Hoiem, Derek},
  journal = {IEEE Transactions on Pattern Analysis and Machine Intelligence},
  year = {2018},
  volume = {40},
  number = {12},
  pages = {2935--2947},
  doi = {10.1109/tpami.2017.2773081},
  url = {https://doi.org/10.1109/tpami.2017.2773081}
}

@article{kirkpatrick2017ewc,
  title = {{Overcoming catastrophic forgetting in neural networks}},
  author = {Kirkpatrick, James and Pascanu, Razvan and Rabinowitz, Neil and Veness, Joel and Desjardins, Guillaume and Rusu, Andrei A. and Milan, Kieran and Quan, John and Ramalho, Tiago and Grabska-Barwinska, Agnieszka and Hassabis, Demis and Clopath, Claudia and Kumaran, Dharshan and Hadsell, Raia},
  journal = {Proceedings of the National Academy of Sciences},
  year = {2017},
  volume = {114},
  number = {13},
  pages = {3521--3526},
  doi = {10.1073/pnas.1611835114},
  url = {https://doi.org/10.1073/pnas.1611835114}
}

@inproceedings{li2018l2sp,
  title = {{Explicit Inductive Bias for Transfer Learning with Convolutional Networks}},
  author = {Li, Xuhong and Grandvalet, Yves and Davoine, Franck},
  booktitle = {Proceedings of the 35th International Conference on Machine Learning},
  series = {Proceedings of Machine Learning Research},
  volume = {80},
  pages = {2825--2834},
  year = {2018},
  publisher = {PMLR},
  url = {https://proceedings.mlr.press/v80/li18a.html}
}

@article{jin2026rpsft,
  title = {{Rotation-Preserving Supervised Fine-Tuning}},
  author = {Jin, Hangzhan and Ni, Tianwei and Li, Lu and Bacon, Pierre-Luc and Hamdaqa, Mohammad and Precup, Doina},
  journal = {arXiv preprint arXiv:2605.10973},
  year = {2026},
  url = {https://arxiv.org/abs/2605.10973}
}

@inproceedings{zhao2022dkd,
  title = {{Decoupled Knowledge Distillation}},
  author = {Zhao, Borui and Cui, Quan and Song, Renjie and Qiu, Yiyu and Liang, Jiajun},
  booktitle = {Proceedings of the IEEE/CVF Conference on Computer Vision and Pattern Recognition},
  year = {2022},
  pages = {11953--11962},
  url = {https://openaccess.thecvf.com/content/CVPR2022/html/Zhao_Decoupled_Knowledge_Distillation_CVPR_2022_paper.html}
}

@inproceedings{huang2022dist,
  title = {{Knowledge Distillation from A Stronger Teacher}},
  author = {Huang, Tao and You, Shan and Wang, Fei and Qian, Chen and Xu, Chang},
  booktitle = {Advances in Neural Information Processing Systems},
  year = {2022},
  volume = {35},
  pages = {33716--33727},
  url = {https://proceedings.neurips.cc/paper_files/paper/2022/hash/da669dfd3c36c93905a17ddba01eef06-Abstract-Conference.html}
}

@inproceedings{reddi2021rankdistil,
  title = {{RankDistil: Knowledge Distillation for Ranking}},
  author = {Reddi, Sashank and Kumar Pasumarthi, Rama and Menon, Aditya and Singh Rawat, Ankit and Yu, Felix and Kim, Seungyeon and Veit, Andreas and Kumar, Sanjiv},
  booktitle = {Proceedings of the 24th International Conference on Artificial Intelligence and Statistics},
  series = {Proceedings of Machine Learning Research},
  volume = {130},
  pages = {2368--2376},
  year = {2021},
  publisher = {PMLR},
  url = {https://proceedings.mlr.press/v130/reddi21a.html}
}

@article{moshkov2025openmathreasoning,
  title = {{AIMO-2 Winning Solution: Building State-of-the-Art Mathematical Reasoning Models with OpenMathReasoning dataset}},
  author = {Moshkov, Ivan and Hanley, Darragh and Sorokin, Ivan and Toshniwal, Shubham and Henkel, Christof and Schifferer, Benedikt and Du, Wei and Gitman, Igor},
  journal = {arXiv preprint arXiv:2504.16891},
  year = {2025},
  url = {https://arxiv.org/abs/2504.16891}
}

@article{liu2025rstarcoder,
  title = {{rStar-Coder: Scaling Competitive Code Reasoning with a Large-Scale Verified Dataset}},
  author = {Liu, Yifei and Zhang, Li Lyna and Zhu, Yi and Dong, Bingcheng and Zhou, Xudong and Shang, Ning and Yang, Fan and Yang, Mao},
  journal = {arXiv preprint arXiv:2505.21297},
  year = {2025},
  url = {https://arxiv.org/abs/2505.21297}
}

@inproceedings{li2024gsmplus,
    title = "{GSM}-Plus: A Comprehensive Benchmark for Evaluating the Robustness of {LLM}s as Mathematical Problem Solvers",
    author = "Li, Qintong  and
      Cui, Leyang  and
      Zhao, Xueliang  and
      Kong, Lingpeng  and
      Bi, Wei",
    editor = "Ku, Lun-Wei  and
      Martins, Andre  and
      Srikumar, Vivek",
    booktitle = "Proceedings of the 62nd Annual Meeting of the Association for Computational Linguistics (Volume 1: Long Papers)",
    month = aug,
    year = "2024",
    address = "Bangkok, Thailand",
    publisher = "Association for Computational Linguistics",
    url = "https://aclanthology.org/2024.acl-long.163/",
    doi = "10.18653/v1/2024.acl-long.163",
    pages = "2961--2984"
}

@inproceedings{suzgun2023bbh,
    title = "Challenging {BIG}-Bench Tasks and Whether Chain-of-Thought Can Solve Them",
    author = {Suzgun, Mirac  and
      Scales, Nathan  and
      Sch{\"a}rli, Nathanael  and
      Gehrmann, Sebastian  and
      Tay, Yi  and
      Chung, Hyung Won  and
      Chowdhery, Aakanksha  and
      Le, Quoc  and
      Chi, Ed  and
      Zhou, Denny  and
      Wei, Jason},
    editor = "Rogers, Anna  and
      Boyd-Graber, Jordan  and
      Okazaki, Naoaki",
    booktitle = "Findings of the Association for Computational Linguistics: ACL 2023",
    month = jul,
    year = "2023",
    address = "Toronto, Canada",
    publisher = "Association for Computational Linguistics",
    url = "https://aclanthology.org/2023.findings-acl.824/",
    doi = "10.18653/v1/2023.findings-acl.824",
    pages = "13003--13051"
}

@inproceedings{zellers2019hellaswag,
    title = "{H}ella{S}wag: Can a Machine Really Finish Your Sentence?",
    author = "Zellers, Rowan  and
      Holtzman, Ari  and
      Bisk, Yonatan  and
      Farhadi, Ali  and
      Choi, Yejin",
    editor = "Korhonen, Anna  and
      Traum, David  and
      M{\`a}rquez, Llu{\'i}s",
    booktitle = "Proceedings of the 57th Annual Meeting of the Association for Computational Linguistics",
    month = jul,
    year = "2019",
    address = "Florence, Italy",
    publisher = "Association for Computational Linguistics",
    url = "https://aclanthology.org/P19-1472/",
    doi = "10.18653/v1/P19-1472",
    pages = "4791--4800"
}

@inproceedings{he2024olympiadbench,
    title = "{O}lympiad{B}ench: A Challenging Benchmark for Promoting {AGI} with Olympiad-Level Bilingual Multimodal Scientific Problems",
    author = "He, Chaoqun  and
      Luo, Renjie  and
      Bai, Yuzhuo  and
      Hu, Shengding  and
      Thai, Zhen  and
      Shen, Junhao  and
      Hu, Jinyi  and
      Han, Xu  and
      Huang, Yujie  and
      Zhang, Yuxiang  and
      Liu, Jie  and
      Qi, Lei  and
      Liu, Zhiyuan  and
      Sun, Maosong",
    editor = "Ku, Lun-Wei  and
      Martins, Andre  and
      Srikumar, Vivek",
    booktitle = "Proceedings of the 62nd Annual Meeting of the Association for Computational Linguistics (Volume 1: Long Papers)",
    month = aug,
    year = "2024",
    address = "Bangkok, Thailand",
    publisher = "Association for Computational Linguistics",
    url = "https://aclanthology.org/2024.acl-long.211/",
    doi = "10.18653/v1/2024.acl-long.211",
    pages = "3828--3850"
}

@article{liu2023logiqa2,
  title = {{LogiQA 2.0—An Improved Dataset for Logical Reasoning in Natural Language Understanding}},
  author = {Liu, Hanmeng and Liu, Jian and Cui, Leyang and Teng, Zhiyang and Duan, Nan and Zhou, Ming and Zhang, Yue},
  journal = {IEEE/ACM Transactions on Audio, Speech, and Language Processing},
  year = {2023},
  volume = {31},
  pages = {2947--2962},
  doi = {10.1109/taslp.2023.3293046},
  url = {https://doi.org/10.1109/taslp.2023.3293046}
}

@article{clark2018arc,
  title = {{Think you have Solved Question Answering? Try ARC, the AI2 Reasoning Challenge}},
  author = {Clark, Peter and Cowhey, Isaac and Etzioni, Oren and Khot, Tushar and Sabharwal, Ashish and Schoenick, Carissa and Tafjord, Oyvind},
  journal = {arXiv preprint arXiv:1803.05457},
  year = {2018},
  url = {https://arxiv.org/abs/1803.05457}
}

@article{austin2021mbpp,
  title = {{Program Synthesis with Large Language Models}},
  author = {Austin, Jacob and Odena, Augustus and Nye, Maxwell and Bosma, Maarten and Michalewski, Henryk and Dohan, David and Jiang, Ellen and Cai, Carrie and Terry, Michael and Le, Quoc and Sutton, Charles},
  journal = {arXiv preprint arXiv:2108.07732},
  year = {2021},
  url = {https://arxiv.org/abs/2108.07732}
}

@article{cobbe2021gsm8k,
  title = {{Training Verifiers to Solve Math Word Problems}},
  author = {Cobbe, Karl and Kosaraju, Vineet and Bavarian, Mohammad and Chen, Mark and Jun, Heewoo and Kaiser, Lukasz and Plappert, Matthias and Tworek, Jerry and Hilton, Jacob and Nakano, Reiichiro and Hesse, Christopher and Schulman, John},
  journal = {arXiv preprint arXiv:2110.14168},
  year = {2021},
  url = {https://arxiv.org/abs/2110.14168}
}

@misc{qwen2026qwen35,
    title  = {{Qwen3.5}: Towards Native Multimodal Agents},
    author = {{Qwen Team}},
    month  = {February},
    year   = {2026},
    url    = {https://qwen.ai/blog?id=qwen3.5}
}

@misc{ibm2026granite41,
  title = {{Granite-4.1-3B Model Card}},
  author = {{IBM Granite Team}},
  year = {2026},
  howpublished = {Hugging Face model card},
  url = {https://huggingface.co/ibm-granite/granite-4.1-3b}
}

@misc{mistral2025ministral3,
  title = {{Ministral 3 3B Instruct Model Card}},
  author = {{Mistral AI}},
  year = {2025},
  howpublished = {Hugging Face model card},
  url = {https://huggingface.co/mistralai/Ministral-3-3B-Instruct-2512}
}

@article{google2025gemma3,
  title = {{Gemma 3 Technical Report}},
  author = {{Gemma Team} and others},
  journal = {arXiv preprint arXiv:2503.19786},
  year = {2025},
  url = {https://arxiv.org/abs/2503.19786}
}

@misc{hf2025smollm3,
  title={{SmolLM3: smol, multilingual, long-context reasoner}},
  author={Bakouch, Elie and Ben Allal, Loubna and Lozhkov, Anton and Tazi, Nouamane and Tunstall, Lewis and Patiño, Carlos Miguel and Beeching, Edward and Roucher, Aymeric and Reedi, Aksel Joonas and Gallouédec, Quentin and Rasul, Kashif and Habib, Nathan and Fourrier, Clémentine and Kydlicek, Hynek and Penedo, Guilherme and Larcher, Hugo and Morlon, Mathieu and Srivastav, Vaibhav and Lochner, Joshua and Nguyen, Xuan-Son and Raffel, Colin and von Werra, Leandro and Wolf, Thomas},
  year={2025},
  howpublished={\url{https://huggingface.co/blog/smollm3}}
}

@inproceedings{kwon2023vllm,
  title={Efficient Memory Management for Large Language Model Serving with PagedAttention},
  author={Woosuk Kwon and Zhuohan Li and Siyuan Zhuang and Ying Sheng and Lianmin Zheng and Cody Hao Yu and Joseph E. Gonzalez and Hao Zhang and Ion Stoica},
  booktitle={Proceedings of the ACM SIGOPS 29th Symposium on Operating Systems Principles},
  year={2023},
  url={https://arxiv.org/abs/2309.06180}
}

@misc{gao2023evalharness,
  author       = {Gao, Leo and Tow, Jonathan and Abbasi, Baber and Biderman, Stella and Black, Sid and DiPofi, Anthony and Foster, Charles and Golding, Laurence and Hsu, Jeffrey and Le Noac'h, Alain and Li, Haonan and McDonell, Kyle and Muennighoff, Niklas and Ociepa, Chris and Phang, Jason and Reynolds, Laria and Schoelkopf, Hailey and Skowron, Aviya and Sutawika, Lintang and Tang, Eric and Thite, Anish and Wang, Ben and Wang, Kevin and Zou, Andy},
  title        = {A framework for few-shot language model evaluation},
  month        = 12,
  year         = 2023,
  publisher    = {Zenodo},
  version      = {v0.4.0},
  doi          = {10.5281/zenodo.10256836},
  url          = {https://zenodo.org/records/10256836}
}

@article{srp-csiszar1975,
    author = {Csisz{\'a}r, Imre},
    title = {{$I$}-Divergence Geometry of Probability Distributions and Minimization Problems},
    journal = {The Annals of Probability},
    volume = {3},
    number = {1},
    pages = {146--158},
    year = {1975},
    doi = {10.1214/aop/1176996454}
  }

@article{srp-stout2012,
    author = {Stout, Quentin F.},
    title = {Strict {$L_\infty$} Isotonic Regression},
    journal = {Journal of Optimization Theory and Applications},
    volume = {152},
    number = {1},
    pages = {121--135},
    year = {2012},
    doi = {10.1007/s10957-011-9865-8}
  }

@book{srp-shaked2007,
    author = {Shaked, Moshe and Shanthikumar, J. George},
    title = {Stochastic Orders},
    series = {Springer Series in Statistics},
    publisher = {Springer},
    address = {New York},
    year = {2007},
    doi = {10.1007/978-0-387-34675-5}
  }

\clearpage
\appendix
\setcounter{topnumber}{3}
\setcounter{bottomnumber}{2}
\setcounter{totalnumber}{5}
\setcounter{dbltopnumber}{2}
\renewcommand{\topfraction}{0.95}
\renewcommand{\bottomfraction}{0.90}
\renewcommand{\textfraction}{0.05}
\renewcommand{\floatpagefraction}{0.85}
\renewcommand{\dbltopfraction}{0.95}
\renewcommand{\dblfloatpagefraction}{0.85}
\raggedbottom
\makeatletter
\setlength{\@fptop}{0pt}
\setlength{\@fpbot}{0pt plus 1fil}
\setlength{\@dblfptop}{0pt}
\setlength{\@dblfpbot}{0pt plus 1fil}
\makeatother

\appendixcontents
\appsection{Additional Background and Method Details}
\label{app:background}
\paragraph{CE as a distribution update.}
At a response position, the student probabilities satisfy
$p_t^S(k)=\exp(z_{t,k}^S)/\sum_v\exp(z_{t,v}^S)$ and CE optimizes
$-\log p_t^S(y_t)$. In free-logit coordinates, a CE step changes a non-target margin by roughly
$-\eta(p_t^S(i)-p_t^S(j))$, so the higher-probability alternative is pushed
down more strongly. Parameter sharing can change this behavior across states.
Appendix~\ref{app:theory_parameter} gives the corresponding parameter-space
identity.

\paragraph{Directional versus distribution preservation.}
Local KL is zero at the reference candidate distribution. SRP-SFT instead
has a region of zero auxiliary loss, defined by
$\Delta^S_{ij,t}\ge\rho\Delta^T_{ij,t}$. Both permit common logit shifts.
SRP additionally permits nonconstant changes satisfying these inequalities.
Further expansion beyond the scaled boundary incurs no penalty, although
expansion relative to the teacher can still be penalized when $\rho>1$.

\paragraph{Gradient and cost.}
For one active hinge,
$\ell_{ij,t}=\rho\Delta^T_{ij,t}-(z^S_{t,i}-z^S_{t,j})$, so
$\partial\ell/\partial z^S_{t,i}=-1$ and
$\partial\ell/\partial z^S_{t,j}=+1$. Multiple active pairs accumulate these
local corrections. Reference logits can be cached before SFT. Subsequent
training uses only cached token IDs/logits and gathered student logits. Thus
SRP-SFT adds at most $K(K-1)/2$ simple pairwise terms per supervised position
without an extra reference forward pass during student training.

\appsection{Proofs and Theoretical Extensions}
\label{app:srp_proofs}

This appendix proves the KL threshold, the feasible-set and utility
characterization, and the finite-step repair result in
Section~\ref{sec:srp_theory}. It also develops the sparse-graph and
objective-budget extensions, the pairwise certificate, CE dynamics,
exact-penalty analysis, and coverage conditions.

\appsubsection{Exact KL Thresholds and Margin Preservation}
\label{app:srp_kl}

\paragraph{Notation.}
All probabilities in this subsection are normalized on the same candidate
set. For $a,b\in(0,1)$, define
\begin{equation}
 \operatorname{kl}(a\|b)
 =a\log\frac{a}{b}+(1-a)\log\frac{1-a}{1-b}.
 \label{eq:srp_binary_kl}
\end{equation}
Theorem~\ref{thm:srp_kl_threshold} is a finite-alphabet constrained-KL
calculation. KL projection is classical \citep{srp-csiszar1975}. The
orientation here is explicitly $\operatorname{KL}(q\|p)$ with $q$ fixed.
The calculation below uses this fixed-reference orientation.

\begin{proof}[Proof of Theorem~\ref{thm:srp_kl_threshold}]
Write $m=q_i+q_j$, $a=q_i/m=\sigma(d)$,
$s=p_i+p_j$, and $r=p_i/s$. If there are other candidates, grouping
$\{i,j\}$ and its complement gives the exact decomposition
\begin{equation}
 \begin{aligned}
 \operatorname{KL}(q\|p)
 &=\operatorname{kl}(m\|s)+m\operatorname{kl}(a\|r)\\
 &\quad +(1-m)\operatorname{KL}(q_{-ij}\|p_{-ij}),
 \end{aligned}
 \label{eq:srp_kl_chain}
\end{equation}
where $q_{-ij},p_{-ij}$ are the conditional distributions on the
complement. The constraint is $r\leq\sigma(b)<a$. Differentiation yields
\begin{equation}
 \frac{\partial}{\partial r}\operatorname{kl}(a\|r)
 =\frac{r-a}{r(1-r)}<0\qquad(0<r<a).
\end{equation}
Consequently the minimum conditional cost occurs at $r=\sigma(b)$.
The remaining nonnegative terms vanish simultaneously at $s=m$ and
$p_{-ij}=q_{-ij}$. This yields the claimed optimizer. For two candidates,
$m=s=1$ and the conditional Bernoulli term is the whole KL, so the same
formula holds without a complement. A strict inequality in the margin
constraint has the same infimum, approached at the stated boundary.
\end{proof}

\paragraph{Local contraction cost.}
Let $A(x)=\log(1+e^x)$. Direct substitution gives
\begin{equation}
 \operatorname{kl}(\sigma(d)\|\sigma(d-\delta))
 =A(d-\delta)-A(d)+\delta\sigma(d).
\end{equation}
Since $A''(d)=\sigma(d)(1-\sigma(d))$, Taylor expansion establishes the
quadratic term in the main text. This is the \emph{minimum} KL cost of the
prescribed contraction. An arbitrary student can pay more because it also
changes pair mass or other candidates. The unweighted SRP pair term at
$\rho=1$ is exactly $\delta$, independent of $m$. The position-level
contribution is $\delta/|\mathcal P|$.

\paragraph{Pair mass and gap.}
Fix $d>0$ and any $b<0$. On three candidates, let
\begin{equation}
 q=(m\sigma(d),m\sigma(-d),1-m)
\end{equation}
and use the optimizing student from the theorem. Its selected pair is
inverted, while its KL is
$m\operatorname{kl}(\sigma(d)\|\sigma(b))\to0$ as $m\downarrow0$.
Both candidates can belong to the selected set: this example can be the
entire locally normalized three-candidate set. Independently, at fixed
pair mass the cost of reaching a tie tends to zero as $d\downarrow0$.
For a \emph{fixed} $q$ and specified pair the threshold in
Eq.~\ref{eq:srp_kl_threshold} is strictly positive.

\paragraph{Zero SRP with diverging KL.}
For a nonuniform, strictly positive teacher distribution $q$, define
$p_i^{(c)}=q_i^c/\sum_jq_j^c$. All pair margins are multiplied by $c$,
so SRP is zero for $c\geq\max\{1,\rho\}$. Nevertheless,
\begin{equation}
 \operatorname{KL}(q\|p^{(c)})
 =(1-c)\sum_iq_i\log q_i+\log\sum_iq_i^c.
\end{equation}
Dividing by $c$ and sending $c\to\infty$ gives
$\log q_{\max}-\sum_iq_i\log q_i>0$, so
$\operatorname{KL}(q\|p^{(c)})$ diverges as $c\to\infty$ while SRP stays zero.

\appsubsection{Order-Cone Geometry and Conditional Utility}
\label{app:srp_geometry}

Fix one prefix and suppress its index. Let $\mathcal A$ have $k$ tokens.
Write $t_i=z_i^T$, $s_i=z_i^S$, and $x_i=s_i-\rho t_i$ for $i\in\mathcal A$.
The teacher defines the strict order $i\succ j$ whenever $t_i>t_j$. Equal
teacher logits are incomparable. Define the corresponding order cone
\begin{equation}
 \mathcal O=\{x:x_i\geq x_j\text{ whenever }i\succ j\}.
\end{equation}
Then $\mathcal C_\rho=\rho t+\mathcal O$ and
$v_{ij}=[x_j-x_i]_+$. This is an order-restricted approximation problem,
related to classical $L_\infty$ isotonic regression \citep{srp-stout2012}.

\begin{proof}[Proof of Eq.~\ref{eq:srp_repair_distance}]
For any $x^*\in\mathcal O$ and $i\succ j$,
\begin{equation}
 x_j-x_i\leq |x_j-x_j^*|+|x_i-x_i^*|
 \leq2\|x-x^*\|_\infty.
\end{equation}
Taking the largest positive violation proves the lower bound $V/2$.
For the matching upper bound, set
\begin{equation}
 \begin{aligned}
 \overline x_i&=\max\{x_j:j=i\text{ or }i\succ j\},\\
 \underline x_i&=\min\{x_j:j=i\text{ or }j\succ i\},\\
 x_i^*&=(\overline x_i+\underline x_i)/2.
 \end{aligned}
 \label{eq:srp_envelopes}
\end{equation}
If $i\succ j$, transitivity makes the index set for $\overline x_i$ contain that
for $\overline x_j$, while that for $\underline x_i$ is contained in that for $\underline x_j$.
Thus $\overline x_i\geq\overline x_j$, $\underline x_i\geq\underline x_j$, and $x^*\in\mathcal O$.
Moreover, $0\leq\overline x_i-x_i\leq V$ and $0\leq x_i-\underline x_i\leq V$, giving
$|x_i^*-x_i|\leq V/2$. Hence $s^*=\rho t+x^*$ is feasible and has the
required distance. If there are no strict pairs, $\overline x_i=\underline x_i=x_i$ and both
sides of Eq.~\ref{eq:srp_repair_distance} are zero. Teacher ties leave the
distance at zero without a within-tie order.
\end{proof}

\paragraph{Zero-violation utility ordering.}
Let $\widetilde p=\operatorname{softmax}(s^*)$ for any
$s^*\in\mathcal C_\rho$ and let $q^{(\rho)}=\operatorname{softmax}(\rho t)$.
Write
$r_i=\widetilde p_i/q_i^{(\rho)}$. The feasibility inequalities make
$r_i$ nondecreasing with teacher preference across distinct teacher levels.
For a utility $U$ as in Theorem~\ref{thm:srp_geometry_utility}, the
likelihood-ratio ordering implication can be written as a covariance
\citep{srp-shaked2007}:
\begin{equation}
 \begin{aligned}
 \mathbb E_{\widetilde p}U-\mathbb E_{q^{(\rho)}}U
 &=\operatorname{Cov}_{q^{(\rho)}}(r,U)\\
 &=\frac12\sum_{i,j}q_i^{(\rho)}q_j^{(\rho)}
       (r_i-r_j)(U_i-U_j)\geq0.
 \end{aligned}
 \label{eq:srp_utility_order}
\end{equation}
The final inequality follows because both $r$ and $U$ are nondecreasing in
teacher preference. Pairs tied under the teacher have equal utility and
contribute zero. Applying the same covariance argument to
$q^{(\rho)}$ and $q$ shows that
$\mathbb E_{q^{(\rho)}}U\geq\mathbb E_qU$ for $\rho\geq1$, since
$q^{(\rho)}_i/q_i$ is proportional to
$\exp((\rho-1)t_i)$. The reference in Eq.~\ref{eq:srp_utility_bound} is the flattened
$q^{(\rho)}$.

\begin{srplemma}[Softmax sensitivity to a bounded logit range]
\label{lem:srp_softmax_tv}
Let $p=\operatorname{softmax}(s)$,
$p'=\operatorname{softmax}(s+e)$, and
$D=\max_i e_i-\min_i e_i$. Then
\begin{equation}
 \operatorname{TV}(p,p')\leq\tanh(D/4).
 \label{eq:srp_softmax_tv}
\end{equation}
In particular, $\|e\|_\infty\leq\varepsilon$ implies
$\operatorname{TV}(p,p')\leq\tanh(\varepsilon/2)$.
\end{srplemma}
\begin{proof}
For a nontrivial event $B$ with $a=p(B)$, the ratio of the odds of $p'(B)$
to those of $p(B)$ is the ratio of the conditional means of $e^{e_i}$ on
$B$ and $B^c$. It lies in $[e^{-D},e^D]$, hence
\begin{equation}
 \sigma(\operatorname{logit}(a)-D)
 \leq p'(B)\leq\sigma(\operatorname{logit}(a)+D).
\end{equation}
Maximizing either possible difference over $a\in(0,1)$ gives
$\tanh(D/4)$: for the upper difference its maximizer is
$a=\sigma(-D/2)$. Taking the supremum over $B$ proves the claim.
The bound is attained by two probabilities with opposite logit changes
of total range $D$.
\end{proof}

\begin{proof}[Proof of Eq.~\ref{eq:srp_utility_bound} and tightness]
Use the feasible $s^*$ constructed in Eq.~\ref{eq:srp_envelopes}.
The correction $s^*-s$ has infinity norm at most $V/2$ and range at most
$V$. Lemma~\ref{lem:srp_softmax_tv} and $0\leq U_i\leq1$ give
\begin{equation}
 \mathbb E_pU\geq\mathbb E_{\widetilde p}U-\tanh(V/4)
 \geq\mathbb E_{q^{(\rho)}}U-\tanh(V/4).
\end{equation}
For tightness, choose two candidates, any $\rho>0$, and a desired
$V>0$. Set $\rho t=(V/4,-V/4)$, $s=(-V/4,V/4)$, and $U=(1,0)$.
The unique selected hinge equals $V$, and
\begin{equation}
 \begin{aligned}
 \mathbb E_{q^{(\rho)}}U-\mathbb E_pU
 &=\sigma(V/2)-\sigma(-V/2)\\
 &=\tanh(V/4).
 \end{aligned}
\end{equation}
Thus the uniform error term cannot be improved as a function of $V$
alone. At $V=0$ the bound follows by continuity.
\end{proof}

\paragraph{A common decoding temperature.}
For $\tau>0$, apply the theorem to $s/\tau$ and $t/\tau$. With
$p_\tau=\operatorname{softmax}(s/\tau)$ and
$q^{(\rho)}_\tau=\operatorname{softmax}(\rho t/\tau)$, the error term
becomes $\tanh(V/(4\tau))$.

\paragraph{A probability-weighted refinement.}
The same pairwise identity gives a second, sometimes sharper bound:
\begin{equation}
 \begin{aligned}
 \mathbb E_{q^{(\rho)}}U-\mathbb E_pU
 &\leq B(p,q^{(\rho)},v),\\
 B(p,q^{(\rho)},v)
 &:=\sum_{i\succ j}q_i^{(\rho)}p_j(1-e^{-v_{ij}}).
 \end{aligned}
 \label{eq:srp_weighted_utility}
\end{equation}
Indeed, for a violating pair,
$q_i^{(\rho)}p_j-p_iq_j^{(\rho)}
=q_i^{(\rho)}p_j(1-e^{-v_{ij}})$. Nonviolating pairs have nonpositive
contributions, and $0\leq U_i-U_j\leq1$. One may take the minimum of
$B$ and $\tanh(V/4)$ in the utility bound.

\paragraph{Imperfect utility ordering.}
Let $\mathcal U_T$ be the set of $[0,1]$-valued utilities satisfying the
teacher order and tie condition, and for an arbitrary utility $U\in[0,1]^k$
define $\zeta_T(U)=\inf_{\widetilde U\in\mathcal U_T}
\|U-\widetilde U\|_\infty$. The set is compact, so the infimum is attained.
For $\rho\geq1$,
\begin{equation}
 \mathbb E_pU\geq\mathbb E_qU
 -\min\{\tanh(V/4),B\}-2\zeta_T(U).
 \label{eq:srp_unreliable_utility}
\end{equation}
This follows by applying the ordered-utility bound to $\widetilde U$ and
bounding each expectation error by $\zeta_T(U)$.
A uniform utility bound without such a condition is impossible: a
zero-violation sharpening can put more probability on a teacher-preferred
but incorrect candidate.

\paragraph{Mean hinge and worst-pair distance.}
For $M=|\mathcal P|>0$,
$\mathcal L_{\mathrm{SRP},t}\leq V\leq M\mathcal L_{\mathrm{SRP},t}$.
Together with Eq.~\ref{eq:srp_repair_distance}, this yields
\begin{equation}
 \frac{\mathcal L_{\mathrm{SRP},t}}{2}
 \leq\operatorname{dist}_\infty(s,\mathcal C_\rho)
 \leq\frac{M\mathcal L_{\mathrm{SRP},t}}{2}.
\end{equation}
Bounding $M\leq K(K-1)/2$ introduces a factor that depends on the candidate count.

\appsubsection{Full and Adjacent Pair Graphs}
\label{app:srp_graph}
For distinct teacher logits, relabel the candidates so that
$t_1>\cdots>t_k$. Write
$a_r=[x_{r+1}-x_r]_+$, $r=1,\ldots,k-1$, and let $R_{\mathrm{adj}}$
and $R_{\mathrm{full}}$ be the respective \emph{mean} hinges.
Then
\begin{equation}
 v_{ij}\leq\sum_{r=i}^{j-1}a_r,\qquad i<j,
 \label{eq:srp_path_residual}
\end{equation}
and therefore
\begin{equation}
 \begin{aligned}
 R_{\mathrm{full}}
 &\leq\frac{1}{\binom{k}{2}}\sum_{r=1}^{k-1}r(k-r)a_r,\\
 R_{\mathrm{adj}}&\leq\frac{k}{2}R_{\mathrm{full}}.
 \end{aligned}
 \label{eq:srp_graph_bounds}
\end{equation}
In particular, the two penalties have the \emph{same zero set}. Away from
zero their means differ, as the examples below show.
\begin{proof}
The identity $x_j-x_i=\sum_{r=i}^{j-1}(x_{r+1}-x_r)$ and
$[\sum_r b_r]_+\leq\sum_r[b_r]_+$ prove
Eq.~\ref{eq:srp_path_residual}. Each adjacent edge $r$ belongs to
$r(k-r)$ full-pair paths, giving the first bound. The full sum includes
all adjacent terms, giving the second after accounting for the two
different averaging denominators. Zero adjacent residuals imply every
full constraint by telescoping. The reverse follows by subset inclusion.
\end{proof}
For example, with $k=3$ and $x=(1,0,1)$, the adjacent and full means are
$1/2$ and $1/3$. With $x=(0,1,2)$ they are $1$ and $4/3$. Thus neither
mean dominates the other pointwise. Moreover,
$V\leq\sum_ra_r=(k-1)R_{\mathrm{adj}}$, which supplies a valid
worst-pair bound for this sparse graph.
For teacher ties, arbitrarily sorting tied neighbors and then skipping them
can change the zero set. The order graph
connects every token across consecutive \emph{distinct teacher levels}
(or otherwise retain paths for every strict comparison).

\appsubsection{Finite-Step Repair with Shared Parameters}
\label{app:srp_repair}

Fix a finite collection of state--pair triples, their weights, and their
reference gaps. The collection and weights stay fixed inside each inequality.
Let
\begin{equation}
 \begin{aligned}
 b_a&=\rho d_a^T-d_a^S(\theta),\\
 I_a&=\mathbf1[b_a>0],\qquad
 r=\sum_aw_aI_a\nabla d_a^S(\theta),\\
 L_R&=\sum_aw_aL_a.
 \end{aligned}
\end{equation}
The convention $I_a=0$ at equality matches the implementation's zero
hinge subgradient. For an arbitrary actual update $h$, set
$u_a=-\langle\nabla d_a^S,h\rangle$ and define
\begin{equation}
 C(h)=\sum_aw_a\bigl([b_a+u_a]_+-[b_a]_+-I_au_a\bigr).
 \label{eq:srp_crossing}
\end{equation}
Convexity of the scalar hinge implies $C(h)\geq0$. Its summands are zero
whenever the line segment from $b_a$ to $b_a+u_a$ stays in the same affine
hinge region. At $b_a=0$, a move into the positive region contributes its
full positive linearized residual.

\begin{proof}[Proof of Proposition~\ref{prop:srp_aggregate_repair}]
The Lipschitz-gradient assumption yields
\begin{equation}
 \begin{aligned}
 d_a^S(\theta+h)&=d_a^S(\theta)
     +\langle\nabla d_a^S,h\rangle+e_a,\\
 |e_a|&\leq\tfrac12L_a\|h\|_2^2.
 \end{aligned}
\end{equation}
The scalar hinge is 1-Lipschitz, hence
\begin{equation}
 [b_a+u_a-e_a]_+\leq[b_a+u_a]_++\tfrac12L_a\|h\|_2^2.
\end{equation}
Multiplying by $w_a$, summing, and using Eq.~\ref{eq:srp_crossing} gives
\begin{equation}
 R(\theta+h)-R(\theta)
 \leq-\langle r,h\rangle+C(h)+\tfrac12L_R\|h\|_2^2.
 \label{eq:srp_arbitrary_step}
\end{equation}
Substitute $h=-\eta(\nabla f-\mu r)$ and set $C_\eta=C(h)$ to obtain
Eq.~\ref{eq:srp_finite_repair}.
\end{proof}

\paragraph{A checkable sufficient condition.}
For the SGD step and $r\ne0$, aggregate repair is certified whenever
\begin{equation}
 \mu\|r\|_2^2>
 \langle r,\nabla f\rangle+C_\eta/\eta
             +\tfrac12L_R\eta\|\nabla f-\mu r\|_2^2.
 \label{eq:srp_repair_condition}
\end{equation}
For a fixed point with no hinge at equality, sufficiently small steps
have $C_\eta=0$. Then
$\mu>\langle r,\nabla f\rangle/\|r\|_2^2$ is the leading-order repair
condition, and strict inequality ensures an actual decrease for sufficiently
small $\eta$.

\paragraph{Why the Gram geometry matters.}
Writing $G_{ab}=\langle\nabla d_a^S,\nabla d_b^S\rangle$,
\begin{equation}
 \|r\|_2^2=\sum_{a,b}w_aw_bI_aI_bG_{ab}.
\end{equation}
Although this Gram matrix is positive semidefinite, individual entries can
be negative. Distinct active pairs can therefore cancel in shared
parameter space. This is the same interference issue emphasized by
parameter-space forgetting analyses \citep{theory-gem,theory-ntk}.

There is a useful contrast in one-prefix \emph{free-logit} coordinates.
Order the $k$ candidates by the teacher (breaking ties only for this
auxiliary argument), set $c_i=(k+1)/2-i$, and let
$W_A=\sum_{a:I_a=1}w_a$. Every active edge $i\succ j$ has
$\langle c,e_i-e_j\rangle=j-i\geq1$. Thus
\begin{equation}
 \|r\|_2\geq
 \frac{W_A}{\sqrt{k(k^2-1)/12}}\qquad(k\geq2).
 \label{eq:srp_logit_non_cancellation}
\end{equation}
This proves that nonempty teacher-oriented active sets cannot cancel
completely in these logit coordinates.

\paragraph{AdamW and changing minibatches.}
Equation~\ref{eq:srp_arbitrary_step} accepts the actual AdamW displacement,
including momentum, preconditioning, clipping, and weight decay.

\appsubsection{From an Objective Budget to Sorting Risk}
\label{app:srp_train_to_risk}

The original residual certificate is useful for evaluating a trained
checkpoint. The following result controls residuals from an objective budget.

\begin{srpproposition}[Reference-comparator bound]
\label{prop:srp_comparator}
Let $\mu>0$ and let $f,R$ be evaluated on the same fixed weighted training
collection. If a checkpoint satisfies
\begin{equation}
 (f+\mu R)(\theta)\leq(f+\mu R)(\theta_0)+\xi,
 \qquad\xi\geq0,
 \label{eq:srp_objective_budget}
\end{equation}
then
\begin{equation}
 R(\theta)\leq R(\theta_0)
                 +\frac{f(\theta_0)-f(\theta)+\xi}{\mu}.
 \label{eq:srp_comparator}
\end{equation}
If $d_a^S(\theta_0)=d_a^T$ on every selected pair, then
\begin{equation}
 R(\theta_0)=(\rho-1)_+\sum_aw_ad_a^T.
 \label{eq:srp_initial_debt}
\end{equation}
\end{srpproposition}
\begin{proof}
Rearrange Eq.~\ref{eq:srp_objective_budget}. At matched initialization,
$v_a(\theta_0)=[(\rho-1)d_a^T]_+=(\rho-1)_+d_a^T$ because $d_a^T>0$.
Summing proves Eq.~\ref{eq:srp_initial_debt}.
\end{proof}
For $0\leq\rho\leq1$ the initialization is feasible. For $\rho>1$ there
is an explicit sharpening debt. When the stored reference
and the initialization's evaluated logits differ numerically, the general
bound uses the measured $R(\theta_0)$.

\paragraph{Margin-distribution calibration.}
Let $Z=\sum_aw_a>0$, let $\mathsf D$ assign mass $w_a/Z$ to each pair,
and write $\bar v=R(\theta)/Z$. Define
$F_T(\gamma)=\Pr_{\mathsf D}(d_a^T<\gamma)$.
For any $\rho\gamma>2\epsilon$, the residual certificate gives
\begin{equation}
 \Pr_{\mathsf D}(d_a^{S,e}\leq0)
 \leq F_T(\gamma)+\frac{\bar v}{\rho\gamma-2\epsilon}.
 \label{eq:srp_budget_risk}
\end{equation}
This counts ties as failures. The full proof is given with the retained
certificate in Appendix~\ref{app:theory_rate}. No pair independence is
needed on a fixed collection.

For a concrete rate, suppose $\epsilon=0$, $\rho>0$, and
$F_T(\gamma)\leq C_T\gamma^\beta$ for every $\gamma>0$, where
$C_T>0$ and $\beta>0$. Minimizing the right-hand side yields
\begin{equation}
 \Pr_{\mathsf D}(d_a^S\leq0)
 \leq c_\beta C_T^{1/(\beta+1)}
       (\bar v/\rho)^{\beta/(\beta+1)},
 \label{eq:srp_margin_tail_rate}
\end{equation}
where $c_\beta=(\beta+1)\beta^{-\beta/(\beta+1)}$. The bound can be
clipped at one. For $\bar v>0$, the optimizer is
$\gamma_*=(\bar v/(\rho C_T\beta))^{1/(\beta+1)}$.
For $\bar v=0$ take the limit $\gamma\downarrow0$.
With matched initialization, $0<\rho\leq1$, and CE $f\geq0$,
Eq.~\ref{eq:srp_comparator} implies
$\bar v\leq(f(\theta_0)+\xi)/(\mu Z)$.
Under fixed constants this gives a conditional
$O(\mu^{-\beta/(\beta+1)})$ sorting-risk bound.

A different prefix distribution is treated in
Appendix~\ref{app:theory_coverage}.

\appsubsection{Proof of Proposition~\ref{prop:theory_contraction}}
\label{app:theory_contraction}
\label{sec:theory_dynamics}

\begin{srpproposition}[Free-logit CE dynamics]
\label{prop:theory_contraction}
Fix one state and label $y$, and regard $z$ itself as the optimization
variable. Under $\dot z=-\nabla_z[-\log p_y]$, every pair $i,j\ne y$ satisfies
\begin{align}
 \dot d_{ij}&=-(p_i-p_j),\\
 \frac{d}{dt}d_{ij}^{\,2}&=-2d_{ij}(p_i-p_j)\leq0.
 \label{eq:theory_contraction}
\end{align}
The inequality is strict whenever $d_{ij}\ne0$, and a nonzero margin cannot
change sign at finite time. A direct discrete logit-gradient step also
preserves its sign and decreases its magnitude when $0<\eta\leq2$.
\end{srpproposition}

\begin{proof}
For $f(z)=-z_y+\log\sum_k e^{z_k}$,
$\partial f/\partial z_k=p_k-\mathbf1[k=y]$. Hence, for $i,j\ne y$,
$\dot d_{ij}=-(p_i-p_j)$. Moreover,
\begin{equation}
 p_i-p_j=(p_i+p_j)\tanh(d_{ij}/2).
 \label{eq:proof_tanh}
\end{equation}
The right-hand side has the same sign as $d_{ij}$, proving the stated
squared-margin derivative and its strict sign when $d_{ij}\ne0$.

For the no-crossing statement define, along the trajectory,
\begin{equation}
 c(t)=
 \begin{cases}
 (p_i+p_j)\tanh(d_{ij}/2)/d_{ij},&d_{ij}\ne0,\\
 (p_i+p_j)/2,&d_{ij}=0.
 \end{cases}
\end{equation}
This is continuous and satisfies $0<c(t)\leq1/2$ at finite logits, since
$|\tanh(x)|\leq|x|$. The margin solves $\dot d_{ij}=-c(t)d_{ij}$, so
\begin{equation}
 d_{ij}(t)=d_{ij}(0)
     \exp\!\left(-\int_0^t c(u)\,du\right).
\end{equation}
The exponential is positive at finite time. An initial tie remains a tie.
An initially strict order keeps its sign.

For a discrete logit-gradient step,
$d_{ij}^{+}=d_{ij}-\eta(p_i-p_j)=(1-\eta c)d_{ij}$, where $c$ is evaluated
before the update. For $d_{ij}\ne0$, the strict inequality
$|\tanh(d_{ij}/2)|<|d_{ij}|/2$ gives $0<c<1/2$. Thus
$0<1-\eta c<1$ when $0<\eta\leq2$. This proves both sign preservation and
strict magnitude reduction. The zero-margin case is immediate.
\end{proof}

The bound $\eta\leq2$ belongs to the unscaled, unit-temperature,
\emph{free-logit} objective.

\appsubsection{Proof of Proposition~\ref{prop:theory_parameter}}
\label{app:theory_parameter}

\begin{srpproposition}[Finite-step parameter-space identity]
\label{prop:theory_parameter}
Let $d_a(\theta)$ have an $L_a$-Lipschitz gradient along the segment from
$\theta$ to $\theta+h$. Then
\begin{equation}
 \begin{aligned}
 d_a(\theta+h)
 &=d_a(\theta)+\langle\nabla d_a(\theta),h\rangle+r_a,\\
 |r_a|&\leq\tfrac12L_a\|h\|_2^2.
 \end{aligned}
 \label{eq:theory_parameter}
\end{equation}
For an SGD step $h=-\eta g$ on a training state $u$,
$g=\nabla_\theta\mathcal L_{\mathrm{CE}}(u)$, the first-order change is
$-\eta I_{u\to a}$, where
\begin{equation}
 I_{u\to a}=\langle\nabla_\theta d_a(s),g\rangle.
 \label{eq:theory_interference}
\end{equation}
The interference term $I_{u\to a}$ can take either sign.
For AdamW, $h$ is the actual preconditioned update.
\end{srpproposition}

\begin{proof}
By the fundamental theorem of calculus,
\begin{equation}
 \begin{aligned}
 &d_a(\theta+h)-d_a(\theta)\\
 &\qquad=\int_0^1\langle\nabla d_a(\theta+t h),h\rangle\,dt.
 \end{aligned}
\end{equation}
Subtracting the first-order term yields
\begin{equation}
 \begin{aligned}
 r_a&=\int_0^1\langle\nabla d_a(\theta+t h),h\rangle\,dt\\
 &\quad-\langle\nabla d_a(\theta),h\rangle.
 \end{aligned}
\end{equation}
The Lipschitz condition and Cauchy--Schwarz give
$|r_a|\leq\int_0^1 L_a t\|h\|_2^2\,dt=L_a\|h\|_2^2/2$.
Substituting $h=-\eta g$ establishes the SGD identity.

To expose the geometry, let $J_s=\partial z(s;\theta)/\partial\theta$ and
$J_u=\partial z(u;\theta)/\partial\theta$. If the training label at $u$ is
$y_u$, then
\begin{equation}
 I_{u\to a}
 =(e_i-e_j)^\top J_sJ_u^\top(p_u-e_{y_u}),
 \label{eq:proof_kernel}
\end{equation}
where $e_k$ is a vocabulary coordinate vector. This expression can take
either sign.
\end{proof}

\paragraph{Same-state parameter-space expansion.}
Take $y=1$, $(i,j)=(2,3)$, and
$p=(0.5,0.3,0.2)$, so $d_{23}=\log(3/2)>0$. Consider the linear
parameterization $z=J\theta$ with $J=\operatorname{diag}(1,1,\sqrt{10})$,
initialized to realize this $p$. Under parameter-space CE gradient flow,
\begin{equation}
 \dot z=-JJ^\top(p-e_1)=(0.5,-0.3,-2),
\end{equation}
and therefore $\dot d_{23}=1.7>0$. Same-state parameter-space CE
can expand a non-target margin.

\paragraph{Finite-step sufficient conditions.}
For $h=-\eta g$, actual contraction is guaranteed if
$I_{u\to a}>\eta L_a\|g\|_2^2/2$. For a currently positive margin,
\begin{equation}
 \eta I_{u\to a}
 >d_a(\theta)+\tfrac12\eta^2L_a\|g\|_2^2
\end{equation}
is sufficient for inversion.

\paragraph{Several auxiliary pairs can interact.}
Away from hinge boundaries, an SGD update using the penalty has the form
\begin{equation}
 h=-\eta\nabla f
     +\eta\mu\sum_{a:g_a>0}w_a\nabla d_a^S.
\end{equation}
Its first-order effect on another margin $d_b^S$ contains the terms
$\langle\nabla d_b^S,\nabla d_a^S\rangle$, which may be negative.

\appsubsection{Proof of Lemma~\ref{thm:theory_stability}}
\label{app:theory_stability}
\label{sec:theory_certificate}

\begin{srplemma}[Residual-aware pairwise stability]
\label{thm:theory_stability}
For any selected pair and $\rho\geq0$,
\begin{equation}
 \frac{p_i^S}{p_j^S}\geq e^{-v_a}
 \left(\frac{p_i^T}{p_j^T}\right)^\rho.
 \label{eq:theory_odds}
\end{equation}
Under any logit perturbation $e$ with $\|e\|_\infty\leq\epsilon$,
\begin{equation}
 d_a^{S,e}\geq\rho d_a^T-v_a-2\epsilon.
 \label{eq:theory_robust}
\end{equation}
Hence $\rho d_a^T-v_a>2\epsilon$ guarantees strict order preservation.
\end{srplemma}

For an already positive student margin, its exact $\ell_\infty$ distance
to the pairwise tie boundary is $d_a^S/2$
\citep{theory-robustness}.

\begin{proof}
By definition $v_a\geq\rho d_a^T-d_a^S$, so
$d_a^S\geq\rho d_a^T-v_a$. At unit temperature,
$\log(p_i^M/p_j^M)=z_i^M-z_j^M=d_a^M$ for either model $M$. Exponentiating
the margin inequality proves Eq.~\ref{eq:theory_odds}.

Under the perturbation,
\begin{align}
 d_a^{S,e}
 &=d_a^S+e_i-e_j\\
 &\geq d_a^S-|e_i|-|e_j|\\
 &\geq\rho d_a^T-v_a-2\epsilon.
\end{align}
The strict positivity condition follows immediately. It is strict because
equality can permit a tie.

For the exact radius, any perturbation of norm less than $d_a^S/2$ leaves
a positive margin. Conversely, setting
$e_i=-d_a^S/2$, $e_j=d_a^S/2$, and all other coordinates to zero reaches a
tie with norm $d_a^S/2$. Therefore this is exactly the distance to the
closed tie-or-reversal set. The infimum radius for a strict reversal is
the same, although it is not attained by the tie perturbation itself.
\end{proof}

\paragraph{Parameter-space extension.}
Suppose a pair-margin function is $G_a$-Lipschitz with respect to a
parameter norm on a neighborhood of the evaluated student. Then
\begin{equation}
 d_a^S(\theta+h)\geq\rho d_a^T-v_a-G_a\|h\|.
\end{equation}
Thus $G_a\|h\|<\rho d_a^T-v_a$ is sufficient for local order preservation,
provided the entire perturbation remains in that neighborhood.

\appsubsection{Proof of Corollary~\ref{cor:theory_rate}}
\label{app:theory_rate}

\begin{srpcorollary}[From average violation to reversal rate]
\label{cor:theory_rate}
Let $\mathsf D$ be any fixed distribution over selected state--pair triples.
For $\rho>0$, $\gamma>0$, and $\rho\gamma>2\epsilon$, define
$\bar v_{\mathsf D}=\mathbb E_{\mathsf D}v_a$. Under any perturbation obeying
the bound in Lemma~\ref{thm:theory_stability},
\begin{equation}
 \Pr_{\mathsf D}(d_a^{S,e}\leq0)
 \leq \Pr_{\mathsf D}(d_a^T<\gamma)
       +\frac{\bar v_{\mathsf D}}{\rho\gamma-2\epsilon}.
 \label{eq:theory_rate}
\end{equation}
\end{srpcorollary}

\begin{proof}
Let $c=\rho\gamma-2\epsilon>0$. On the event
$\{d_a^T\geq\gamma,\ d_a^{S,e}\leq0\}$,
Lemma~\ref{thm:theory_stability} gives
\begin{equation}
 0\geq\rho d_a^T-v_a-2\epsilon
 \quad\Longrightarrow\quad v_a\geq c.
\end{equation}
Consequently the following pointwise domination holds:
\begin{equation}
 \mathbf1[d_a^{S,e}\leq0]
 \leq\mathbf1[d_a^T<\gamma]+v_a/c.
\end{equation}
Taking expectations yields Eq.~\ref{eq:theory_rate}. No independence among
tokens or among pairs is required for this deterministic expectation
inequality. The upper bound may be clipped at one.
\end{proof}

The result counts ties as failures and therefore also bounds the strict
rank-inversion event used by a flip diagnostic. The bound uses $\rho>0$.

\appsubsection{Weighted Penalties: Exactness and Approximate Feasibility}
\label{app:theory_penalty}

The following specialization of exact-penalty reasoning states explicitly
what is required for a soft hinge penalty to recover a constrained optimum
\citep{theory-penalty,theory-convex}. Let $\mathcal C$ be a finite constraint
index set, $w_a>0$, and
\begin{align}
 g_a(\theta)&=\rho d_a^T-d_a^S(\theta),\\
 R(\theta)&=\sum_{a\in\mathcal C}w_a[g_a(\theta)]_+,\\
 J_\mu(\theta)&=f(\theta)+\mu R(\theta).
\end{align}

\begin{srpproposition}[Conditional weighted exactness]
\label{prop:theory_exact_penalty}
Assume there exist a feasible $\theta^\star$ and finite multipliers
$\lambda_a^\star\geq0$ satisfying
$\lambda_a^\star g_a(\theta^\star)=0$ and the global lower-bound condition
\begin{equation}
 f(\theta)+\sum_a\lambda_a^\star g_a(\theta)
 \geq f(\theta^\star)\qquad\text{for all }\theta.
 \label{eq:proof_dual_condition}
\end{equation}
If $\mu w_a>\lambda_a^\star$ for every $a$, the global minimizers of
$J_\mu$ are exactly the global solutions of
$\min_\theta f(\theta)$ subject to all $g_a(\theta)\leq0$.
\end{srpproposition}

\begin{proof}
For every real $g$ and nonnegative $\lambda$,
\begin{equation}
 \mu w[g]_+-\lambda g
 =(\mu w-\lambda)[g]_++\lambda[-g]_+.
\end{equation}
Combining this identity with Eq.~\ref{eq:proof_dual_condition} yields
\begin{equation}
 J_\mu(\theta)
 \geq f(\theta^\star)
       +\sum_a(\mu w_a-\lambda_a^\star)[g_a(\theta)]_+.
\end{equation}
At $\theta^\star$, feasibility gives $J_\mu(\theta^\star)=f(\theta^\star)$.
If any constraint is violated, the strict coefficient inequality makes
$J_\mu(\theta)>f(\theta^\star)$, so every global penalty minimizer is
feasible. On feasible points $J_\mu=f$. Thus every penalty minimizer is a
constrained minimizer. Conversely, every constrained minimizer has value
$f(\theta^\star)$ and zero penalty, hence minimizes $J_\mu$.
\end{proof}

\appsubsection{State Coverage, Normalization, and Independent-Holdout Certificates}
\label{app:theory_coverage}

\paragraph{Matching the averaging measure.}
For one sequence with $T$ supervised positions, each nonempty pair set has
weight $w_{(t,i,j)}=1/(T|\mathcal P_t|)$ in the auxiliary part of
Eq.~\ref{eq:total_loss}. The total weight is
$Z=T^{-1}\sum_t\mathbf1[|\mathcal P_t|>0]$. If $Z>0$, normalizing these
weights gives a distribution $\mathsf D$ over selected triples and
\begin{equation}
 \mathbb E_{\mathsf D}v_a
 =\frac{1}{Z T}\sum_t\mathcal L_{\mathrm{SRP},t}.
\end{equation}
When every position has a nonempty pair set, $Z=1$. Mixtures of sequences
must likewise reproduce the actual position/sequence weights used by the
objective. Empty pair sets contribute no auxiliary term.

\paragraph{Distribution shift.}
Let $\mathsf D_{\mathrm{src}}$ and $\mathsf D_{\mathrm{old}}$ be distributions
on a common space of state--pair triples, each oriented using the same
reference rule. Define
$\operatorname{TV}(P,Q)=\sup_B|P(B)-Q(B)|$.

\begin{srpcorollary}[Coverage-dependent transfer]
\label{cor:theory_transfer}
Under the conditions of Corollary~\ref{cor:theory_rate},
\begin{align}
 \Pr_{\mathsf D_{\mathrm{old}}}(d_a^{S,e}\leq0)
 &\leq\Pr_{\mathsf D_{\mathrm{src}}}(d_a^T<\gamma)
       +\frac{\bar v_{\mathsf D_{\mathrm{src}}}}{\rho\gamma-2\epsilon}
       \nonumber\\
 &\quad+\operatorname{TV}
       (\mathsf D_{\mathrm{old}},\mathsf D_{\mathrm{src}}).
 \label{eq:proof_transfer}
\end{align}
Alternatively, if $\mathsf D_{\mathrm{old}}\ll\mathsf D_{\mathrm{src}}$ and
$d\mathsf D_{\mathrm{old}}/d\mathsf D_{\mathrm{src}}\leq C$ almost surely,
then
\begin{equation}
 \Pr_{\mathsf D_{\mathrm{old}}}(d_a^{S,e}\leq0)
 \leq\Pr_{\mathsf D_{\mathrm{old}}}(d_a^T<\gamma)
       +\frac{C\bar v_{\mathsf D_{\mathrm{src}}}}{\rho\gamma-2\epsilon}.
\end{equation}
\end{srpcorollary}

\begin{proof}
For the fixed failure event, its probabilities under two distributions
differ by at most their total variation. Apply
Corollary~\ref{cor:theory_rate} to the source distribution. For the second
claim, change of measure and nonnegativity imply
$\mathbb E_{\mathrm{old}}v_a\leq C\mathbb E_{\mathrm{src}}v_a$.
Apply Corollary~\ref{cor:theory_rate} to the old distribution.
\end{proof}

\begin{srpcorollary}[A bounded independent-holdout certificate]
\label{cor:theory_holdout}
Fix the trained student, reference, $\rho,\gamma,\epsilon$ independently of
an i.i.d. Holdout $a_1,\ldots,a_n\sim\mathsf D$, with
$c=\rho\gamma-2\epsilon>0$. Define
\begin{equation}
 h(a)=
 \begin{cases}
 1,&d_a^T<\gamma,\\
 \min\{1,v_a/c\},&d_a^T\geq\gamma.
 \end{cases}
\end{equation}
For any $0<\delta<1$, with probability at least $1-\delta$ over the holdout,
\begin{equation}
 \Pr_{\mathsf D}(d_a^{S,e}\leq0)
 \leq\frac1n\sum_{r=1}^n h(a_r)
      +\sqrt{\frac{\log(1/\delta)}{2n}}.
 \label{eq:proof_holdout}
\end{equation}
The right-hand side can be clipped at one. It bounds every admissible
logit perturbation.
\end{srpcorollary}

\begin{proof}
The pointwise argument for Corollary~\ref{cor:theory_rate} shows
$\mathbf1[d_a^{S,e}\leq0]\leq h(a)$: below the margin cutoff $h=1$.
Otherwise any possible failure forces $v_a\geq c$, so again $h=1$.
Since $h\in[0,1]$, Hoeffding's concentration inequality for bounded variables
\citep{theory-hoeffding} gives
\begin{equation}
 \Pr\!\left(\mathbb Eh-\frac1n\sum_r h(a_r)>u\right)
 \leq e^{-2nu^2}.
\end{equation}
Set $u=\sqrt{\log(1/\delta)/(2n)}$ and use
$\Pr(d_a^{S,e}\leq0)\leq\mathbb Eh$.
\end{proof}

Pairs from the same example are averaged within the example, and $n$ is the
number of independent examples. The example-level averages lie in $[0,1]$,
so the same proof applies to the induced example-weighted measure. For a
prespecified family of $M$ model/threshold choices, a union bound replaces
$\log(1/\delta)$ by $\log(M/\delta)$.

\appsubsection{Finite-Precision References and Ties}
\label{app:theory_precision}

The formal guarantees can be interpreted directly with the stored reference
logits as the definition of $T$. Relating them to a higher-precision
reference requires an additional numerical-error bound. Suppose
$|\widetilde z_k^T-z_k^T|\leq\epsilon_T$ on the relevant coordinates and
the penalty uses $\widetilde d_a^T$ with measured residual
$\widetilde v_a=[\rho\widetilde d_a^T-d_a^S]_+$. Then
\begin{equation}
 d_a^{S,e}
 \geq\rho d_a^T-\widetilde v_a-2\rho\epsilon_T-2\epsilon.
 \label{eq:proof_precision}
\end{equation}
Indeed, $\widetilde d_a^T\geq d_a^T-2\epsilon_T$ and the previous residual
and perturbation inequalities apply. A stored positive margin certifies the
same orientation for the higher-precision reference when
$\widetilde d_a^T>2\epsilon_T$. The implementation caches finite-precision
logits and excludes exact ties. Candidate membership can change near
top-$K$ ties.

\paragraph{Margin demand.}
The hinge threshold $\rho$ is a margin-demand level: values below one allow
contraction, $\rho=1$ targets the reference gap, and values above one demand
selective expansion.

\appsection{SRP-SFT Training Algorithm}
\label{app:algorithm}

Algorithm~\ref{alg:srp-sft} restates the procedure of
Section~\ref{sec:method}. The frozen reference is queried once. Subsequent
student updates use only cached alternative-token identities and logits,
together with a standard teacher-forced student forward. Empty pair sets
contribute no auxiliary loss. The headline SRP-SFT run uses $K=10$, $\rho=1$,
and $\mu=1$. The sensitivity grid also evaluates $\rho>1$ as a stronger
margin-demand setting.

\begin{algorithm}[!htbp]
\caption{Selective Rank-Preserving SFT (SRP-SFT)}
\label{alg:srp-sft}
\small
\begin{algorithmic}[1]
\Require SFT dataset $\mathcal{D}=\{(x,y)\}$. Pretrained model $T$
\Require $K$. $\rho\geq 0$. $\mu\geq 0$
\Ensure fine-tuned student $S$
\State Initialize $S\leftarrow T$ and freeze $T$
\For{each $(x,y)\in\mathcal{D}$}
  \Statex \hspace{\algorithmicindent}\textit{// offline reference cache}
  \For{each supervised response position $t$}
    \State $z_t^{T}\leftarrow T(x,y_{<t})$
    \State $\mathcal{A}_t\leftarrow
      \operatorname{TopK}_{k\in\mathcal{V}\setminus\{y_t\}} z_{t,k}^{T}$
    \State Cache $\{(k,z_{t,k}^{T}):k\in\mathcal{A}_t\}$
  \EndFor
\EndFor
\For{each minibatch $\mathcal{B}\subset\mathcal{D}$}
  \State $\mathcal{L}\leftarrow 0$. $n\leftarrow 0$
  \For{each $(x,y)\in\mathcal{B}$}
    \State Compute $(z_t^{S})_t$ with one forward of $S$
    \For{each supervised response position $t$}
      \State $\mathcal{L}_{\mathrm{CE},t}\leftarrow -\log p_t^{S}(y_t)$
      \State $\mathcal{P}_t\leftarrow
        \{(i,j):i,j\in\mathcal{A}_t,\ i\neq j,$
      \Statex $\Delta_{ij,t}^{T}>0\}$
      \If{$\mathcal{P}_t=\emptyset$}
        \State $\mathcal{L}_{\mathrm{SRP},t}\leftarrow 0$
      \Else
        \State $\mathcal{L}_{\mathrm{SRP},t}\leftarrow
          \dfrac{1}{|\mathcal{P}_t|}
          \sum_{(i,j)\in\mathcal{P}_t}$
        \Statex $\bigl[\rho\Delta_{ij,t}^{T}-\Delta_{ij,t}^{S}\bigr]_{+}$
      \EndIf
      \State $\mathcal{L}\leftarrow
        \mathcal{L}+\mathcal{L}_{\mathrm{CE},t}$
      \Statex $+\mu\mathcal{L}_{\mathrm{SRP},t}$. $n\leftarrow n+1$
    \EndFor
  \EndFor
  \State Update $S$ with $\nabla_{\theta}(\mathcal{L}/n)$
\EndFor
\State \Return $S$
\end{algorithmic}
\end{algorithm}

The two loops separate offline reference construction from student
optimization. The cache fixes the candidate identities and reference margins.
The student margins are recomputed at each update. An empty pair set leaves
the supervised CE term in place and sets the auxiliary term to zero.
Values $\rho>1$ in the sensitivity screen require the student margin to
exceed the reference margin.

\paragraph{$K=1$ sanity check.}
After target removal, a one-token candidate set cannot form an ordered pair:
$\mathcal{P}_t=\varnothing$, so $\mathcal{L}_{\mathrm{SRP},t}=0$ and the
objective is exactly CE. We ran a deterministic same-logits audit with the
target both present and absent from the cache. In both cases the pair count
and auxiliary loss were zero, and the maximum objective and gradient
difference from CE was zero to machine precision. The trained $K=1$ row in
Table~\ref{tab:sensitivity} is a separate 478-step run. That table
compares it with the headline CE checkpoint.

A separate controlled training audit now matches CE and $K=1$ for eight
optimizer steps on eight ranks with identical initialization, data order,
RNG state, optimizer, and BF16 precision. We use deterministic SDPA math
attention and disable communication overlap. Every microbatch loss and
optimizer-step gradient norm agrees exactly (maximum difference zero), and
initial and post-update parameter hashes agree on every rank, including the
final step. This identity holds in the stated deterministic configuration.

\appsection{Detailed Experimental Configuration}
\label{app:configuration}

\appsubsection{Coefficient Protocols and Aggregates}
\label{app:main_protocols}
All fine-tuned methods share the training implementation, software and
hardware environment, initialization, data, per-seed data order,
optimizer settings, and numerical precision. The common budget is 478
updates with a 478-step cosine schedule and 15 warmup steps, micro-batch
size one per GPU, 16 accumulation steps, and effective batch size 128.
Method-specific losses are specified below. SRP-SFT uses
$K=10,\rho=1,\mu=1$ at all three main-study seeds.

The main comparison fixes the coefficients in advance and uses them at
every seed. LP-SFT uses $\mu=1$. SRP-SFT uses $K=10$, $\rho=1$, and
$\mu=1$. At $\rho=1$ the requested margin equals the teacher margin.
The value $\mu=1$ is the numerical coefficient on each auxiliary.
The SRP sensitivity study varies one coefficient at a time. The headline
remains $K=10,\rho=1,\mu=1$.

The reasoning mean averages twelve metrics. The code mean averages
HumanEval+ and MBPP+ pass@1/pass@4. Each aggregate uses one checkpoint
and requires every constituent metric. Ranking-loss comparisons and sensitivity studies retain their
separately stated aggregate definitions and configurations. Beyond the
main study, six backbones use separate mathematics checkpoints trained
on OpenMathReasoning \citep{moshkov2025openmathreasoning} and code
checkpoints trained on rStar-Coder \citep{liu2025rstarcoder}
(Appendix~\ref{app:cross_backbone_tables}). These domain-specific studies
are evaluated separately from the UltraFeedback protocol.

\appsubsection{Training Data and Preprocessing}
We initialize all fine-tuned models from Qwen3-4B-Base using HuggingFace
revision
\texttt{906bfd4b4dc7f14ee432\allowbreak0094d8b41684abff8539}.

Training uses the \texttt{train\_sft} split of
\nolinkurl{HuggingFaceH4/ultrafeedback_binarized}.

The original split contains 61,135 rows. We apply right truncation to a maximum
sequence length of 2,048 tokens, after which 61,095 rows retain a valid
assistant target. ChatML formatting is used and the training loss supervises
the final assistant response. Distributed padding and complete gradient
accumulation windows yield 61,184 processed examples per epoch.

\appsubsection{Optimization}
The main experiment trains for one epoch, corresponding to 478 optimizer
steps. We use an effective batch size of 128 and AdamW with learning rate
$2\times10^{-5}$,
$\beta_1=0.9$,
$\beta_2=0.95$,
$\epsilon=10^{-8}$,
and no weight decay.

The learning rate follows a 478-step cosine schedule with 15 warmup
steps for all main-table fine-tuned methods. Training uses
bfloat16 precision, gradient checkpointing, and gradient-norm clipping at 1.0.

The primary screen uses training seed 1234. Independent seed-42 and
seed-3407 checks are reported separately from the headline values.

Experiments run with DeepSpeed ZeRO-2 on eight A800-80GB GPUs. The per-device
micro-batch size is 1 with 16 accumulation steps, giving effective
batch size 128. All compared methods use the same training implementation,
software and hardware environment, data-loader configuration, and common
training hyperparameters. This shared setup applies to CE, LP-SFT,
SRP-SFT, GEM, DFT, ASFT, and EAFT.
Method-specific loss configurations are detailed below.

\appsubsection{Loss Configurations}
We compare plain CE with GEM, DFT, ASFT, EAFT, LP-SFT, and SRP-SFT.
All fine-tuned methods share the same initial checkpoint and supervised
examples. The frozen Base is the untouched evaluation reference.
The ranking-loss screen replaces the auxiliary objective while retaining
the same initialization, data order, optimizer, and one-epoch budget.

\paragraph{Cross-Entropy.}
The CE baseline optimizes the standard causal language-modeling loss on the
supervised assistant response without an additional preservation objective.

\paragraph{GEM.}
GEM uses the official diversity-preserving token loss with $\beta=0.7$ and
linear importance weights \citep{li2025preserving}.

\paragraph{DFT.}
DFT reweights each supervised token by the student's detached target
probability and otherwise uses the same CE reduction
\citep{wu2025generalization}. It introduces no extra hyperparameters.

\paragraph{ASFT.}
ASFT combines the DFT token weight with a frozen-base forward KL term
$\mathrm{KL}(\pi_{\mathrm{base}}\,\|\,\pi_{\mathrm{student}})$ at
$\lambda=0.05$ \citep{zhu2025anchored}.

\paragraph{EAFT.}
EAFT gates the token CE loss by the student's Top-20 predictive entropy,
normalized by $\log 20$, with exponent $\alpha=1$
\citep{diao2026entropy}. High-entropy tokens retain stronger supervision.
Low-entropy confident-conflict tokens are down-weighted.

\paragraph{LP-SFT.}
For LP-SFT, the frozen base model supplies the top-10 non-label alternative
tokens at each valid response position. A locally normalized forward-KL
regularizer is applied over these alternatives. The LP-SFT baseline in the
main tables uses $\mu=1$.

\paragraph{SRP-SFT.}
SRP-SFT constructs all ordered teacher-preferred token pairs among the top-10
non-target alternatives. The headline SRP-SFT configuration uses the
margin-preservation loss with $K=10$, target multiplier $\rho=1$, and
auxiliary loss weight $\mu=1$. This is the teacher-margin setting used
in every main-table seed. The sensitivity screen uses its own $\rho=0.5$
reference. The headline coefficients remain $K=10$, $\rho=1$, and $\mu=1$.

The regularizer is active only when the student margin falls below the
reference-scaled target margin. Further expansion beyond that target
incurs no penalty. For $\rho>1$, a margin larger than the teacher's can
still violate the scaled lower bound.

For preservation-based methods, reference logits are generated once using
bfloat16 inference and cached as float16 values for subsequent training.

\appsubsection{Evaluation Protocol}

\paragraph{MATH-500.}
MATH-500, a 500-problem subset of MATH \citep{hendrycks2021measuring},
is evaluated with temperature 0.6, top-$p$ 1.0, 16 generated samples
per problem, and a maximum of 4,096 generated tokens. Final answers are
extracted from boxed expressions and checked using
\texttt{math-verify} version 0.9.0. We report pass@1 and pass@4.

\paragraph{AIME and GSM-Plus.}
AIME 2021--2026 is evaluated on 180 pinned problems with temperature 0.6,
top-$p$ 1.0, 32 samples per problem, and the same boxed-answer verifier.
The problem sources are the Pandores AIME archive for 2021--2025 and
MathArena for 2026.\footnote{\url{https://huggingface.co/datasets/Pandores/aime-1983-2025}.
\url{https://huggingface.co/datasets/MathArena/aime_2026}.}
The
GSM-Plus mini follow-up \citep{li2024gsmplus} uses the frozen lm-evaluation-harness task with
five-shot prompting. Both strict-match and flexible-extract exact-match
scores are reported.

\paragraph{Broad reasoning screen.}
The step-478 checkpoint is additionally evaluated on LogiQA2
\citep{liu2023logiqa2}, BBH logical deduction \citep{suzgun2023bbh},
ARC-Challenge \citep{clark2018arc}, and HellaSwag \citep{zellers2019hellaswag}
under the same seed-1234 screen.

Table~\ref{tab:main_benchmark} reports LogiQA2 accuracy, BBH exact match for
the five-object split only (BBH-5), and normalized accuracy
(\texttt{acc\_norm}) for ARC-Challenge and HellaSwag. BBH-5 contains 250 questions. The three/five/seven-object diagnostic
suite contains 750 requests in total. The LP-SFT entry follows
the baseline configuration specified in the loss-configuration appendix above.
For BBH, we extract the first explicit \texttt{the answer is (A--G)} marker
from each raw generation (case-insensitive, with an optional colon). Outputs
without a marker are scored incorrect.

\paragraph{MMLU.}
MMLU \citep{hendrycks2020measuring} is evaluated with lm-evaluation-harness 5-shot likelihood scoring
($14{,}042$ questions, seed 1234), using the same step-478 checkpoints as
the rest of Table~\ref{tab:main_benchmark}. We report the official aggregate
accuracy.

\paragraph{Phy-Oly.}
Phy-Oly is the pinned OlympiadBench \citep{he2024olympiadbench} split \texttt{OE\_TO\_physics\_en\_COMP}
($n=236$): English, text-only, competition-source physics problems, excluding
multimodal items, CEE/Gaokao, and theorem-proof questions. All methods use
temperature $0$, greedy decoding, one sample, and at most $4{,}096$ new tokens.
Answers are scored with the official OlympiadBench \texttt{AutoScoringJudge}.
The same extractor and judge are shared across methods.

\paragraph{HumanEval+ and MBPP+.}
HumanEval \citep{chen2021evaluating} and MBPP \citep{austin2021mbpp}
are evaluated with the augmented EvalPlus tests \citep{liu2023your},
yielding HumanEval+ and MBPP+ under EvalPlus version 0.3.1. Generation
uses temperature 0.8, top-$p=1$, and 20 samples per task on 164 HumanEval+
and 378 MBPP+ tasks. All eight seed-1234 checkpoints use
\texttt{max\_tokens=4096} and \texttt{max\_model\_len=8192} throughout.
The dataset, sampling budget, extractor, and evaluator are shared. A sample
is correct only if both \texttt{base\_status} and \texttt{plus\_status} are
\texttt{pass}. For each task with $n=20$ samples and $c$ correct samples,
we compute the unbiased estimate
\begin{equation}
\widehat{\mathrm{pass@}k}=1-\frac{\binom{n-c}{k}}{\binom{n}{k}},
\end{equation}
with $\binom{n-c}{k}=0$ when $n-c<k$, then average over tasks.
Table~\ref{tab:code_benchmark} reports pass@1 and pass@4.
Table~\ref{tab:code_long_cap_full} also reports pass@16. All estimates use
all 20 samples through this estimator.
Table~\ref{tab:code_benchmark} contains all eight checkpoints evaluated
under this shared protocol. The separate seed-42 and seed-3407 panels also
use the 4096/8192 caps. Loss-form and sensitivity studies use the same
evaluation protocol.

\paragraph{Generation backend.}
The MATH and code evaluations use vLLM 0.11.0.
The supplementary broad-reasoning runs use vLLM 0.12.0.
Manifests pin the per-task settings, caps, and generation seed 1234.

\paragraph{Held-out diagnostics.}
Held-out token cross-entropy and local KL are computed relative to the frozen
base model. At each valid response position we restrict to the reference
top-$K$ non-target alternatives and form teacher-preferred pairs
$(i,j)$ with $\Delta^T_{ij}>0$ (typically 8--45 pairs, median 35). Rank
flips, margin contraction, and margin expansion are pair-averaged
indicators and hinges on this set: contraction is
$\frac{1}{|\mathcal{P}_t|}\sum_{(i,j)\in\mathcal{P}_t}[\Delta^T_{ij}-\Delta^S_{ij}]_+$.
Flip and violation are fractions. Contr. is the mean positive logit-margin
contraction in logit units. Whenever we report a percentage
reduction, it is a relative reduction in this magnitude. The tabulated diagnostics use 998 examples and 280{,}908
valid response positions. The mechanism figure uses 265{,}751 positions.
On a log axis the $10^{-2}$ decade therefore looks sparse: CE and LP-SFT
place almost no tokens there, while SRP-SFT's preserved tail is a
minority of positions whose discrete pair averages form faint horizontal
bands.

\paragraph{Cache provenance.}
\label{app:cache_provenance}
Both caches are built from
\texttt{HuggingFaceH4/ultrafeedback\_binarized} at revision
\texttt{3949bf5f8c17c394422ccfab0c31ea9c20bdeb85}, with the Qwen3-4B-Base
chat template, a 2{,}048-token limit, and teacher top-10.
Training uses \texttt{train\_sft}: 61{,}135 rows, of which 61{,}095 retain a
supervised target position. Among the retained training rows, 11 prompt
identifiers occur twice and 3 message sequences are exact duplicates.
The held-out cache uses \texttt{test\_sft}. That split has 1{,}000 rows.
Rows 158 and 374 fill the 2{,}048-token context and leave no supervised
position, so the cache has 998 rows. Those rows have distinct
indices, prompt identifiers, and message sequences. No retained
prompt identifier or message sequence occurs in the retained training cache.
The 998-example statistics in Section~\ref{sec:mechanism} use this cache.

\appsubsection{Software and Hardware}
All compared methods train in the same software environment: Python 3.12,
PyTorch 2.6.0+cu124, Transformers 4.56.2, Datasets 3.2.0, and
DeepSpeed 0.16.9.

Evaluation uses PyTorch 2.9.0+cu128, vLLM \citep{kwon2023vllm} for
generation, and lm-evaluation-harness \citep{gao2023evalharness} for the
likelihood and broad-reasoning tasks described above.

All compared methods use the same hardware configuration of eight
A800-80GB GPUs.

\appsubsection{Reporting Scope}
Tables~\ref{tab:main_benchmark} and~\ref{tab:code_benchmark} report the
seed-1234 checkpoints.
GEM, DFT, ASFT, and EAFT follow the same one-epoch protocol.
Additional CE, LP-SFT, and SRP-SFT
seed checks use the same protocol.
The seed-1234, seed-42, and seed-3407 reasoning and code means are summarized in
Table~\ref{tab:two_seed_summary}.

The matched EvalPlus evaluation is complete for Base, CE, LP-SFT,
and SRP-SFT, with 20 samples per task. The temperature-control comparison
uses the same pass@$k$ estimator, including CE-TM MBPP+.
The manifests match dataset hashes, generation seed, sample count, and caps.
The fitted logit scale is applied before sampling. For EvalPlus, a sample
counts as correct only when both base and additional tests pass.

\begin{table}[!htbp]
\centering
\scriptsize
\setlength{\tabcolsep}{1.4pt}
\renewcommand{\arraystretch}{1.08}
\begin{tabularx}{\columnwidth}{@{}l*{6}{Y}@{}}
\toprule
Method & \multicolumn{3}{c}{HumanEval+} & \multicolumn{3}{c}{MBPP+} \\
\cmidrule(lr){2-4}\cmidrule(lr){5-7}
& p@1 & p@4 & p@16 & p@1 & p@4 & p@16 \\
\midrule
Base$^{\dagger}$ & 45.49 & 71.56 & 87.81 & 46.12 & \secbest{72.50} & \secbest{84.13} \\
CE & 44.18 & 69.56 & 84.10 & 44.66 & 65.79 & 76.73 \\
LP-SFT & \secbest{57.90} & \secbest{81.67} & \secbest{90.50} & \secbest{51.87} & 72.23 & 81.07 \\
\textbf{SRP-SFT} & \textbf{62.84} & \textbf{83.87} & \textbf{91.76} & \textbf{56.24} & \textbf{75.06} & \textbf{84.64} \\
\bottomrule
\end{tabularx}
\caption{Matched code results for Base, CE, LP-SFT, and SRP-SFT
(seed 1234, percentages), adding pass@16 to Table~\ref{tab:code_benchmark}.
All entries use unbiased estimates from 20 samples per task.}
\label{tab:code_long_cap_full}
\end{table}

\FloatBarrier
\appsubsection{Equal Benchmark Weighting}
\label{app:benchmark_equal_mean}
The headline reasoning mean gives each metric equal weight. As a robustness
check, we first average the two metrics within MATH-500, AIME, and GSM-Plus,
then equally average those three scores with LogiQA2, BBH-5, ARC-C,
HellaSwag, MMLU, and Phy-Oly (nine benchmarks). Values below use the same
unrounded seed-1234 measurements.

\begin{table}[!htbp]
\centering
\footnotesize
\begin{tabularx}{\columnwidth}{@{}lYY@{}}
\toprule
Method & Metric mean & Benchmark mean \\
\midrule
Base$^{\dagger}$ & 42.24 & 43.07 \\
CE & 35.99 & 38.27 \\
GEM & 35.01 & 37.44 \\
DFT & 35.25 & 35.87 \\
ASFT & 37.64 & 39.39 \\
EAFT & 34.91 & 37.38 \\
LP-SFT & 43.50 & 43.75 \\
\textbf{SRP-SFT} & 45.11 & 45.13 \\
\bottomrule
\end{tabularx}
\caption{Equal metric versus equal benchmark weighting for the
seed-1234 reasoning screen (percentages). Both aggregates require all
twelve metrics.}
\label{tab:benchmark_equal_mean}
\end{table}

\appsection{Independent-Seed Robustness Checks}
\label{app:seed42}

The independent-seed comparison uses Qwen3-4B-Base, UltraFeedback,
478-step cosine schedules, 15 warmup steps, and effective batch size 128.
SRP-SFT uses $K=10$, $\rho=1$, $\mu=1$ throughout.
Tables~\ref{tab:seed42_reasoning} and~\ref{tab:seed42_code} report all
benchmark columns and methods.

\begin{table}[!htbp]
\centering
\footnotesize
\begin{tabularx}{\columnwidth}{@{}lYY@{}}
\toprule
Method & Reasoning & Code \\
\midrule
CE & $35.96\pm0.10$ & $55.66\pm1.18$ \\
LP-SFT & $43.10\pm0.46$ & $66.17\pm0.27$ \\
\textbf{SRP-SFT} & $44.96\pm0.55$ & $69.48\pm0.02$ \\
\bottomrule
\end{tabularx}
\caption{Mean $\pm$ sample standard deviation across training seeds
1234, 42, and 3407 (percentages). Code
averages four metrics within each seed before computing the statistics.}
\label{tab:three_seed_variation}
\end{table}

Table~\ref{tab:three_seed_variation} treats the training seed as the sample
unit ($n=3$, sample SD with denominator $n-1$). SRP-SFT has the highest
code mean at every seed, with a three-seed mean of 69.48 and SD of 0.02.
Its reasoning mean is $44.96\pm0.55$. The CE and LP-SFT reasoning
means are $35.96\pm0.10$ and $43.10\pm0.46$.

\paragraph{Seed-42 headline results.}
SRP-SFT has MATH pass@1/pass@4 of 59.50/81.60,
a reasoning mean of 45.41, and a code mean of 69.47. LP-SFT has MATH pass@1/pass@4 of 49.45/77.00
and a code mean of 66.14. The LP entropy calibration uses $\alpha=1.0989$.
The per-seed benchmark panels below provide the rows used for the
three-seed comparison. The aggregate summary is reported for the reasoning
and code means above, whose checkpoint-level values are fully listed.

The additional seed-2026 SRP run uses $\rho=0.5$, with MATH
pass@1/pass@4 of 50.00/77.65 and HE+/MBPP+ pass@1/pass@4 of
59.54/82.62 and 51.16/72.74 (code mean 66.52). It is a standalone
configuration-specific result without paired CE/LP runs, and is excluded
from the headline-$\rho=1$ three-seed summary.

\begin{table}[!htbp]
\centering
\scriptsize
\setlength{\tabcolsep}{1.1pt}
\begin{tabularx}{\columnwidth}{@{}l*{5}{Y}@{}}
\toprule
Method & HE+ p@1 & p@4 & MBPP+ p@1 & p@4 & Code mean \\
\midrule
CE & 42.74 & 66.61 & 43.06 & 64.92 & 54.33 \\
LP-SFT & 58.29 & 82.35 & 51.68 & 72.24 & 66.14 \\
GEM & 20.15 & 43.99 & 22.12 & 47.63 & 33.47 \\
DFT & 61.74 & 68.60 & 54.40 & 60.44 & 61.30 \\
SRP ($\rho=1$) & 62.80 & 84.39 & 56.19 & 74.49 & 69.47 \\
LP-TM (target $\rho=1$) & 62.16 & 83.29 & 53.77 & 72.37 & 67.90 \\
\bottomrule
\end{tabularx}
\caption{Seed-42 code results (percentages).
The seed-42 LP-TM code evaluation uses $\alpha=1.0989$.}
\label{tab:seed42_matched_available}
\end{table}


\begin{table*}[!t]
\centering
\scriptsize
\setlength{\tabcolsep}{1.2pt}
\renewcommand{\arraystretch}{1.10}
\begin{tabularx}{\textwidth}{@{}>{\raggedright\arraybackslash}p{1.20cm}*{13}{Y}@{}}
\toprule
Method
& \multicolumn{2}{c}{MATH-500}
& \multicolumn{2}{c}{AIME}
& \multicolumn{2}{c}{GSM-Plus mini}
& LogiQA2
& BBH-5
& ARC-C
& \shortstack{Hella-\\Swag}
& MMLU
& Phy-Oly
& Mean \\
\cmidrule(lr){2-3}
\cmidrule(lr){4-5}
\cmidrule(lr){6-7}
& p@1
& p@4
& p@1
& p@4
& Strict
& Flexible
& Acc.
& EM
& Norm.
& Norm.
& 5-shot
& Acc.
& \\
\midrule
\multicolumn{14}{c}{Training seed 42} \\
\midrule
CE & 31.72 & 58.62 & 0.59 & 2.16 & 45.42 & 45.08 & 42.75 & 2.80 & 53.41 & 71.87 & 69.07 & 8.90 & 36.03 \\
GEM & 20.65 & 47.71 & 0.43 & 1.61 & 47.71 & 48.00 & 44.34 & 8.00 & 53.33 & 72.94 & 69.29 & 4.66 & 34.89 \\
DFT & 31.51 & 42.05 & 0.56 & 1.17 & 50.67 & 52.25 & 31.81 & 0.80 & 51.02 & 65.91 & 66.73 & 5.93 & 33.37 \\
ASFT & 45.50 & 68.85 & 2.73 & 6.94 & 52.92 & 54.12 & 37.15 & 15.60 & 53.58 & 70.75 & 67.80 & 10.17 & 40.51 \\
EAFT & 28.30 & 56.44 & 0.99 & 3.38 & 45.38 & 36.00 & 41.67 & 11.60 & 54.18 & 72.16 & 69.89 & 9.75 & 35.81 \\
LP-SFT & 49.45 & 77.00 & 2.92 & 8.20 & 57.88 & 58.54 & 44.40 & 12.40 & 54.52 & 73.16 & 71.44 & 8.47 & 43.20 \\
SRP-SFT & 59.50 & 81.60 & 4.60 & 11.49 & 58.38 & 59.08 & 46.69 & 14.40 & 54.61 & 73.12 & 72.18 & 9.32 & 45.41 \\
\midrule
\multicolumn{14}{c}{Training seed 3407} \\
\midrule
CE & 30.31 & 57.45 & 0.61 & 2.25 & 46.88 & 47.50 & 40.52 & 2.80 & 53.16 & 71.76 & 70.18 & 6.78 & 35.85 \\
GEM & 21.02 & 48.10 & 0.57 & 2.14 & 46.46 & 46.67 & 45.17 & 11.60 & 54.18 & 72.75 & 70.75 & 6.78 & 35.52 \\
DFT & 48.36 & 60.96 & 1.37 & 3.53 & 45.54 & 46.96 & 34.41 & 0.80 & 49.74 & 65.15 & 64.77 & 7.63 & 35.77 \\
ASFT & 32.19 & 51.22 & 1.82 & 4.85 & 50.33 & 52.88 & 39.89 & 16.40 & 54.35 & 70.81 & 67.55 & 7.20 & 37.46 \\
EAFT & 31.46 & 58.16 & 0.78 & 2.84 & 46.42 & 34.67 & 41.92 & 12.40 & 52.30 & 71.78 & 70.72 & 7.20 & 35.89 \\
LP-SFT & 45.61 & 75.78 & 2.81 & 7.96 & 57.25 & 58.04 & 41.73 & 16.40 & 52.90 & 73.32 & 71.35 & 8.05 & 42.60 \\
SRP-SFT & 53.76 & 79.55 & 4.17 & 10.77 & 57.50 & 58.63 & 43.70 & 15.20 & 53.84 & 73.28 & 72.49 & 9.32 & 44.35 \\
\bottomrule
\end{tabularx}
\caption{Independent seed-42 and seed-3407 reasoning results (percentages).
Each mean averages all twelve metrics from the same checkpoint.
Both SRP-SFT rows use $K=10$, $\rho=1$, $\mu=1$.}
\label{tab:seed42_reasoning}
\end{table*}

\begin{table}[!t]
\centering
\footnotesize
\setlength{\tabcolsep}{2.0pt}
\renewcommand{\arraystretch}{1.10}
\begin{tabularx}{\columnwidth}{@{}l*{5}{Y}@{}}
\toprule
Method & \multicolumn{2}{c}{HumanEval+} & \multicolumn{2}{c}{MBPP+} & Mean \\
\cmidrule(lr){2-3}\cmidrule(lr){4-5}
& pass@1 & pass@4 & pass@1 & pass@4 & \\
\midrule
\multicolumn{6}{c}{Training seed 42} \\
\midrule
CE & 42.74 & 66.61 & 43.06 & 64.92 & 54.33 \\
GEM & 20.00 & 44.48 & 22.91 & 48.31 & 33.93 \\
DFT & 61.86 & 68.59 & 54.48 & 60.49 & 61.35 \\
ASFT & 63.02 & 76.25 & 57.78 & 68.45 & 66.37 \\
EAFT & 46.55 & 71.48 & 45.32 & 65.93 & 57.32 \\
LP-SFT & 58.29 & 82.35 & 51.68 & 72.24 & 66.14 \\
SRP-SFT & 62.80 & 84.39 & 56.19 & 74.49 & 69.47 \\
\midrule
\multicolumn{6}{c}{Training seed 3407} \\
\midrule
CE & 44.57 & 69.79 & 45.62 & 66.41 & 56.60 \\
GEM & 21.31 & 45.93 & 26.93 & 53.33 & 36.88 \\
DFT & 58.38 & 65.59 & 52.53 & 57.03 & 58.38 \\
ASFT & 64.36 & 77.51 & 59.59 & 70.01 & 67.87 \\
EAFT & 47.59 & 71.70 & 45.67 & 65.64 & 57.65 \\
LP-SFT & 59.45 & 83.17 & 51.23 & 71.94 & 66.45 \\
SRP-SFT & 63.35 & 85.06 & 55.34 & 74.16 & 69.48 \\
\bottomrule
\end{tabularx}
\caption{Independent-seed EvalPlus results (percentages).
Mean averages HumanEval+ and MBPP+ pass@1/pass@4 using all 20 samples per
task. All panels use generation seed 1234,
temperature 0.8, top-$p$ 1, and token/context caps 4096/8192.}
\label{tab:seed42_code}
\end{table}

At seed 42, the code means are 69.47 for SRP-SFT and 66.14
for LP-SFT. At seed 3407, they are 69.48 and 66.45. SRP-SFT has
higher MATH pass@1, twelve-metric reasoning mean, and code mean than CE
and LP-SFT at all three seeds. SRP reasoning means are 45.11, 45.41, and
44.35 at seeds 1234, 42, and 3407. The CE means are 35.99, 36.03,
and 35.85, and the LP-SFT means are 43.50, 43.20, and 42.60.

The CPU-side evaluation covers all seven methods and both code datasets.

\appsection{Cross-Backbone High-Quality Robustness}
\label{app:cross_backbone_tables}

As a data-quality and architecture robustness check, we repeated the
Base/CE/LP-SFT/SRP-SFT comparison on six backbones using separate
filtered corpora: OpenMathReasoning \citep{moshkov2025openmathreasoning}
for mathematics SFT and rStar-Coder \citep{liu2025rstarcoder} for code SFT.
These corpora replace UltraFeedback in this robustness study.
Each backbone has separate mathematics-trained and code-trained checkpoints.
We use Qwen base checkpoints because these releases
provide the chat template required by our SFT pipeline. Granite, Ministral,
Gemma, and SmolLM3 use their instruction-tuned releases, which supply that
template.
The released backbone families are Qwen3.5 \citep{qwen2026qwen35},
Granite 4.1 \citep{ibm2026granite41}, Ministral 3 \citep{mistral2025ministral3},
Gemma 3 \citep{google2025gemma3}, and SmolLM3 \citep{hf2025smollm3}.
GSM8K evaluates the mathematics-trained checkpoints \citep{cobbe2021gsm8k}.

\appsubsection{Reproduction Protocol}
\label{app:cross_backbone_protocol}

\begin{table}[!htbp]
\centering
\scriptsize
\setlength{\tabcolsep}{3pt}
\begin{tabularx}{\columnwidth}{@{}Xrr@{}}
\toprule
Setting & Mathematics & Code \\
\midrule
Retained training examples & 96{,}949 & 95{,}966 \\
Example presentations per epoch & 97{,}024 & 96{,}000 \\
Sequence length (tokens) & 4{,}096 & 4{,}096 \\
Effective batch size & 128 & 128 \\
Epochs & 1 & 1 \\
Optimizer updates / cosine horizon & 758 & 750 \\
Linear warmup updates & 23 & 23 \\
Evaluated checkpoint step & 758 & 750 \\
Training seed & 1234 & 1234 \\
\bottomrule
\end{tabularx}
\caption{Cross-backbone training budgets, shared by CE, LP-SFT, and
SRP-SFT for every backbone. Example presentations include deterministic
sampler padding to complete gradient-accumulation windows.}
\label{tab:cross_backbone_training_protocol}
\end{table}

\paragraph{Training subsets.}
For mathematics, we scan the \texttt{cot} split of
\texttt{nvidia/OpenMathReasoning} in sorted Parquet-file and row order,
retaining the first eligible solution for each unique problem after
lowercasing and whitespace normalization. We require a nonempty expected
answer, a problem of at least 20 characters, a solution of at least 80
characters, and at most 12{,}000 characters for problem plus solution.
We exclude normalized exact matches to MATH-500 test prompts and substring
matches against test prompts of at least 80 characters. The recorded scan
contains 3{,}201{,}061 rows and retains 96{,}949 examples, including the
removal of 25 prompt-overlap rows. This filter removes MATH-500 prompt overlap.

For code, we use \texttt{microsoft/rStar-Coder}, scanning
\texttt{seed\_sft} before \texttt{synthetic\_sft} in deterministic file/row
order. Seed examples must have both \texttt{verified} and
\texttt{is\_passed} set. Questions must contain at least 80 characters,
responses at least 120, and their combined length must not exceed
16{,}000 characters. We retain one example per normalized question and
exclude any question sharing a 16-token sequence with a pinned HumanEval+
or MBPP+ prompt, tokenized by the regular expression
\texttt{[a-z0-9\_]+} after lowercasing. The resulting 95{,}966 examples
comprise 3{,}629 seed and 92{,}337 synthetic examples. 17 overlap rows
are excluded. Synthetic examples are kept under the length and overlap
filters. Each domain's resulting JSONL is shared by all six
backbones and all three fine-tuning methods.

\paragraph{Tokenization and reference cache.}
Each example is a user question followed by an assistant solution. We use
the checkpoint's native chat template, disable optional thinking mode,
truncate the rendered sequence on the right at 4{,}096 tokens, and supervise
only assistant-response tokens, including the end-of-turn token when retained.
Each example is a separate sequence. All retained rows have supervised tokens
for all six tokenizers. Table~\ref{tab:cross_backbone_reproduction_models}
gives the resulting supervised-token totals before sampler padding.
For each backbone/domain pair, its frozen initialization supplies one cache
shared by CE, LP-SFT, and SRP-SFT. LP uses temperature one and $\mu=1$.
SRP uses all eligible teacher-preferred pairs, the hinge penalty,
$K=10$, $\rho=1$, and $\mu=1$. These coefficients are fixed across backbones.

\begin{table*}[!t]
\centering
\scriptsize
\setlength{\tabcolsep}{3pt}
\begin{tabularx}{\textwidth}{@{}>{\raggedright\arraybackslash}Xlrr@{}}
\toprule
Initialization (repository ID) & Revision & Math tokens & Code tokens \\
\midrule
\texttt{Qwen/Qwen3.5-2B-Base} & \texttt{b1485b2fa6df} & 306{,}157{,}988 & 234{,}957{,}303 \\
\texttt{Qwen/Qwen3.5-0.8B-Base} & \texttt{dc7cdfe2ee41} & 306{,}157{,}988 & 234{,}957{,}303 \\
\texttt{ibm-granite/granite-4.1-3b} & \texttt{c0650403e44e} & 294{,}685{,}290 & 224{,}987{,}190 \\
\texttt{mistralai/Ministral-3-3B-Instruct-2512-BF16} & \texttt{b6d637bef239} & 283{,}617{,}839 & 220{,}836{,}597 \\
\texttt{google/gemma-3-4b-it} & \texttt{093f9f388b31} & 305{,}993{,}476 & 240{,}191{,}774 \\
\texttt{HuggingFaceTB/SmolLM3-3B} & \texttt{a07cc9a04f16} & 292{,}833{,}133 & 224{,}365{,}722 \\
\bottomrule
\end{tabularx}
\caption{Pinned cross-backbone initializations (12-character revision prefixes)
and supervised-token counts after tokenization and truncation. Counts cover
one copy of each retained example. Token budgets vary with the tokenizer.}
\label{tab:cross_backbone_reproduction_models}
\end{table*}

\paragraph{Optimization and checkpoint selection.}
All 36 fine-tuning runs (six backbones, two domains, three methods) use
eight A800-80GB GPUs, BF16, and DeepSpeed ZeRO-2. We update the text-model
parameters. Vision modules and multimodal projectors, when present, remain
frozen. Each GPU uses micro-batch one and 16 accumulation steps, giving
an effective batch of 128. AdamW uses peak learning rate $2\times10^{-5}$,
$(\beta_1,\beta_2)=(0.9,0.95)$, $\epsilon=10^{-8}$, zero weight decay,
and gradient clipping at 1.0. A linear warmup over 23 updates is followed
by cosine decay to zero over the domain's full one-epoch horizon.
Training and sampler seeds are 1234, with two data-loader workers.
The sampler repeats examples to complete accumulation windows:
97{,}024 example presentations yield 758 math updates, and 96{,}000
yield 750 code updates. We evaluate the final checkpoint at those steps.
The Base row uses the corresponding untouched initialization.

\paragraph{Evaluation settings.}
All four arms use native chat formatting, generation seed 1234, and
disabled optional thinking mode. Mathematics checkpoints are evaluated
on MATH-500 (500 test questions), GSM8K (1{,}319 test questions), and
AIME 2021--2025 (150 questions). MATH-500 and GSM8K use 16 samples per
question, temperature 0.6, top-$p=1$, a 4{,}096-token generation cap,
and an 8{,}192-token context limit. AIME uses 32 samples at the same
temperature and top-$p$, with caps of 8{,}192 generated and 12{,}288
context tokens. The data revisions are \texttt{6e4ed1a2a79a} for
\texttt{HuggingFaceH4/MATH-500}, \texttt{740312add88f} for
\texttt{openai/gsm8k} (\texttt{main}, test), and \texttt{292f1a1b4f04}
for \texttt{Pandores/aime-1983-2025}, using the recorded
\texttt{repaired-v1} AIME data restricted to 2021--2025.
Code checkpoints use 20 samples on each of 164 HumanEval+ and 378 MBPP+
tasks, temperature 0.8, top-$p=0.95$, and generation/context caps of
4{,}096/8{,}192 tokens. This top-$p$ is specific to the cross-backbone
study. EvalPlus 0.3.1 requires both base and additional tests to pass.
Pass@$k$ uses the unbiased estimator given above. The mathematics and
code scores therefore come from separate final checkpoints.

\paragraph{Software and audit records.}
The cross-backbone environment locks record Python 3.12.3,
PyTorch \texttt{2.8.0a0+5228986c39.nv25.5} with CUDA 12.9,
Transformers 5.17.0, DeepSpeed 0.16.9, Datasets 4.1.1, and
Accelerate 1.10.1. The generation launcher first attempts vLLM with
tensor parallelism eight and falls back to eight independent Transformers
workers when a backbone is unsupported. We retain the per-job commands
and backend logs. Dataset manifests record source-file SHA-256 hashes,
filter settings, and retained counts. Cache manifests record tokenization
statistics. Each run records the
resolved optimizer and scheduler settings, and each evaluation has a
generation manifest and raw outputs. These records separate the
cross-backbone protocol from the Qwen3-4B main-study environment.

\appsubsection{Cross-Backbone Results}

These runs use 4{,}096-token sequences, one epoch, $K=10$, and $\rho=1$. They
are therefore a separate robustness study from the 2{,}048-token main screen.
OpenMathReasoning is the mathematics training
domain for the reasoning tables, and the code tables use rStar-Coder.
The reported scores measure in-domain adaptation after domain-specific SFT.
We report the four pre-specified
Base/CE/LP-SFT/SRP-SFT arms, all of which are complete in the reported
six-backbone comparison.

The AIME tables use two groups of three backbones.
Tables~\ref{tab:cross_backbone_reasoning_three_a}
and~\ref{tab:cross_backbone_reasoning_three_b} report GSM8K and MATH-500
checks for the first and second groups, respectively. The code tables retain
separate base-model and instruction-tuned groups. All of these tables occupy
one column. Model names and evaluation metrics identify each comparison.

Figure~\ref{fig:cross_backbone_robustness} reports the difference between
SRP-SFT and the arithmetic mean of CE and LP-SFT for the same backbone
and dataset. The reasoning tables report
pass@16 from the corresponding GSM8K and MATH-500 evaluations.

\begin{table}[!htbp]
\centering
\footnotesize
\setlength{\tabcolsep}{2.0pt}
\renewcommand{\arraystretch}{1.10}
\begin{tabularx}{\columnwidth}{@{}l*{3}{Y}@{}}
\toprule
Method & \shortstack{Qwen3.5-\\2B\\Base} & \shortstack{Qwen3.5-\\0.8B\\Base} & \shortstack{Granite-\\4.1-3B\\Instruct} \\
\midrule
Base & 10.94 & 1.65 & 6.54 \\
CE & 5.42 & 1.77 & 17.38 \\
LP-SFT & 6.83 & 1.58 & 19.46 \\
SRP-SFT & 11.31 & 1.77 & 19.31 \\
\bottomrule
\end{tabularx}
\caption{AIME 2021--2025 pass@1 results for Qwen3.5-2B Base,
Qwen3.5-0.8B Base, and Granite-4.1-3B Instruct.}
\label{tab:cross_backbone_aime_base}
\end{table}

\begin{table}[!htbp]
\centering
\footnotesize
\setlength{\tabcolsep}{2.0pt}
\renewcommand{\arraystretch}{1.10}
\begin{tabularx}{\columnwidth}{@{}l*{3}{Y}@{}}
\toprule
Method & \shortstack{Ministral-\\3-3B\\Instruct} & \shortstack{Gemma-\\3-4B\\it} & \shortstack{SmolLM3-\\3B\\Instruct} \\
\midrule
Base & 18.42 & 9.62 & 8.67 \\
CE & 17.48 & 5.58 & 30.31 \\
LP-SFT & 20.88 & 9.52 & 30.06 \\
SRP-SFT & 21.46 & 8.33 & 29.23 \\
\bottomrule
\end{tabularx}
\caption{AIME 2021--2025 pass@1 results for Ministral-3-3B Instruct,
Gemma-3-4B it, and SmolLM3-3B Instruct.}
\label{tab:cross_backbone_aime_instruct}
\end{table}

\paragraph{Instruction-tuned comparison.}
Across Tables~\ref{tab:cross_backbone_aime_base}
and~\ref{tab:cross_backbone_aime_instruct}, SRP-SFT leads on Ministral, CE leads on SmolLM3,
LP-SFT leads on Granite, and the untouched Base leads on Gemma. These
comparisons use the complete set of pre-specified arms for every backbone.

\paragraph{Independent reasoning checks.}
Table~\ref{tab:cross_backbone_reasoning_three_a} groups Qwen3.5-2B Base,
Qwen3.5-0.8B Base, and Granite-4.1-3B Instruct.
Table~\ref{tab:cross_backbone_reasoning_three_b} contains Ministral-3-3B
Instruct, Gemma-3-4B it, and SmolLM3-3B Instruct. Both tables use the same
backbone grouping and report GSM8K and MATH-500 pass@1/pass@16.
AIME is reported separately in Tables~\ref{tab:cross_backbone_aime_base}
and~\ref{tab:cross_backbone_aime_instruct}.

\begin{table}[!htbp]
\centering
\footnotesize
\setlength{\tabcolsep}{2.0pt}
\renewcommand{\arraystretch}{1.08}
\begin{tabularx}{\columnwidth}{@{}l*{4}{Y}@{}}
\toprule
Method & \multicolumn{2}{c}{GSM8K} & \multicolumn{2}{c}{MATH-500} \\
\cmidrule(lr){2-3}\cmidrule(lr){4-5}
& pass@1 & pass@16 & pass@1 & pass@16 \\
\midrule
\multicolumn{5}{@{}l}{\textit{Qwen3.5-2B Base}} \\
\addlinespace[1pt]
Base & 80.70 & 96.51 & 69.25 & 89.40 \\
CE & 73.46 & 94.69 & 56.95 & 88.60 \\
LP-SFT & 76.46 & 95.75 & 65.59 & 91.80 \\
SRP-SFT & 78.13 & 96.21 & 70.39 & 92.40 \\
\midrule
\multicolumn{5}{@{}l}{\textit{Qwen3.5-0.8B Base}} \\
\addlinespace[1pt]
Base & 56.11 & 91.58 & 45.16 & 77.00 \\
CE & 52.48 & 90.07 & 43.70 & 80.40 \\
LP-SFT & 57.32 & 90.45 & 45.35 & 79.80 \\
SRP-SFT & 55.49 & 90.83 & 45.60 & 78.60 \\
\midrule
\multicolumn{5}{@{}l}{\textit{Granite-4.1-3B Instruct}} \\
\addlinespace[1pt]
Base & 88.53 & 94.31 & 67.04 & 91.00 \\
CE & 89.40 & 97.42 & 75.67 & 93.80 \\
LP-SFT & 91.12 & 96.44 & 81.45 & 95.20 \\
SRP-SFT & 91.16 & 96.06 & 80.75 & 96.00 \\
\bottomrule
\end{tabularx}
\caption{Reasoning checks for Qwen3.5-2B Base, Qwen3.5-0.8B Base, and
Granite-4.1-3B Instruct (percentages). Columns report GSM8K and MATH-500
pass@1/pass@16.}
\label{tab:cross_backbone_reasoning_three_a}
\end{table}

\paragraph{Adaptation depends on the benchmark.}
For Qwen3.5-2B, SRP-SFT exceeds the untouched Base on both MATH-500 metrics but
not on GSM8K. Granite also separates the two tasks: SRP-SFT has the highest
GSM8K pass@1 and MATH-500 pass@16, while LP-SFT leads MATH-500 pass@1. Keeping
all four metrics visible avoids reducing these distinct outcomes to a single
model ranking.

\begin{table}[!htbp]
\centering
\footnotesize
\setlength{\tabcolsep}{2.0pt}
\renewcommand{\arraystretch}{1.08}
\begin{tabularx}{\columnwidth}{@{}l*{4}{Y}@{}}
\toprule
Method & \multicolumn{2}{c}{GSM8K} & \multicolumn{2}{c}{MATH-500} \\
\cmidrule(lr){2-3}\cmidrule(lr){4-5}
& pass@1 & pass@16 & pass@1 & pass@16 \\
\midrule
\multicolumn{5}{@{}l}{\textit{Ministral-3-3B Instruct}} \\
\addlinespace[1pt]
Base & 80.20 & 97.35 & 66.53 & 93.80 \\
CE & 88.03 & 96.97 & 73.69 & 91.60 \\
LP-SFT & 89.93 & 97.27 & 74.76 & 91.80 \\
SRP-SFT & 90.36 & 97.42 & 75.30 & 91.20 \\
\midrule
\multicolumn{5}{@{}l}{\textit{Gemma-3-4B it}} \\
\addlinespace[1pt]
Base & 89.58 & 96.06 & 76.39 & 91.80 \\
CE & 77.79 & 96.89 & 57.55 & 85.00 \\
LP-SFT & 86.52 & 97.27 & 63.94 & 92.00 \\
SRP-SFT & 86.52 & 96.36 & 61.59 & 91.80 \\
\midrule
\multicolumn{5}{@{}l}{\textit{SmolLM3-3B Instruct}} \\
\addlinespace[1pt]
Base & 77.79 & 96.97 & 71.12 & 92.20 \\
CE & 89.63 & 96.44 & 80.76 & 93.60 \\
LP-SFT & 90.13 & 97.04 & 80.64 & 93.80 \\
SRP-SFT & 90.01 & 97.42 & 80.92 & 93.80 \\
\bottomrule
\end{tabularx}
\caption{Reasoning checks for Ministral-3-3B Instruct, Gemma-3-4B it, and
SmolLM3-3B Instruct (percentages). Columns report GSM8K and MATH-500
pass@1/pass@16.}
\label{tab:cross_backbone_reasoning_three_b}
\end{table}

\paragraph{Sampling comparisons.}
For Ministral, SRP-SFT leads GSM8K pass@1/pass@16 and MATH-500 pass@1,
while the untouched Base leads MATH-500 pass@16. On SmolLM3, SRP-SFT leads
MATH-500 pass@1 and ties LP-SFT at pass@16, whereas LP-SFT leads GSM8K
pass@1. Gemma shows a different trade-off: the untouched Base leads GSM8K
and MATH-500 pass@1, while LP-SFT leads MATH-500 pass@16.

We also evaluate code generation after domain-specific SFT on rStar-Coder
with 4{,}096-token sequences. Values below are EvalPlus pass@1 percentages.

\begin{table}[!htbp]
\centering
\footnotesize
\setlength{\tabcolsep}{1.7pt}
\renewcommand{\arraystretch}{1.10}
\begin{tabularx}{\columnwidth}{@{}l*{4}{Y}@{}}
\toprule
& \multicolumn{2}{c}{\shortstack{Qwen3.5-2B\\Base}}
& \multicolumn{2}{c}{\shortstack{Qwen3.5-0.8B\\Base}} \\
\cmidrule(lr){2-3}\cmidrule(lr){4-5}
Method & \shortstack{Human\\Eval+} & MBPP+ & \shortstack{Human\\Eval+} & MBPP+ \\
\midrule
Base & 38.17 & 37.51 & 20.67 & 19.64 \\
CE & 30.61 & 23.86 & 16.65 & 15.09 \\
LP-SFT & 38.75 & 34.27 & 21.34 & 18.64 \\
SRP-SFT & 39.76 & 35.21 & 21.89 & 20.17 \\
\bottomrule
\end{tabularx}
\caption{rStar-Coder adaptation on the two base-model backbones
(EvalPlus pass@1 percentages).}
\label{tab:cross_backbone_code_base}
\end{table}

\begin{table}[!htbp]
\centering
\footnotesize
\setlength{\tabcolsep}{3.0pt}
\renewcommand{\arraystretch}{1.08}
\begin{tabularx}{\columnwidth}{@{}l*{3}{Y}@{}}
\toprule
Method & HumanEval+ & MBPP+ & Mean \\
\midrule
\multicolumn{4}{@{}l}{\textit{Granite-4.1-3B Instruct}} \\
\addlinespace[1pt]
Base & 76.25 & 59.33 & 67.79 \\
CE & 65.49 & 51.71 & 58.60 \\
LP-SFT & 73.05 & 55.87 & 64.46 \\
SRP-SFT & 76.22 & 56.73 & 66.48 \\
\midrule
\multicolumn{4}{@{}l}{\textit{Ministral-3-3B Instruct}} \\
\addlinespace[1pt]
Base & 66.68 & 56.52 & 61.60 \\
CE & 54.24 & 47.57 & 50.91 \\
LP-SFT & 68.11 & 55.54 & 61.83 \\
SRP-SFT & 69.73 & 54.93 & 62.33 \\
\midrule
\multicolumn{4}{@{}l}{\textit{Gemma-3-4B it}} \\
\addlinespace[1pt]
Base & 63.02 & 67.78 & 65.40 \\
CE & 44.63 & 35.98 & 40.31 \\
LP-SFT & 58.93 & 61.57 & 60.25 \\
SRP-SFT & 63.02 & 64.70 & 63.86 \\
\midrule
\multicolumn{4}{@{}l}{\textit{SmolLM3-3B Instruct}} \\
\addlinespace[1pt]
Base & 56.95 & 49.87 & 53.41 \\
CE & 47.90 & 44.37 & 46.14 \\
LP-SFT & 57.71 & 48.19 & 52.95 \\
SRP-SFT & 58.29 & 49.66 & 53.98 \\
\bottomrule
\end{tabularx}
\caption{rStar-Coder adaptation on four instruction-tuned backbones
(EvalPlus pass@1 percentages). Mean averages HumanEval+ and MBPP+.}
\label{tab:cross_backbone_code_instruct}
\end{table}

\paragraph{Scope of the code comparison.}
Granite's SRP-SFT HumanEval+ score, 76.22, is close to the untouched Base's
76.25, whereas its MBPP+ score remains lower. Ministral also shows different
orderings across the two tasks: SRP-SFT leads HumanEval+, while the Base
leads MBPP+. Gemma and SmolLM3 also have complete two-task comparisons: SRP-SFT
matches the Gemma Base on HumanEval+ and trails it on MBPP+. On SmolLM3,
SRP-SFT exceeds Base on HumanEval+ but remains slightly below Base on MBPP+.

\begin{table}[!htbp]
\centering
\scriptsize
\setlength{\tabcolsep}{1.5pt}
\begin{tabularx}{\columnwidth}{@{}l*{4}{Y}@{}}
\toprule
Backbone & MATH & GSM8K & HE+ & MBPP+ \\
\midrule
Granite & -0.70 & +0.04 & +3.17 & +0.86 \\
SmolLM3 & +0.28 & -0.12 & +0.58 & +1.47 \\
Ministral & +0.54 & +0.43 & +1.62 & -0.61 \\
Gemma & -2.35 & +0.00 & +4.09 & +3.13 \\
Qwen 2B & +4.80 & +1.67 & +1.01 & +0.94 \\
Qwen 0.8B & +0.25 & -1.83 & +0.55 & +1.53 \\
\bottomrule
\end{tabularx}
\caption{Direct cross-backbone $\Delta_{\mathrm{LP}}=\mathrm{SRP}-\mathrm{LP}$
in pass@1 percentage points, using the same source data as
Figure~\ref{fig:cross_backbone_robustness}. Negative values are retained.}
\label{tab:cross_backbone_delta_lp}
\end{table}

\appsection{Parameter Sensitivity and Ablation}
\label{sec:sensitivity_ablation}
\label{app:sensitivity_ablation}

The parameter study separates three choices in SRP-SFT: the margin-demand
level $\rho$, the loss weight $\mu$, and the number of teacher
alternatives $K$. We use the held-out diagnostic for the structural measures
and report MATH-500 and EvalPlus scores under the same fixed evaluation
protocol. Figure~\ref{fig:parameter_sensitivity_sweeps} makes the response
curves visible across all reported arms. The completed scores peak at
$\rho=2.5$ and $\mu=3$. Among the
seven $K$ arms ($K\in\{1,2,3,5,7,10,20\}$), $K=5$ is highest. The
$K=20$ arm is included in the tabulated sweep with matching training,
MATH-500, EvalPlus, and held-out diagnostic measurements.

The sensitivity and ablation grid was evaluated on the held-out
diagnostic. ``MATH'' is MATH-500 pass@1, ``Code'' is the mean of HumanEval+
and MBPP+ pass@1, and ``Avg.'' is the arithmetic mean of these two domain
scores. The compact sensitivity table retains the main performance and
directional-preservation measures. Redundant per-dataset and auxiliary
diagnostics are omitted for readability. Flip and Violation are fractions.
Contr. is the mean positive logit-margin contraction in logit units.
The sensitivity reference setting is $\rho=0.5$, $\mu=1$, and $K=10$.
The headline SRP-SFT checkpoint uses $\rho=1$, $\mu=1$, and
$K=10$ as disclosed above. In Table~\ref{tab:sensitivity}, the
$\rho=1$ task scores are that headline evaluation: MATH pass@1 matches
the main table, and Code avg.\ is the mean of its HumanEval+ and MBPP+
pass@1. The $\mu=0$ row is likewise the headline CE checkpoint. A zero
auxiliary weight leaves only the target-token cross-entropy, on the same
478-step schedule as the other cells. Every remaining row is the
checkpoint trained for that sensitivity cell. The $\rho=0.5$ reference is
one of those checkpoints and is reused by the loss screen below.

To make the one-factor design explicit, the $\rho$ sweep varies $\rho$ while
holding $K=10$ and $\mu=1$. The $\mu$ sweep varies $\mu$ while holding
$K=10$ and $\rho=0.5$. And the $K$ sweep varies $K$ while holding
$\rho=0.5$ and $\mu=1$. Training settings other than the varied
coefficient stay fixed. $K=1$ is one such cell: its candidate set contains
one non-target token, so it forms no ordered pair and the auxiliary term
is zero. Held-out flip and contraction are $0.1541$ and $0.6007$, next to
$0.1540$ and $0.5995$ for headline CE. MATH pass@1 is $28.58$, next to
$27.98$. With deterministic attention and a locked RNG state, CE and
$K=1$ agree exactly over eight optimizer steps
(Appendix~\ref{app:algorithm}).

\begin{table}[!htbp]
\centering
\scriptsize
\setlength{\tabcolsep}{1.5pt}
\renewcommand{\arraystretch}{1.08}
\begin{tabularx}{\columnwidth}{@{}ll*{5}{Y}@{}}
\toprule
Grid & Value & \shortstack{MATH\\p@1 $\uparrow$}
& \shortstack{Code\\avg. $\uparrow$} & Avg. $\uparrow$
& Flip $\downarrow$ & Contr. $\downarrow$ \\
\midrule
$\rho$ & 0.00 & 41.02 & 53.08 & 47.05 & .0828 & .4123 \\
       & 0.25 & 47.80 & 53.57 & 50.69 & .0740 & .3549 \\
       & 0.50 & 50.12 & 54.65 & 52.39 & .0666 & .2663 \\
       & 0.75 & 52.78 & 58.56 & 55.67 & .0596 & .1563 \\
       & 1.00 & 54.56 & 59.54 & 57.05 & .0565 & .1098 \\
       & 1.25 & 56.83 & 61.51 & 59.17 & .0583 & .0950 \\
       & 1.50 & 58.56 & 63.21 & 60.89 & .0627 & .0982 \\
       & 1.75 & 57.70 & 62.89 & 60.29 & .0700 & .1130 \\
       & 2.00 & 58.75 & 62.77 & 60.76 & .0774 & .1325 \\
       & 2.25 & 59.04 & 63.45 & 61.24 & .0865 & .1618 \\
       & 2.50 & 60.45 & 62.68 & 61.56 & .0937 & .1891 \\
       & 2.75 & 58.61 & 61.66 & 60.14 & .1009 & .2218 \\
       & 3.00 & 58.05 & 61.44 & 59.75 & .1070 & .2540 \\
\midrule
$\mu$ & 0.00 & 27.98 & 44.42 & 36.20 & .1540 & .5995 \\
       & 0.10 & 34.15 & 48.13 & 41.14 & .1160 & .4061 \\
       & 0.30 & 38.88 & 52.46 & 45.67 & .0971 & .3513 \\
       & 1.00 & 50.12 & 54.65 & 52.39 & .0666 & .2663 \\
       & 3.00 & 52.83 & 57.29 & 55.06 & .0453 & .2131 \\
       & 5.00 & 49.40 & 56.54 & 52.97 & .0397 & .1998 \\
       & 10.00 & 39.92 & 57.17 & 48.55 & .0356 & .1908 \\
\midrule
$K$ & 1  & 28.58 & 45.62 & 37.10 & .1541 & .6007 \\
     & 2  & 33.88 & 47.54 & 40.71 & .1359 & .4843 \\
     & 3  & 49.95 & 55.53 & 52.74 & .0829 & .3196 \\
     & 5  & 50.66 & 55.79 & 53.23 & .0718 & .2833 \\
     & 7  & 48.23 & 55.81 & 52.02 & .0683 & .2720 \\
     & 10 & 50.12 & 54.65 & 52.39 & .0666 & .2663 \\
     & 20 & 49.48 & 55.01 & 52.24 & .0668 & .2678 \\
\bottomrule
\end{tabularx}
\caption{One-factor sensitivity of $\rho$, $\mu$, and $K$. MATH is MATH-500
pass@1. Code avg. Averages HumanEval+ and MBPP+ pass@1. Avg. Averages the two
domains. Flip and Contr. are lower-is-better diagnostics.}
\label{tab:sensitivity}
\end{table}

\paragraph{Reading the sweeps.}
Each block in Table~\ref{tab:sensitivity} varies one control around the
stated sensitivity reference setting. Increasing $\mu$ past the highest
observed overall setting continues to reduce Flip.

Figure~\ref{fig:parameter_sensitivity_frontier} summarizes the
performance--geometry trade-off. The $\rho=2.5$ point has the highest
observed overall score (61.56), while $\rho=1.5$ has score 60.89 and lower
contraction (0.0982). Relative to $\rho=1$, $\rho=2.5$ improves Avg.
From 57.05 to 61.56 while increasing Flip from 0.0565 to 0.0937 and
contraction from 0.1098 to 0.1891. Neither measured preservation nor performance is
monotone in $\rho$. The $K$ sweep covers seven
values, including $K=20$.

\begin{figure}[!htbp]
\centering
\includegraphics[width=\columnwidth]{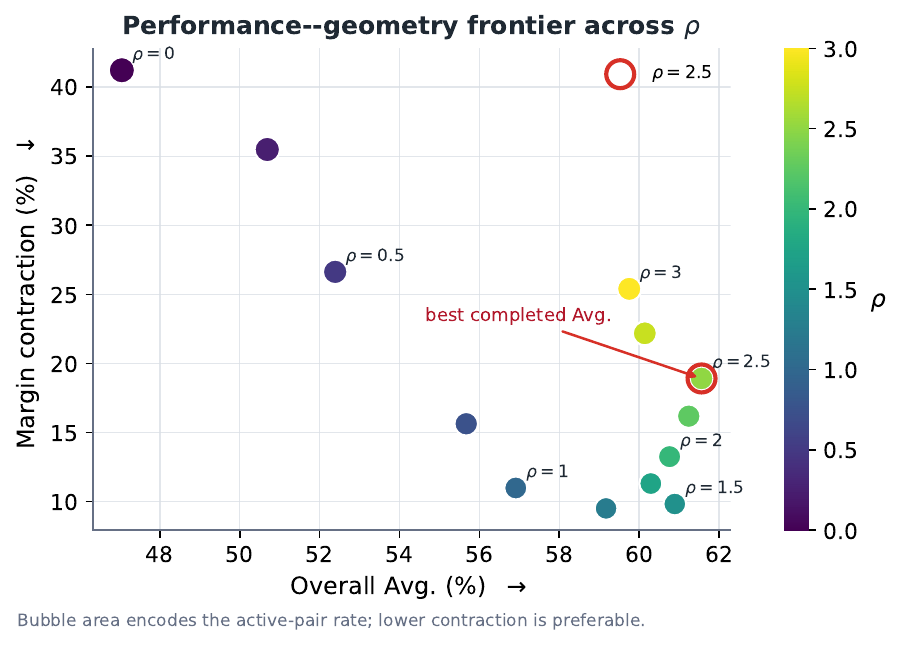}
\caption{Performance--geometry frontier over the reported $\rho$ sweep.
Bubble area is the active-pair rate. Lower contraction is preferable.}
\label{fig:parameter_sensitivity_frontier}
\end{figure}

\begin{table}[!htbp]
\centering
\scriptsize
\setlength{\tabcolsep}{1.5pt}
\renewcommand{\arraystretch}{1.10}
\begin{tabularx}{\columnwidth}{@{}>{\raggedright\arraybackslash}p{0.28\columnwidth}*{5}{Y}@{}}
\toprule
Variant & \shortstack{UF\\NLL} & \shortstack{Math/\\Code} & Flip & Contr. & KL \\
\midrule
CE & .917 & 38.94 & .1540 & .5995 & .3592 \\
$\rho=0$ & .919 & 49.06 & .0828 & .4123 & .0797 \\
Include-GT & .922 & 52.58 & .0705 & .2750 & .0564 \\
$\rho=.5$ (sensitivity ref.) & .922 & 53.14 & .0666 & .2663 & .0523 \\
Adjacent-rank pairs & .927 & 55.26 & .0573 & .3238 & .0470 \\
Top1-vs-rest & .919 & 50.72 & .0946 & .2961 & .0873 \\
\bottomrule
\end{tabularx}
\caption{Objective and sparse-pair ablations around the
$\rho=0.5$ sensitivity reference. The Random-$K$ arm is omitted.}
\label{tab:ablation_followup}
\end{table}

\paragraph{Pair selection.}
Adjacent-rank pairs in Table~\ref{tab:ablation_followup} give a different
trade-off, with higher Math/Code than the $\rho=0.5$ sensitivity reference
and higher contraction. The comparison reports task scores together
with Flip, contraction, and KL.

\appsubsection{Headline Pair-Graph Controls}
\label{app:rho1_efficiency}

\begin{table}[!htbp]
\centering
\scriptsize
\setlength{\tabcolsep}{1.3pt}
\begin{tabularx}{\columnwidth}{@{}l*{6}{Y}@{}}
\toprule
Variant & MATH p@1 & MATH p@4 & Code mean & Flip & Contr. & KL \\
\midrule
Full pairs & 54.56 & 80.50 & 69.50 & 0.0565 & 0.1098 & 0.0748 \\
Adjacent & 57.65 & 82.02 & 69.04 & 0.0520 & 0.1346 & 0.0431 \\
\bottomrule
\end{tabularx}
\caption{SRP pair-graph controls at $\rho=1$, $K=10$, $\mu=1$
(training seed 1234).
Full pairs is the headline checkpoint, and its task scores are the
main-table evaluation. Adjacent pairs is a separately trained
checkpoint. Task scores are percentages. Code mean averages the four
matched HumanEval+/MBPP+ pass@1/pass@4 metrics. Diagnostics use the
998-example held-out cache. Both checkpoints have 478 training steps.
Flip is a fraction. Contr. is in logit units.}
\label{tab:rho1_pair_controls}
\end{table}

The pair-graph comparison in Table~\ref{tab:rho1_pair_controls}
uses the headline full-pair checkpoint and a separately trained
adjacent-rank checkpoint. Adjacent pairs have higher MATH scores and a
slightly lower code mean, with lower Flip and KL but higher contraction.
The two runs share the training implementation, environment, and common
hyperparameters. The pair set is the changed factor.

\begin{table}[!htbp]
\centering
\scriptsize
\begin{tabularx}{\columnwidth}{@{}lYYYY@{}}
\toprule
Method & Pairs & Tokens/s & Wall (s) & Peak (GiB) \\
\midrule
CE & -- & 6748 & 175.31 & 19.91 \\
LP-SFT & -- & 6743 & 180.81 & 19.91 \\
SRP full & 45 & 6625 & 176.41 & 19.91 \\
SRP adjacent & 9 & 6514 & 181.48 & 19.91 \\
\bottomrule
\end{tabularx}
\caption{Measured 20-step efficiency pilots with the same eight-GPU
allocation and training configuration. Throughput excludes the first five
warmup steps and checkpoint I/O. Wall time covers the training run.
Peak is the maximum allocated memory over ranks. Pairs counts auxiliary
rank constraints per position. CE and LP-SFT have no auxiliary pairs.}
\label{tab:pair_efficiency}
\end{table}

Full and adjacent graphs use 45 and 9 pairs at $K=10$, respectively.
In this pilot the adjacent graph records 6{,}514 tokens/s and 19.91 GiB,
against 6{,}625 tokens/s and 19.91 GiB for the full graph.
The adjacent checkpoint's downstream scores are in
Table~\ref{tab:rho1_pair_controls}. Both graphs reuse the same top-$K$ cache,
whose 56 NumPy arrays occupy 1.143 GiB.
CE reads the supervised cache fields. LP-SFT and SRP-SFT additionally read
the top-$K$ fields. A new timed reconstruction of the shared full cache
takes 486.63 seconds on eight GPUs, including model loading, tokenization,
teacher computation, and writing. Of 61,135 requested rows, 61,095 are
valid and 40 empty rows are skipped. The 56 arrays occupy 1.143 GiB.
This measurement is the reconstruction of the shared cache used by
LP-SFT and both SRP pair graphs.

\appsubsection{LP-SFT/ListNet Equivalence and Training Audit}
\label{app:lp_listnet_audit}
For frozen teacher probabilities $q$ and student probabilities $p$ on
the same candidate set,
\begin{equation}
\mathcal L_{\mathrm{LP}}=\mathcal L_{\mathrm{ListNet}}-H(q).
\end{equation}
With matching temperature, candidate mask, reduction, and coefficient,
$\nabla_\theta\mathcal L_{\mathrm{LP}}=
\nabla_\theta\mathcal L_{\mathrm{ListNet}}$. Both objectives include the
same target CE outside the alternatives.

We checked the current implementation on identical logits through the full
CE-plus-auxiliary code path, with coefficient one and temperature one.
Eight cases cover FP32/BF16 logits, valid/padded/empty positions, and target
exclusion. The student gradients are bitwise equal in these CPU checks.
The loss difference equals the reduced teacher entropy to floating-point
precision.

\paragraph{Controlled training audit.}
We also compare LP-SFT and ListNet over eight optimizer updates of
Qwen3-4B-Base using BF16 DeepSpeed ZeRO-2 on eight ranks. Both runs use
seed 1234, the same initialization, cache, data order, optimizer and
scheduler, with $K=10$, temperature one, and $\mu=1$. The microbatch size
is one per rank, with 16 accumulation steps and effective batch size 128.
We enable deterministic algorithms and SDPA math attention, and disable
communication overlap. Table~\ref{tab:lp_listnet_training_identity}
summarizes the paired traces: all 1,024 microbatches have matching input,
RNG, candidate-mask, and selected-logit-gradient hashes. CE losses and
optimizer-step gradient norms agree exactly. Parameter hashes recorded
on rank 0 match at initialization and after every update. The maximum
residual in $\mathcal L_{\mathrm{ListNet}}-\mathcal L_{\mathrm{LP}}-H(q)$
is $7.15\times10^{-7}$. These results verify the predicted equivalence
through eight updates of the controlled training execution.

\begin{table}[!htbp]
\centering
\scriptsize
\setlength{\tabcolsep}{2.0pt}
\begin{tabularx}{\columnwidth}{@{}Xrr@{}}
\toprule
Paired check & Pairs & Result \\
\midrule
Input and RNG hashes & 1,024 & All match \\
Candidate-mask hashes & 1,024 & All match \\
Selected-logit-gradient hashes & 1,024 & All match \\
Maximum CE loss difference & 1,024 & $0$ \\
Maximum loss--entropy residual & 1,024 & $7.15\!\times\!10^{-7}$ \\
Maximum update gradient-norm difference & 64 & $0$ \\
Parameter hashes (rank 0) & 9 & All match \\
\bottomrule
\end{tabularx}
\caption{LP-SFT/ListNet training identity audit. Microbatch checks cover
$8$ ranks $\times$ $8$ updates $\times$ $16$ accumulation steps. Gradient
norms cover all $64$ rank--update pairs. The nine parameter snapshots
comprise initialization and eight post-update states on rank 0.
The loss residual subtracts the reduced teacher entropy from the
ListNet-minus-LP total objective. Its absolute tolerance is $10^{-5}$.}
\label{tab:lp_listnet_training_identity}
\end{table}

The reported LP and ListNet runs use a 478-step schedule, with micro-batch
one and accumulation 16. They use the same training implementation,
environment, data-loader configuration, and common training hyperparameters.
ListNet is a separately trained instance of the LP-equivalent objective.
The matched training audit above directly verifies implementation identity.

\appsubsection{Ranking-Loss Form Ablation}
\label{sec:loss_ablation}

The parameter grid above varies SRP hyperparameters. We additionally compare
the loss form at fixed coefficients and a shared Qwen3-4B-Base
initialization. All rows share the same training implementation,
environment, and common training hyperparameters.
Candidate-set details are described in the full screen below. Table~\ref{tab:loss_ablation}
shows representative alternatives. The complete fixed-configuration screen is given
in Appendix~\ref{app:loss_ablation_full}.
Table~\ref{tab:ranking_matched_available} gives the four-metric code mean
for ListMLE, Margin-MSE, and SRP-SFT at $\mu=1$.

\begin{table}[!htbp]
\centering
\scriptsize
\setlength{\tabcolsep}{1.1pt}
\begin{tabularx}{\columnwidth}{@{}l*{7}{Y}@{}}
\toprule
Method & \multicolumn{2}{c}{MATH} & \multicolumn{2}{c}{HE+} & \multicolumn{2}{c}{MBPP+} & Code mean \\
\cmidrule(lr){2-3}\cmidrule(lr){4-5}\cmidrule(lr){6-7}
& p@1 & p@4 & p@1 & p@4 & p@1 & p@4 & \\
\midrule
ListMLE & 52.21 & 79.58 & 62.35 & 83.85 & 55.46 & 73.74 & 68.85 \\
Margin-MSE & 49.19 & 79.77 & 61.52 & 84.66 & 54.85 & 74.64 & 68.92 \\
SRP-SFT & 54.56 & 80.50 & 62.84 & 83.87 & 56.24 & 75.06 & 69.50 \\
\bottomrule
\end{tabularx}
\caption{One-epoch ranking controls at $\mu=1$
(seed 1234, percentages).}
\label{tab:ranking_matched_available}
\end{table}

\begin{table}[!htbp]
\centering
\scriptsize
\setlength{\tabcolsep}{1.0pt}
\renewcommand{\arraystretch}{1.10}
\begin{tabularx}{\columnwidth}{@{}>{\raggedright\arraybackslash}p{0.25\columnwidth}*{6}{Y}@{}}
\toprule
Loss & MATH & HE+ & MBPP+ & Mean & \shortstack{Flip\\$\downarrow$} & \shortstack{Contr.\\$\downarrow$} \\
\midrule
CE & 27.98 & 44.18 & 44.66 & 38.94 & 0.1540 & 0.5995 \\
LP-SFT & 49.56 & 57.90 & 51.87 & 53.11 & 0.0741 & 0.2681 \\
RankOnly ($\rho=0$) & 41.89 & 55.37 & 49.96 & 49.07 & .0828 & .4124 \\
RankNet & 46.41 & 59.09 & 52.63 & 52.71 & .0853 & .1872 \\
ListMLE & 52.15 & 61.71 & 56.18 & 56.68 & .0691 & .1271 \\
Margin-MSE & 48.68 & 61.07 & 54.72 & 54.82 & .0380 & .1522 \\
SRP-full ($\rho=1$) & 54.56 & 62.84 & 56.24 & 57.88 & .0565 & .1098 \\
SRP-SFT ($\rho=.5$) & 50.12 & 57.23 & 52.08 & 53.14 & .0666 & .2663 \\
\bottomrule
\end{tabularx}
\caption{Representative ranking-loss screen. MATH, HE+, and MBPP+
are pass@1. Mean averages them. Flip is a fraction and Contr. is in logit
units. SRP-full ($\rho=1$) is the headline checkpoint and matches
Table~\ref{tab:code_benchmark}. SRP-SFT ($\rho=.5$) is the
sensitivity-reference checkpoint in Table~\ref{tab:sensitivity}. Its
HumanEval+ and MBPP+ pass@1 average to that table's code score of
54.65 before rounding.}
\label{tab:loss_ablation}
\end{table}

RankOnly lowers inversion relative to CE but remains below the margin-aware
SRP variants in this screen. SRP-full has the strongest three-task mean,
whereas Margin-MSE has fewer inversions (0.0380 versus 0.0565).
The RLKD-style top-$K$ approximation
is specified in Appendix~\ref{app:loss_ablation_full}.

\appsection{Full Ranking-Loss Ablation}
\label{app:loss_ablation_full}

Table~\ref{tab:ranking_loss_performance} in
Section~\ref{sec:ablation} reports the full downstream performance panel for
the headline $\rho=1$ method and the other loss-form controls. This appendix
retains the companion relation and directional diagnostics. The
RLKD-style row is a controlled top-$K$ variant of the cited method. The fixed-candidate rows reuse the teacher top-$K$ non-target cache,
target exclusion, CE reduction, and valid-token normalization. BiLD T+S
and RLKD-style union additionally change the candidate set and use a
teacher-logit floor for student-only candidates. All methods use the same
training implementation, environment, and common training hyperparameters.
Every auxiliary row uses $\mu=1$.
SRP-SFT uses the fixed headline configuration.

The following definitions specify the objectives used in
Tables~\ref{tab:ranking_loss_performance} and~\ref{tab:loss_ablation}.
For each supervised position $t$, let $\mathcal C_t$ be the cached teacher
top-$K$ non-target alternatives ($K=10$), let
$\Delta^M_{ij,t}=z^M_{t,i}-z^M_{t,j}$ for $M\in\{T,S\}$, and let
$\mathcal P_t=\{(i,j)\in\mathcal C_t^2:\Delta^T_{ij,t}>0\}$. Thus the
pairwise rows use every teacher-preferred ordered pair
and never add the reverse orientation. Invalid or padded positions are
removed before reduction. Every auxiliary is averaged over the valid
response positions and added to the ordinary target CE with $\mu=1$. The
screen uses the shared reference cache, a checkpoint after one processed
epoch, seed $1234$, effective batch size $128$, peak AdamW learning rate
$2\times10^{-5}$, and 2,048-token context. All rows
use a 478-step cosine horizon with 15 warmup steps.
The CE row is the same target CE with no auxiliary term. These rows illustrate
three kinds of supervision. LP-SFT and ListNet share the same one. Pointwise supervision gives an individual
token a target label or score, as CE does for the observed response token.
Pairwise supervision compares two candidates and uses their relation, as in
RankOnly, RankNet, Margin-MSE, and SRP-SFT. Listwise supervision uses the full
candidate set at once, as in ListNet and ListMLE. Pointwise supervision is
the CE term already present in every row.

The distinction between the two listwise controls is also specific. ListNet
matches a teacher probability distribution over the candidate set, whereas
ListMLE uses only the teacher's sorted permutation and a Plackett--Luce
likelihood. At matched temperature, candidate mask, reduction, and weight, LP-SFT's
forward KL differs from ListNet only by the frozen teacher entropy. They
therefore have the same student gradient and distribution-matching geometry.

\paragraph{Pairwise objectives.}
RankOnly is the $\rho=0$ special case of the SRP implementation,
$|\mathcal P_t|^{-1}\sum_{(i,j)\in\mathcal P_t}[-\Delta^S_{ij,t}]_+$. It
enforces order but has no teacher-margin target. SRP-full (the $\rho=1$ row,
also called SRP-SFT in the main performance panel) uses
$|\mathcal P_t|^{-1}\sum_{(i,j)\in\mathcal P_t}[\Delta^T_{ij,t}-\Delta^S_{ij,t}]_+$, while the
appendix SRP-SFT ($\rho=.5$) row replaces the target by
$.5\Delta^T_{ij,t}$. Margin-MSE uses the symmetric exact-gap objective
$|\mathcal P_t|^{-1}\sum_{(i,j)\in\mathcal P_t}(\Delta^S_{ij,t}-\Delta^T_{ij,t})^2$.
RankNet \citep{burges2005ranknet} uses $\operatorname{softplus}(-\Delta^S_{ij,t}/\tau)$ on the same
pairs with $\tau=1$. Smooth-SRP uses the raw (unshifted) smooth hinge
$\tau\operatorname{softplus}((.5\Delta^T_{ij,t}-\Delta^S_{ij,t})/\tau)$ with
$\tau=.5$. Consequently it remains nonzero on the feasible side, unlike the
hard SRP hinge. Kendall-$\tau$ is a controlled soft-concordance proxy, not
the cited paper's official implementation: it averages
$1-\tanh(\Delta^S_{ij,t}/\tau)\tanh(\Delta^T_{ij,t}/\tau)$ with $\tau=1$.

\paragraph{Listwise and correlation objectives.}
ListNet \citep{cao2007listnet} forms $q_i=\operatorname{softmax}(z^T_i)$ and
$p_i=\operatorname{softmax}(z^S_i)$ only over $\mathcal C_t$, and uses
$-\sum_{i\in\mathcal C_t}q_i\log p_i$ (temperature $1$). ListMLE \citep{xia2008listmle} sorts
$\mathcal C_t$ by descending teacher logit, $\pi$, and uses the normalized
Plackett--Luce negative log-likelihood
$|\mathcal C_t|^{-1}\sum_r[-z^S_{\pi_r}+
\log\sum_{u\ge r}\exp z^S_{\pi_u}]$. It uses the teacher permutation.
Spearman first assigns each candidate the
pairwise-sigmoid soft rank
$1+\sum_{j\ne i}\sigma((z_j-z_i)/\tau)$, then minimizes one minus the
Pearson correlation of teacher and student soft ranks on $\mathcal C_t$
($\tau=1$). It adapts the rank-correlation statistic \citep{spearman1904}.

\paragraph{Cited-method approximations.}
BiLD teacher-led applies
$\operatorname{KL}(\operatorname{softmax}(\Delta^T/T)\,\|\,\operatorname{softmax}(\Delta^S/T))$
to all pairs in $\mathcal P_t$ with $T=3$ and the shared teacher top-$K$ set.
It is therefore a teacher-led controlled approximation. BiLD T+S uses
$k=8$ and adds a second branch over the student's target-excluded top-$8$.
For a student-only token, the cached teacher logit is set to
$(\min z^T|_{\mathcal C_t})-1$, exactly the floor used in the code. RLKD-style
union uses differentiable Spearman on the deduplicated union of the teacher
top-$5$ and student's target-excluded top-$5$ candidates.
Student-only teacher logits use the same floor and the rank temperature is
$1$. LP-SFT is the separate local forward KL
$\operatorname{KL}(q\|p)$ on the same target-excluded candidate set, with the
target CE retained outside that set.

\begin{table}[!htbp]
\centering
\scriptsize
\setlength{\tabcolsep}{1.8pt}
\renewcommand{\arraystretch}{1.06}
\begin{tabularx}{\columnwidth}{@{}l>{\raggedright\arraybackslash}Xrr@{}}
\toprule
\multicolumn{4}{@{}l}{\textbf{Relation and preservation diagnostics}} \\
\addlinespace[2pt]
Loss & Relation & Flip & Contr. \\
\midrule
CE & None & 0.1540 & 0.5995 \\
LP-SFT & Distribution & 0.0741 & 0.2681 \\
RankOnly ($\rho=0$) & Pairwise order & .0828 & .4124 \\
RankNet & Pairwise order & .0853 & .1872 \\
ListMLE & Listwise order & .0691 & .1271 \\
Spearman & Rank correlation & .0757 & .2844 \\
ListNet$^{\dagger}$ & Distribution (LP-equivalent) & .0739 & .2675 \\
Margin-MSE & Pairwise gap & .0380 & .1522 \\
BiLD teacher-led & Diff-KL & .0539 & .2549 \\
BiLD T+S & Diff-KL & .0855 & .3415 \\
RLKD-style union & Rank correlation & .0682 & .2509 \\
Kendall-$\tau$ & Pair concordance & .0924 & .2352 \\
Smooth-SRP & Pairwise gap bound & .0732 & .1602 \\
SRP-SFT ($\rho=1$) & Pairwise gap bound & .0565 & .1098 \\
SRP-SFT ($\rho=.5$) & Pairwise gap bound & .0666 & .2666 \\
\bottomrule
\end{tabularx}
\caption{Relation and directional diagnostics for the ranking-loss screen.
Flip is a fraction and Contr. is in logit units. Both are lower-is-better.
$^{\dagger}$ListNet is LP-equivalent.
Performance is in Table~\ref{tab:ranking_loss_performance}.}
\label{tab:ranking_loss_diagnostics}
\end{table}

\paragraph{Relating performance and geometry.}
Tables~\ref{tab:ranking_loss_performance} and~\ref{tab:ranking_loss_diagnostics}
separate task performance from the preserved relation. Margin-MSE has the
lowest displayed Flip rate, whereas SRP-SFT with $\rho=1$ has the highest
three-task mean. These properties therefore select different rows. The
supplementary $\rho=0.5$ arm and its performance remain in the compact
appendix screen (Table~\ref{tab:loss_ablation}).

RankNet continually rewards larger positive pairwise margins. SRP's
auxiliary becomes inactive when the requested bound is met. The relation
labels specify the controlled implementation compared, including the
approximate BiLD and RLKD-style variants.

\appsection{Entropy-Matched Temperature Control}
\label{app:temperature}

This appendix records the diagnostic behind
Section~\ref{sec:temperature}. The downstream screen itself is
Table~\ref{tab:tempmatch}.

\begin{table}[!htbp]
\centering
\scriptsize
\setlength{\tabcolsep}{1.6pt}
\renewcommand{\arraystretch}{1.10}
\begin{tabularx}{\columnwidth}{@{}l*{6}{Y}@{}}
\toprule
Method & $H$ & $\alpha^{*}$ & $1/\alpha^{*}$ & Contr. & Flip & Viol. \\
\midrule
CE & 0.906 & 1.081 & 0.925 & 0.600 & 0.154 & 0.501 \\
LP-SFT & 0.938 & 1.099 & 0.910 & 0.268 & 0.074 & 0.513 \\
SRP-SFT & 0.802 & 1 & 1 & 0.110 & 0.056 & 0.229 \\
\bottomrule
\end{tabularx}
\caption{Mechanism diagnostics on the held-out cache
(280{,}908 response positions).
$H$ is mean full-vocabulary entropy.
Contr./Flip are contraction/inversion, and Viol.
is the fraction below the $\rho=1$ margin. CE and LP-SFT $\alpha^{*}$ match the SRP-SFT entropy.}
\label{tab:sharpness}
\end{table}

\appsubsection{Sharpness and the Fitted Scale}
\label{app:temperature_fit}

On 280{,}908 held-out response positions, the SRP-SFT
mean full-vocabulary
entropy is $H_{\mathrm{SRP}}=0.802$
(Table~\ref{tab:sharpness}). For each frozen baseline we solve
\begin{equation}
\alpha^{*}
=
\arg\min_{\alpha\in[0.5,\,2]}
\bigl|
H(\mathrm{softmax}(\alpha z))
-
H_{\mathrm{SRP}}
\bigr|.
\end{equation}
The LP solution is interior, $\alpha^{*}=1.0993$. CE has an interior solution
$\alpha^{*}=1.081$ against the same SRP-SFT entropy. Equivalent model temperature
is $1/\alpha^{*}$. Logits are multiplied by $\alpha^{*}$ and then sampled at
temperature $0.6$ on MATH-500 and $0.8$ on HumanEval+ and MBPP+.
The fit uses the held-out SRP-SFT entropy as its only target.

\appsubsection{Margin Geometry}
\label{app:temperature_geometry}

The following two tables report per-pair margin ratios and
violation-conditioned bins. The margin-geometry panel
(Table~\ref{tab:geometry}) uses $\alpha^{*}=1.116$, not the
$1.0993$ scale used in the downstream control
(Table~\ref{tab:tempmatch}). The geometry panel uses $\alpha^{*}=1.116$.
The downstream control uses $1.0993$. On teacher-preferred pairs with reference
margin above $0.1$ (9.90 million pairs), multiplying LP-SFT logits by
$\alpha^{*}=1.116$ multiplies the mean margin ratio by the same factor,
$0.934\rightarrow 1.043$, and leaves the coefficient of variation and the
through-origin $R^{2}$ unchanged at 1.952 and 0.833. SRP-SFT has mean
ratio 1.244, coefficient of variation 1.346, $R^{2}=0.898$, and a
below-margin rate of 0.229, compared with 0.447 for LP-TempMatch and
0.539 for unscaled LP-SFT. Global scaling moves every retained margin
proportionally. SRP-SFT changes which margins fall short of the reference.

\begin{table}[!htbp]
\centering
\footnotesize
\setlength{\tabcolsep}{1.8pt}
\renewcommand{\arraystretch}{1.10}
\begin{tabularx}{\columnwidth}{@{}l*{4}{Y}@{}}
\toprule
Method & \shortstack{Mean\\ratio} & CV & $R^{2}$ & \shortstack{Below\\$\rho$} \\
\midrule
CE & 0.942 & 5.829 & 0.360 & 0.514 \\
LP-SFT & 0.934 & 1.952 & 0.833 & 0.539 \\
LP-TempMatch & 1.043 & 1.952 & 0.833 & 0.447 \\
SRP-SFT & 1.244 & 1.346 & 0.898 & 0.229 \\
\bottomrule
\end{tabularx}
\caption{Teacher-preferred margin geometry for $\Delta^{T}>0.1$.
Mean ratio is $\Delta^{S}/\Delta^{T}$.
CV is its coefficient of variation. $R^{2}$ is the through-origin fit.
LP-TempMatch uses $\alpha^{*}=1.116$.}
\label{tab:geometry}
\end{table}

\appsubsection{Violation-Conditioned Diagnostics}
\label{app:temperature_violation}

The entropy differences also vary across positions
(Table~\ref{tab:violation}), grouped by the CE checkpoint's violation rate.
In the highest bin, SRP-SFT has full-vocabulary entropy 0.165 below CE,
while LP-SFT is 0.016 below CE. SRP-SFT's contraction and flip reductions
are larger in higher CE-violation bins.

\begin{table}[!htbp]
\centering
\scriptsize
\setlength{\tabcolsep}{1.2pt}
\renewcommand{\arraystretch}{1.10}
\begin{adjustbox}{max width=\columnwidth}
\begin{tabular}{@{}lrrrrrrr@{}}
\toprule
CE viol.
& $n$
& \multicolumn{3}{c}{LP $-$ CE}
& \multicolumn{3}{c}{SRP $-$ CE} \\
\cmidrule(lr){3-5}\cmidrule(lr){6-8}
& & $\Delta$C & $\Delta$F & $\Delta H$
& $\Delta$C & $\Delta$F & $\Delta H$ \\
\midrule
$(0, .25]$ & 10{,}898 & $-.019$ & $-.011$ & $+.114$
& $-.117$ & $-.031$ & $-.044$ \\
$(.25, .50]$ & 123{,}110 & $-.246$ & $-.061$ & $+.066$
& $-.425$ & $-.095$ & $-.096$ \\
$(.50, .75]$ & 133{,}563 & $-.501$ & $-.107$ & $+.018$
& $-.780$ & $-.152$ & $-.143$ \\
$(.75, 1]$ & 13{,}315 & $-.705$ & $-.140$ & $-.016$
& $-1.098$ & $-.188$ & $-.165$ \\
\bottomrule
\end{tabular}
\end{adjustbox}
\caption{Diagnostic differences conditioned on the CE checkpoint's
violation rate. $\Delta$C and $\Delta$F are contraction and flip relative
to that CE checkpoint. The 22 zero-violation positions are omitted.}
\label{tab:violation}
\end{table}

\appsection{Convergence Data}
\label{app:convergence_data}

The following tables provide the numerical values plotted in
Figure~\ref{fig:convergence_means}. MATH mean averages MATH-500 pass@1 and
pass@4. Code mean averages HumanEval+ and MBPP+ pass@1 and pass@4. Values are
percentages, and \texttt{--} denotes a checkpoint whose downstream evaluation was not
available in the convergence report.
The step-0 Base row is this run's evaluation of the untouched checkpoint.
On MATH-500, this row and Table~\ref{tab:main_benchmark} share the
evaluation manifest and the sampling protocol: revision
\texttt{6e4ed1a2a79a}, all 500 test problems, the same prompt,
temperature 0.6, top-$p$ 1, 16 samples, and a 4{,}096-token cap.
Later checkpoints in this table use that same manifest and protocol.
The generations are a separate draw, so pass@1/pass@4 are 48.43/72.06
here (mean 60.24) and 48.11/71.23 in the main table (mean 59.67).
The trajectory is anchored to this draw.
The code entry uses temperature 0.8, top-$p$ 1, 20 samples, 768 new tokens,
and a 2{,}048-token context (mean 59.17). Table~\ref{tab:code_benchmark}
uses 4{,}096/8{,}192 and averages to 58.92.

\begin{table}[!htbp]
\centering
\footnotesize
\setlength{\tabcolsep}{2.2pt}
\renewcommand{\arraystretch}{1.02}
\begin{tabularx}{\columnwidth}{@{}l*{6}{Y}@{}}
\toprule
Epoch & Base & CE & LP-SFT & DFT & ASFT & SRP-SFT \\
\midrule
0.0 & 60.24 & -- & -- & -- & -- & -- \\
0.5 & -- & 43.22 & 61.70 & 50.30 & 52.27 & 64.21 \\
1.0 & -- & 40.00 & 59.61 & 24.36 & 37.18 & 60.32 \\
1.5 & -- & 46.72 & 60.16 & 28.13 & 57.97 & 58.42 \\
2.0 & -- & 42.50 & 53.04 & 31.68 & 46.70 & 56.54 \\
2.5 & -- & 48.93 & 57.10 & 33.27 & 46.85 & 59.30 \\
3.0 & -- & 50.65 & 57.19 & 30.33 & 34.16 & 60.24 \\
3.5 & -- & 51.00 & 59.85 & 34.64 & 45.38 & 61.25 \\
4.0 & -- & 50.69 & 58.56 & 35.60 & 48.05 & 61.58 \\
4.5 & -- & 50.86 & 60.36 & 38.66 & 50.11 & 61.00 \\
5.0 & -- & 51.42 & 60.36 & 34.68 & 50.31 & 61.49 \\
5.5 & -- & 51.62 & 60.51 & 39.40 & 49.19 & 62.84 \\
6.0 & -- & 51.72 & 60.05 & 39.33 & 52.29 & 63.04 \\
6.5 & -- & 51.14 & 59.85 & 33.13 & 54.82 & 62.24 \\
7.0 & -- & 51.74 & 59.04 & 41.45 & 53.06 & 63.38 \\
7.5 & -- & 50.24 & 59.09 & 39.32 & 51.55 & 62.97 \\
8.0 & -- & 50.75 & 60.25 & 40.23 & 53.10 & 62.56 \\
8.5 & -- & 50.09 & 60.91 & -- & 53.10 & 61.98 \\
9.0 & -- & 50.51 & 61.11 & -- & 53.38 & 62.59 \\
9.5 & -- & 49.06 & 60.13 & -- & 53.68 & 62.43 \\
10.0 & -- & 48.87 & -- & -- & 54.02 & 62.58 \\
\bottomrule
\end{tabularx}
\caption{MATH mean convergence values underlying Figure~\ref{fig:convergence_means}.
Step 0 is this run's Base draw under the same MATH-500 manifest and
sampling protocol as Table~\ref{tab:main_benchmark}.}
\label{tab:convergence_math_mean}
\end{table}

\begin{table}[!htbp]
\centering
\footnotesize
\setlength{\tabcolsep}{2.2pt}
\renewcommand{\arraystretch}{1.02}
\begin{tabularx}{\columnwidth}{@{}l*{6}{Y}@{}}
\toprule
Epoch & Base & CE & LP-SFT & DFT & ASFT & SRP-SFT \\
\midrule
0.0 & 59.17 & -- & -- & -- & -- & -- \\
0.5 & -- & 55.66 & 66.12 & 61.03 & 66.34 & 66.31 \\
1.0 & -- & 52.87 & 64.32 & 60.12 & 66.66 & 65.25 \\
1.5 & -- & 58.47 & 65.75 & 61.47 & 65.56 & 65.46 \\
2.0 & -- & 58.01 & 65.73 & 60.98 & 65.66 & 66.47 \\
2.5 & -- & 61.91 & 67.44 & 60.83 & 67.99 & 68.13 \\
3.0 & -- & 60.57 & 67.85 & 59.07 & 66.36 & 68.62 \\
3.5 & -- & 60.87 & 68.55 & 58.47 & 65.83 & 68.95 \\
4.0 & -- & 60.73 & 68.20 & 57.87 & 67.04 & 69.71 \\
4.5 & -- & 62.92 & 68.84 & 58.76 & 65.17 & 69.77 \\
5.0 & -- & 62.77 & 69.43 & 59.06 & 66.23 & 70.11 \\
5.5 & -- & 62.26 & 68.98 & 58.96 & 66.91 & 69.98 \\
6.0 & -- & 62.53 & 68.98 & 59.74 & 66.48 & 70.27 \\
6.5 & -- & 61.75 & 69.28 & 61.22 & 67.05 & 69.74 \\
7.0 & -- & 63.21 & 69.25 & 60.03 & 67.52 & 70.40 \\
7.5 & -- & 62.19 & 69.87 & 61.17 & 66.13 & 70.24 \\
8.0 & -- & 62.32 & 69.74 & -- & 67.38 & 70.69 \\
8.5 & -- & 62.05 & 70.22 & -- & 67.22 & 70.67 \\
9.0 & -- & 61.80 & 70.03 & -- & 66.86 & 70.58 \\
9.5 & -- & 61.18 & -- & -- & 67.09 & 70.56 \\
10.0 & -- & 61.48 & -- & -- & 67.20 & -- \\
\bottomrule
\end{tabularx}
\caption{Code mean convergence values underlying Figure~\ref{fig:convergence_means}.}
\label{tab:convergence_code_mean}
\end{table}

\end{document}